\documentclass[twoside,11pt]{article}

\usepackage{blindtext}

\usepackage[preprint]{jmlr2e}
\usepackage{amsmath, amssymb, mathtools}
\usepackage{algorithm, algpseudocode}
\usepackage{enumitem}
\usepackage{makecell, booktabs}  
\usepackage{xcolor}
\usepackage{enumitem}
\usepackage{subcaption}

\newtheorem{assumption}{Assumption}

\newcommand{\norm}[1]{\left\lVert #1 \right\rVert}
\newcommand{\norminf}[2][]{\norm{#2}_{\infty#1}}

\newcommand{\normtwo}[2][]{\norm{#2}_{2#1}}

\newcommand{\state}{\mathcal{S}}
\newcommand{\action}{\mathcal{A}}

\newcommand{\algname}{\mathtt{CycleFQI}}

\DeclareMathOperator*{\argmin}{arg\,min}
\DeclareMathOperator*{\argmax}{arg\,max}

\usepackage{lastpage}
\jmlrheading{1}{2026}{1-\pageref{LastPage}}{2/12}{}{}{Kyungbok Lee, Angelica Cristello Sarteau, and Michael Kosorok}

\ShortHeadings{Provable Offline Reinforcement Learning for Structured
Cyclic MDPs}{Lee, Sarteau, and Kosorok}
\firstpageno{1}

\begin{document}

\title{Provable Offline Reinforcement Learning for Structured Cyclic MDPs}

\author{\name Kyungbok Lee \email kyungbok@unc.edu \\
       \addr Department of Biostatistics\\
       University of North Carolina at Chapel Hill\\
       Chapel Hill, NC 27599-7420, USA
       \AND
       \name Angelica Cristello Sarteau \email angelica.cristellosarteau@vumc.org \\
       \addr Vanderbilt University Medical Center\\
       Nashville, TN 37232, USA
       \AND
       \name Michael R. Kosorok \email kosorok@unc.edu \\
       \addr Department of Biostatistics\\
       University of North Carolina at Chapel Hill\\
       Chapel Hill, NC 27599-7420, USA}

\editor{}

\maketitle

\begin{abstract}
We introduce a novel cyclic Markov decision process (MDP) framework for multi-step decision problems with heterogeneous stage-specific dynamics, transitions, and discount factors across the cycle.
In this setting, offline learning is challenging: optimizing a policy at any stage shifts the state distributions of subsequent stages, propagating mismatch across the cycle.
To address this, we propose a modular structural framework that decomposes the cyclic process into stage-wise sub-problems.
While generally applicable, we instantiate this principle as \(\algname\), an extension of fitted Q-iteration enabling theoretical analysis and interpretation.
It uses a vector of stage-specific Q-functions, tailored to each stage, to capture within-stage sequences and transitions between stages.
This modular design enables partial control, allowing some stages to be optimized while others follow predefined policies.
We establish finite-sample suboptimality error bounds and derive global convergence rates under Besov regularity, demonstrating that the \(\algname\) mitigates the curse of dimensionality compared to monolithic baselines.
Additionally, we propose a sieve-based method for asymptotic inference of optimal policy values under a margin condition.
Experiments on simulated and real-world Type 1 Diabetes data sets demonstrate \(\algname\)'s effectiveness.
\end{abstract}


\section{Introduction and Related Work}

Many real-world decision-making problems exhibit cyclic patterns, unfolding across distinct, repeating stages with unique dynamics and objectives. For instance, Type 1 Diabetes management \citep{javad2019reinforcement, luckett2020estimating} involves decisions that vary across daily stages. State-actions differ contextually: morning decisions use minimal prior data, later actions incorporate cumulative information, and overnight strategies focus on sleep-related stability. Similarly, urban traffic control cycles through distinct phases like rush hour and off-peak periods. Standard reinforcement learning (RL) methods often assume uniform dynamics and struggle with such stage-specific variations and inter-stage dependencies \citep{sutton1998reinforcement}.

Crucially, these applications require a framework that respects the \textit{structural distinctness} of each phase. A monolithic approach that treats diverse stages—ranging from complex active decisions to passive monitoring—as a uniform process inevitably incurs statistical inefficiencies. Instead of forcing a single global model to fit disparate dynamics, our framework explicitly leverages this structure. By decomposing the problem, we match the learning complexity to the intrinsic nature of each stage, ensuring that simple stages are not burdened by the high-dimensional requirements of complex ones. This structural decomposition serves as a meta-strategy adaptable to various value-based algorithms; in this work, we demonstrate its fundamental benefits through the lens of Fitted Q-Iteration.

\textit{Offline Reinforcement Learning.} Traditional RL such as Q-learning \citep{watkins1992q} assumes consistent dynamics, oversimplifying sequential decisions. Offline methods like Fitted Q-Iteration (FQI) learn from fixed data sets but can falter in cyclic settings under standard stationary assumptions \citep{ernst2005tree, riedmiller2005neural}.
Specifically, the distribution mismatch—a core challenge in offline RL—is exacerbated by the cyclic dependency; a policy update in a preceding stage inherently alters the state visitation distribution for the subsequent stage, creating a feedback loop of error accumulation that standard methods fail to mitigate.
Bias-reduction techniques including CQL and BCQ \citep{fujimoto2019off, kumar2020conservative} improve offline robustness but lack specific adaptations for multi-stage cyclic structures. Neural network approaches enhance scalability but rarely handle cyclic dynamics explicitly \citep{nguyen-tang2022on}.

\textit{Cyclic and Hierarchical Reinforcement Learning.} Work on cyclic behaviors often involves periodic MDPs, typically assuming shared state-action spaces with limited variations \citep{aniket2024online, sharma2018phase}, or explores periodic or multi-time-scale MDPs that may neglect intra-stage sequential decisions \citep{jacobson2003markov, wernz2013multi}. Hierarchical RL decomposes tasks but seldom accounts for the explicit cyclic dependencies characterizing our target problems \citep{barto2003recent, dietterich2000hierarchical, sutton1999between}. Research often prioritizes online learning, leaving offline cyclic settings underexplored \citep{hamadanian2023online}. Approaches like cyclic policy distillation focus on domain randomization rather than stage dynamics \citep{kadokawa2023cyclic}.

\textit{Theoretical Guarantees.} Existing theoretical analyses in offline RL often rely on structures limiting applicability to our cyclic, multi-stage setting. Assumptions such as realizability or Bellman completeness \citep{chen2019information, xie2021batch, zhan2022offline} are challenging to satisfy. Standard finite-sample bounds \citep{munos2008finite, nguyen-tang2022on}, sometimes with Besov spaces, do not address cyclic dependencies with stage-varying samples. Our work contributes by introducing interconnected stage-specific Q-functions, offering finite-sample bounds, convergence rates under Besov regularity, and inference for multi-stage settings, extending prior methods \citep{shi2022statistical}.

To address these challenges, we propose \(\algname\), a novel offline RL algorithm tailored for cyclic MDPs. While the core decomposition principle is generalizable, we formulate \(\algname\) within the Fitted Q-Iteration framework to rigorously quantify its statistical properties. Our approach models \( K \) distinct MDP stages cycling infinitely, each with potentially unique state-action spaces, dynamics, rewards, stage transitions $\phi_k$, and discounting $\gamma_k$. 
Unlike standard time-step discounting, we apply discounts at stage transitions, treating each stage as a cohesive block with accumulated rewards. This formulation accommodates variable intra-stage durations—where actions may span non-uniform time intervals—avoiding the unnecessary mathematical complexity of tracking specific time steps without compromising the fundamental cyclic structure.
Learning a vector of stage-specific Q-functions linked via a coupled Bellman system optimizes decisions across the multi-step cyclic process. The framework flexibly learns policies for all stages, or optimizes a subset while using predefined policies $\pi_k^\circ$ for others such as fixed protocols or uncontrollable environmental phases, focusing learning on adaptable stages. This \textit{modular approach} enhances practicality when full control is infeasible and strictly enforces safety constraints required by real-world clinical protocols. We extend FQI for these dynamics, ensuring robust offline performance. Experiments validate \(\algname\)'s effectiveness in capturing complex cyclic patterns and adapting to varying policy specifications on synthetic and real-world T1D data.

Our contributions are:
\begin{itemize}
    \item \textit{Novel Cyclic MDP Framework and Algorithm}: We introduce a cyclic MDP framework designed for stage heterogeneity, featuring interconnected stage-specific Q-functions and a coupled Bellman system. We formalize this via the \(\algname\) algorithm, extending fitted Q-iteration to allow for modular optimization—learning policies for all stages or focusing on a subset while respecting pre-defined policies for the remainder. This structure addresses limitations in prior cyclic methods \citep{aniket2024online, kadokawa2023cyclic}.

    \item \textit{Robust Theoretical Guarantees}: We establish finite-sample guarantees for the proposed \(\algname\) within this cyclic setting. We theoretically demonstrate that our decomposition strategy mitigates the curse of dimensionality compared to monolithic baselines. Specifically, while standard flattened approaches suffer from the cumulative dimension of the cycle, our analysis shows that \(\algname\) is statistically bottlenecked only by the worst-case single stage, providing a fundamental efficiency gain in high-dimensional structured environments. We also extend sieve methods to achieve asymptotic multivariate normality for value estimation \citep{shi2022statistical}.

    \item \textit{Empirical Validation}: Comprehensive experiments on synthetic and real-world T1D data sets demonstrate \(\algname\)'s effectiveness in capturing complex cyclic dependencies and, crucially, validate its flexibility and practical utility, including its capability to handle cycles with partially pre-defined policies.
\end{itemize}

\section{Problem Setup: Cyclic MDP}\label{sec:problem}

We introduce a Cyclic Markov Decision Process (MDP) framework with \( K \) stages, where each stage operates as a finite-horizon MDP with potentially distinct characteristics. The process cycles through these stages, with deterministic transitions triggered upon reaching terminal state-action pairs within each stage.
This section formalizes the components of this cyclic MDP, including stage dynamics, transitions, value functions, and the relevant Bellman equations adapted to our setting.

To facilitate the understanding of our framework and the indices used for stages, iterations, and time steps, we summarize the key notations in Table \ref{tab:notation}.

\begin{table}[t]
\centering
\caption{Summary of Key Notations used in the Cyclic MDP framework.}
\label{tab:notation}
\begin{tabular}{l l}
\toprule
\textbf{Symbol} & \textbf{Description} \\
\midrule
$K$ & Total number of distinct stages in the cycle \\
$k \in \{1, \dots, K\}$ & Index for stages (e.g., $k=1$ for Morning) \\
$m$ & Index for the training iteration ($m=1, \dots, M$) \\
$d_k$ & Dimension of the state space at stage $k$ \\
$\mathcal{S}_k, \mathcal{A}_k$ & State and Action spaces for stage $k$ \\
$\phi_k: \mathcal{S}_k \to \mathcal{S}_{k+1}$ & Inter-stage transition mapping function \\
$\gamma_k$ & Discount factor applied upon transitioning \textit{out} of stage $k$ \\
$\mathcal{U} \subseteq \{1, \dots, K\}$ & Update set containing indices of adaptable stages \\
$\pi_k^\circ$ & Fixed behavior policy for non-adaptable stages ($k \notin \mathcal{U}$) \\
$Q_k^{(m)}$ & Estimated Q-function for stage $k$ at iteration $m$ \\
\bottomrule
\end{tabular}
\end{table}

\subsection{Structure of the Cyclic MDP}
\textit{Intra-Stage Dynamics.} For each stage \( k \in \{1, \dots, K\} \), the dynamics are governed by an MDP \( \mathcal{M}_k = (\mathcal{S}_k, \mathcal{A}_k, P_k, R_k, \mathcal{T}_k) \). The state space \( \mathcal{S}_k \subseteq [0,1]^{d_k} \) can be continuous, while the action space \( \mathcal{A}_k \) is finite. The transition kernel \( P_k: \mathcal{S}_k \times \mathcal{A}_k \to \Delta(\mathcal{S}_k) \) governs intra-stage state dynamics, where \( \Delta(\mathcal{S}_k) \) denotes the set of probability distributions over \( \mathcal{S}_k \). The reward function \( R_k: \mathcal{S}_k \times \mathcal{A}_k \to \Delta([0, R_{\max,k}]) \) defines bounded nonnegative rewards, accumulated \textit{without} discounting within stages. A set of terminal state-action pairs \( \mathcal{T}_k \subseteq \mathcal{S}_k \times \mathcal{A}_k \) triggers transitions to the next stage.

In the \( m \)-th visited stage starting from stage \( k \), the stage index is \( [m+k-1] \), defined cyclically as \( [m] = ((m-1) \pmod{K}) + 1 \). A \textit{stage-specific policy} \( \pi_k: \mathcal{S}_k \to \Delta(\mathcal{A}_k) \) maps states to action distributions, while the \textit{composite policy} is \( \pi = (\pi_1, \dots, \pi_K) \). At each step \( t = 1, \dots, \tau_m \), the dynamics evolve as:
\[
\begin{aligned}
a_{m,t} &\sim \pi_{[m+k-1]}(\cdot \mid s_{m,t}), \\
r_{m,t} &\sim R_{[m+k-1]}(\cdot \mid s_{m,t}, a_{m,t}), \\
s_{m,t+1} &\sim P_{[m+k-1]}(\cdot \mid s_{m,t}, a_{m,t}),
\end{aligned}
\]
until a terminal pair \( (s_{m,\tau_m}, a_{m,\tau_m}) \in \mathcal{T}_{[m+k-1]} \) is reached at stopping time
\[
\tau_m = \inf \{ t \geq 1 : (s_{m,t}, a_{m,t}) \in \mathcal{T}_{[m+k-1]} \},
\]
with \( \tau_m \leq H_{[m+k-1]} \), where \( H_k \) is the finite horizon of stage \( k \).

\textit{Inter-Stage Transitions.} Upon termination at step \( \tau_m \), a deterministic transition to the \((m+1)\)-th stage \( [m+k] \) occurs via a \textit{stage transition mapping}:
\[
s_{m+1,1} = \phi_{[m+k-1]}(s_{m,\tau_m +1}),
\]
with a \textit{stage-specific discount factor} \( \gamma_{[m+k-1]} \), where \( 0 \leq \gamma_k \leq 1 \) for all \( k \), and \textit{at least one} \( \gamma_k < 1 \) to ensure bounded total discounted rewards. The \textit{cycle discount factor} is defined as
\[
\gamma_{\mathrm{cycle}} = \prod_{k=1}^K \gamma_k.
\]

\begin{remark}
This cyclic framework generalizes the infinite-horizon MDP. Setting \( K=1 \) and \( H_1=1 \) recovers the standard MDP, ensuring all results apply to conventional MDPs.
\end{remark}

\begin{figure}[t]
    \centering
    \includegraphics[width=0.7\textwidth]{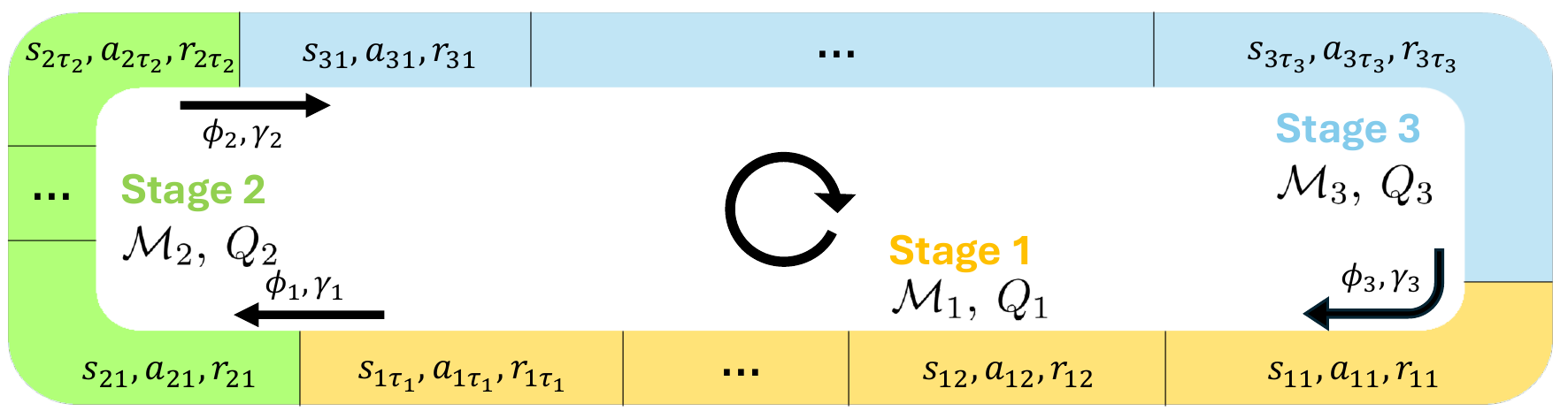}  
    \caption{Illustration of a cyclic MDP with \( K=3 \) stages. Each stage \( k \) is an MDP \( \mathcal{M}_k \) with \( \tau_k \) steps, connected cyclically via transitions \( \phi_k \) with discounts \( \gamma_k \). We estimate the optimal Q-function \( Q_k^* \) for each stage, maximizing expected discounted reward over an infinite loop starting from stage \( k \).}
    \label{fig:donut}
\end{figure}

\subsection{Action-Value and State-Value Functions}  

For the $m$-th visited stage, the stage-level reward is defined as the cumulative sum of undiscounted rewards, $R_m = \sum_{t=1}^{\tau_m} r_{m,t}$. The cumulative inter-stage discount factor starting from stage $k$ is given by
\[ 
\Gamma_{k,m} = \prod_{j=1}^{m-1} \gamma_{[k+j-1]}, 
\]
with $\Gamma_{k,1} = 1$. Given a composite stationary policy $\pi = (\pi_1, \dots, \pi_K)$, the \textit{action-value function} $\mathbf{Q}^\pi = (Q_1^\pi, \dots, Q_K^\pi)$ consists of stage-specific components $Q_k^\pi: \mathcal{S}_k \times \mathcal{A}_k \to [0, Y]$, where the upper bound is
\[ 
Y = \frac{1}{1 - \gamma_{\mathrm{cycle}}} \sum_{j=1}^K H_j R_{\max,j}. 
\]

Each function $Q_k^\pi(s_k, a_k)$ represents the expected total discounted future reward obtained by starting from $(s_k, a_k)$ at stage $k$ and following policy $\pi$:
\[ 
Q_k^\pi(s_k, a_k) = \mathbb{E}_\pi \left[ \sum_{m=1}^\infty \Gamma_{k,m} R_m \;\middle|\; s_{1,1} = s_k, a_{1,1} = a_k \right]. 
\]
The corresponding \textit{state-value function} $\mathbf{V}^\pi = (V_1^\pi, \dots, V_K^\pi)$ is defined as
\[ 
V_k^\pi(s_k) = \mathbb{E}_{a_k \sim \pi_k(\cdot \mid s_k)}[Q_k^\pi(s_k, a_k)]. 
\]

\subsection{Constrained Optimality and the Bellman Operator}  

Building upon the general value function definitions, we introduce the \textit{Bellman optimality operator} $\mathbf{T}_{\mathcal{U}}$ tailored for our cyclic MDP framework, reflecting the capability to handle partially specified policies. We define an \textit{update set} $\mathcal{U} \subseteq \{1,\dots,K\}$ indicating stages where policies are learned, while policies for stages $k \notin \mathcal{U}$ remain fixed as given policies $\pi_k^{\circ}$. The objective in this setting is to find the optimal constrained policy $\pi_{\mathcal{U}}^*$ that maximizes value subject to $\pi_k = \pi_k^\circ$ for $k \notin \mathcal{U}$.

Let \( L_\infty(\mathcal{S}_k \times \mathcal{A}_k) \) denote the space of bounded functions on \( \mathcal{S}_k \times \mathcal{A}_k \). The Bellman operator for this constrained problem, \( \mathbf{T}_{\mathcal{U}} \), maps the space of vector Q-functions to itself, as \( \mathbf{T}_{\mathcal{U}}: \prod_{k=1}^K L_\infty(\mathcal{S}_k \times \mathcal{A}_k) \to \prod_{k=1}^K L_\infty(\mathcal{S}_k \times \mathcal{A}_k) \). To define its action on an input Q-function vector \( \mathbf{Q} = (Q_1, \dots, Q_K) \), we first define the state value \( V_k(s_k) \) for each stage \( k \), which reflects the policy constraints:
\begin{equation}\label{eq:state_value_function}
V_k(s_k) = \max_{a_k \in \mathcal{A}_k} Q_k(s_k, a_k) \mathbb{I}(k \in \mathcal{U}) + \mathbb{E}_{a_k \sim \pi_k^{\circ}(\cdot \mid s_k)} Q_k(s_k, a_k) \mathbb{I}(k \notin \mathcal{U}).
\end{equation}
Using the expected immediate reward $r_k(s_k, a_k)$, the $k$-th component of the output \( (\mathbf{T}_{\mathcal{U}}\mathbf{Q})_k \) is:
\begin{multline}\label{eq:bellman_operator_action}
(\mathbf{T}_{\mathcal{U}}\mathbf{Q})_k(s_k, a_k) = r_k(s_k, a_k) \\
+ \mathbb{E}_{s'_k \sim P_k(\cdot \mid s_k, a_k)} \Big[ 
    V_k(s'_k) \mathbb{I}((s_k, a_k) \notin \mathcal{T}_k) + \gamma_k V_{[k+1]}(\phi_k(s'_k)) \mathbb{I}((s_k, a_k) \in \mathcal{T}_k) 
\Big].
\end{multline}

The unique fixed point of this operator, denoted $\mathbf{Q}_{\mathcal{U}}^*$, represents the optimal action-values achievable under the constraints imposed by $\mathcal{U}$ and $\{\pi_k^\circ\}_{k \notin \mathcal{U}}$. The corresponding optimal constrained policy $\pi_{\mathcal{U}}^*$ acts greedily with respect to $\mathbf{Q}_{\mathcal{U}}^*$ for stages $k \in \mathcal{U}$ and follows the fixed policy $\pi_k^\circ$ for stages $k \notin \mathcal{U}$ \citep{bertsekas1996stochastic}. The following proposition establishes properties of $\mathbf{T}_{\mathcal{U}}$ ensuring convergence to this fixed point.

\begin{proposition}[Contraction Property]\label{prop:contraction}
Let \( H = \sum_{k=1}^K H_k \) and \( \gamma_{\mathrm{cycle}} = \prod_{k=1}^K \gamma_k \). For any update set \(\mathcal{U} \subseteq \{1,\dots,K\}\) and vectors \( \mathbf{f}, \mathbf{g} \in \prod_{k=1}^K L_\infty(\mathcal{S}_k \times \mathcal{A}_k) \), the operator \( \mathbf{T}_{\mathcal{U}} \) defined by Equation ~\ref{eq:state_value_function}–\ref{eq:bellman_operator_action} has a unique fixed point \( \mathbf{Q}_{\mathcal{U}}^* \), is non-expansive, and satisfies the $H$-step contraction property:
\[
\| \mathbf{T}_{\mathcal{U}}\mathbf{f} - \mathbf{T}_{\mathcal{U}} \mathbf{g} \|_\infty \leq \| \mathbf{f} - \mathbf{g} \|_\infty \quad \text{and} \quad \| \mathbf{T}_{\mathcal{U}}^{H} \mathbf{f} - \mathbf{T}_{\mathcal{U}}^{H} \mathbf{g} \|_\infty \leq \gamma_{\mathrm{cycle}} \, \| \mathbf{f} - \mathbf{g} \|_\infty.
\]
\end{proposition}
We now introduce an \(\algname\), for estimating the optimal action-value vector $\mathbf{Q}_{\mathcal{U}}^*$ from offline data.

\section{Proposed Method: Cyclic Fitted Q-Iteration (\texorpdfstring{\(\algname\)}{CycleFQI})}
\label{sec:proposed_method}

\textit{data set Description.} In the offline reinforcement learning setting, our objective is to train stage-specific Q-functions using pre-collected data sets. The complete data set $\mathcal{D}$ consists of stage-specific components $\mathcal{D} = \{\mathcal{D}_k\}_{k=1}^K$, where each $\mathcal{D}_k = \{ (s_k^i, a_k^i, r_k^i, s'^{i}_k) \}_{i=1}^{n_k}$ contains $n_k$ transition tuples collected during stage $k$. The total number of samples is $n = \sum_{k=1}^K n_k$. 

Each tuple $(s_k^i, a_k^i, r_k^i, s'^{i}_k)$ represents a transition: from state $s_k^i$, taking action $a_k^i$ resulted in reward $r_k^i$ and subsequent state $s'^{i}_k$. We assume the rewards $\{ r_k^i \}_{i=1}^{n_k}$ are independent given the state-action pairs, sampled from $R_k(\cdot \mid s_k^i, a_k^i)$, and $s'^{i}_k$ is sampled from the transition kernel $P_k(\cdot \mid s_k^i, a_k^i)$.

\subsection{Algorithm: Cyclic Fitted Q-Iteration}
\label{subsec:algorithm}

The \textit{Cyclic Fitted Q-Iteration (\(\algname\))}, presented in Algorithm \ref{alg:cfqi}, adapts the Fitted Q-Iteration (FQI) framework \citep{ernst2005tree} to the cyclic MDP structure defined in Section \ref{sec:problem}. \(\algname\) iteratively updates a vector of stage-specific Q-functions, $\mathbf{Q}^{(m)} = (Q_1^{(m)}, \dots, Q_K^{(m)})$, aiming to approximate the unique fixed point $\mathbf{Q}_{\mathcal{U}}^*$ of the Bellman operator $\mathbf{T}_{\mathcal{U}}$.

In each iteration $m$, the algorithm computes the target values $y_k^i$ for each transition using the estimates from the previous iteration $\mathbf{Q}^{(m-1)}$:
\begin{equation}\label{eq:target_value_update}
y_k^i = r_k^i + 
\begin{cases} 
V_k^{(m-1)}(s_k'^{i}) & \text{if } (s_k^i, a_k^i) \notin \mathcal{T}_k, \\
\gamma_k V_{[k+1]}^{(m-1)}(\phi_k(s_k'^{i})) & \text{if } (s_k^i, a_k^i) \in \mathcal{T}_k,
\end{cases}
\end{equation}
where $V_k^{(m-1)}$ is obtained via Equation \ref{eq:state_value_function}. This construction explicitly links the stages: the target for stage $k$ depends not only on its own previous value $V_k^{(m-1)}$ but also on the next stage's value $V_{[k+1]}^{(m-1)}$ at terminal transitions, reflecting the inherent cyclic dependency of the framework across iterations.

Although the stages are theoretically coupled through these inter-stage dependencies, a key feature of \(\algname\) is that within each single iteration $m$, the regression tasks for each stage $k$ become computationally independent once the targets $\{y_k^i\}$ are fixed using $\mathbf{Q}^{(m-1)}$. This allows the algorithm to update the Q-function vector by solving $K$ separate least-squares problems:
\begin{equation}\label{eq:q_update_regression}
Q_k^{(m)} = \argmin\limits_{f \in \mathcal{F}_k} \sum_{i=1}^{n_k} \left( f(s_k^i, a_k^i) - y_k^i \right)^2.
\end{equation}
This stage-wise decomposition ensures that the complex cyclic dependencies are updated synchronously in each iteration, maintaining a structured flow of information across the entire cycle while allowing for efficient implementation within each fitting step.

\begin{algorithm}[th]
\caption{Cyclic Fitted Q-Iteration (\(\algname\))}
\label{alg:cfqi}
\begin{algorithmic}[1]
    \State \textbf{Input}: 
    \Statex \quad data sets \( \{\mathcal{D}_k\}_{k=1}^K \), Update set \(\mathcal{U} \subseteq \{1,\dots,K\}\), Fixed policies \(\{\pi_k^{\circ}\}_{k \notin \mathcal{U}}\)
    \Statex \quad Function classes \( \{\mathcal{F}_k\}_{k=1}^K \), Initial Q-functions \( \mathbf{Q}^{(0)} \), Iterations \( M \)
    \Statex \quad Termination sets \( \{\mathcal{T}_k\} \), Discounts \( \{\gamma_k\} \), Transition maps \( \{\phi_k\} \)
    
    \vspace{0.5em}
    \For{\( m = 1 \) to \( M \)} \Comment{Global iteration for cycle-wide convergence}
        \For{\( k = 1 \) to \( K \)} \Comment{Stage-wise parallelizable update}
            
            \State \textbf{Step 1: Compute Target Values}
            \State For each transition tuple \( (s_k^i, a_k^i, r_k^i, s'^i_k) \in \mathcal{D}_k \):
            \State \quad \( y_k^i = r_k^i + 
                \begin{cases} 
                V_k^{(m-1)}(s_k'^{i}) & \text{if } (s_k^i, a_k^i) \notin \mathcal{T}_k \\
                \gamma_k V_{[k+1]}^{(m-1)}(\phi_k(s_k'^{i})) & \text{if } (s_k^i, a_k^i) \in \mathcal{T}_k
                \end{cases} \) \Comment{Using Equation \ref{eq:state_value_function} and \ref{eq:target_value_update}}
            
            \vspace{0.5em}
            \State \textbf{Step 2: Least-Squares Fitting}
            \State \( Q_k^{(m)} = \argmin\limits_{f \in \mathcal{F}_k} \sum_{i=1}^{n_k} \left( f(s_k^i, a_k^i) - y_k^i \right)^2 \) \Comment{Fit new Q-function to targets}
            
        \EndFor
    \EndFor
    \vspace{0.5em}

    \State \textbf{Output}: Policy vector \( \boldsymbol{\pi}^{(M)} = (\pi_1^{(M)}, \dots, \pi_K^{(M)}) \), where for each stage \( k \):
    \State \quad \(\pi_k^{(M)}(s_k) = 
        \begin{cases} 
        \argmax_{a_k \in \mathcal{A}_k} Q_k^{(M)}(s_k, a_k) & \text{if } k \in \mathcal{U} \\ 
        \pi_k^{\circ}(s_k) & \text{if } k \notin \mathcal{U} 
        \end{cases}\)
\end{algorithmic}
\end{algorithm}

\section{Finite-Sample Analysis for \texorpdfstring{\(\algname\)}{CycleFQI}}
\label{sec:finite_sample_analysis}

This section establishes \textit{finite-sample guarantees} for \(\algname\). The analysis applies to any fixed update set \(\mathcal{U} \subseteq \{1,\dots,K\}\), where the goal is to optimize policies for stages \(k \in \mathcal{U}\) while keeping policies \(\pi_k^\circ\) fixed for stages \(k \notin \mathcal{U}\). We provide high-probability uniform bounds on the suboptimality gap across all \(K\) stages. 

Let \(V^*\) denote the optimal state-value function. We define the expected value at stage \(k\) with respect to an initial distribution \(\eta_k \in \Delta(\mathcal{S}_k)\) for the optimal policy and the learned policy \(\pi^{(M)}\) returned by Algorithm \ref{alg:cfqi} as follows:
\begin{align*}
    v_k^* &= \mathbb{E}_{s_k \sim \eta_k}[V_k^*(s_k)], \\
    v_k^{(M)} &= \mathbb{E}_{s_k \sim \eta_k}[V_k^{\pi^{(M)}}(s_k)].
\end{align*}
Our primary objective is to bound the global suboptimality gap defined as the sup-norm difference between the optimal and learned value vectors:
\[
    \text{Gap}(\pi^{(M)}) \coloneqq \norminf{\mathbf{v}^* - \mathbf{v}^{(M)}},
\]
where \(\mathbf{v}^* = (v_1^*, \dots, v_K^*)\) and \(\mathbf{v}^{(M)} = (v_1^{(M)}, \dots, v_K^{(M)})\). These bounds characterize the complex interplay between function class complexities \(\mathcal{F}_k\), per-stage sample sizes \(n_k\), and the approximation errors as they propagate through the cyclic structure.

To this end, our analysis proceeds in three steps: we first establish a general high-probability suboptimality bound in Theorem \ref{thm:main_convergence}, then derive specific convergence rates under Besov regularity in Theorem \ref{thm:finite_sample_rate}, and finally provide expected rates for stochastic algorithms such as Random Forests in Theorem \ref{thm:random_forest}.

\subsection{High-Probability Bound on Suboptimality Gap}
\label{subsection:high_probability_bound}

Building on per-stage analysis, we bound the global suboptimality gap by analyzing error propagation over full cycles.

\begin{assumption}[Sampling Distribution Coverage]
\label{assumption:coverage}
For some constant \( C > 0 \), for each stage \( k \), the data distribution \( \nu_k \) on $\mathcal{S}_k \times \mathcal{A}_k$ satisfies
\[ \sup_{\beta_k \in \mathcal{B}_k} \left\| \frac{d \beta_k}{d \nu_k} \right\|_\infty \leq C, \]
where \(\mathcal{B}_k\) is the set of admissible state-action distributions visited within stage \(k\) over trajectories generated by any (possibly non-stationary) policy starting from the initial distribution \(\eta_{k_0}\).
\end{assumption}

\begin{remark}[Role of Distributional Coverage]
Assumption \ref{assumption:coverage} is a standard offline RL condition \citep{chen2019information, lazaric2016analysis, munos2003error}, adapted for the cyclic multi-stage structure by requiring it to hold stage-wise. It ensures that \(\nu_k\) provides sufficient coverage over relevant distributions \(\beta_k\) arising within potentially complex cyclic trajectories. Crucially, this uniform coverage controls the distributional mismatch during optimization; without it, approximation errors would not merely persist but amplify through the transition maps \(\phi_k\) across repeated cycles, preventing global convergence.
\end{remark}

We define the stage-wise \textit{approximation error} for the constrained optimal Q-function \(Q_k^*\) as:
\begin{equation}\label{eq:approx_error}
\epsilon_{\mathrm{approx}, k} := \inf_{f \in \mathcal{F}_k} \normtwo[,\nu_k]{Q_k^* - f}.
\end{equation}
This captures the expressivity of \(\mathcal{F}_k\) under \(\nu_k\), avoiding stricter assumptions such as realizability \citep{chen2019information, xie2021batch}. 

The cyclic environment of \(\algname\) introduces challenges such as stochastic horizons, which preclude the direct application of single-step Bellman contraction arguments. Our analysis addresses this by employing a modified performance difference lemma tailored to the cyclic setting.  This allows us to decouple the global suboptimality gap into local stage-wise errors, extending the analysis to uncountable function classes beyond the finite classes often considered in prior literature \citep{chen2019information, liu2020provably, xie2021batch}. Using covering numbers and Bernstein’s inequality, we establish the following high-probability guarantee.

\begin{theorem}[Suboptimality Bound for \(\algname\)]
\label{thm:main_convergence}
For any update set \(\mathcal{U} \subseteq \{1,\dots,K\}\), let \(\mathbf{v}^*\) and \(\mathbf{v}^{(M)}\) be the optimal and induced value vectors, respectively. Under Assumption \ref{assumption:coverage}, with probability at least \( 1 - \delta \),
\[
\|\mathbf{v}^* - \mathbf{v}^{(M)}\|_\infty \leq \frac{2 \sqrt{C} H^2}{(1 - \gamma_{\mathrm{cycle}})^2} \cdot \varepsilon + \frac{2 H Y}{1 - \gamma_{\mathrm{cycle}}} \cdot \gamma_{\mathrm{cycle}}^{\lfloor M/H \rfloor},
\]
where \( H = \sum_{k=1}^K H_k \), \( \varepsilon = \max_k \epsilon_k \), and the unified error \(\epsilon_k\) for each stage \(k\) is defined as:
\begin{equation}\label{eq:unified_error}
\epsilon_k = \sqrt{ 45 Y^2 n_k^{-1} \log\left( 2 M K \mathcal{N}_k(1/n_k)/\delta \right) + 40 \epsilon_{\mathrm{approx}, k}^2 }.
\end{equation}
\end{theorem}

Theorem \ref{thm:main_convergence} explicitly highlights the dependence on sample size \(n_k\), function class capacity \(\mathcal{N}_k\), and approximation errors \(\epsilon_{\mathrm{approx}, k}\) relative to the constrained optimum. The first term represents the persistent error, where stage-wise inaccuracies are amplified through the cyclic geometry by the \((1 - \gamma_{\mathrm{cycle}})^{-2}\) factor. The second term denotes the optimization error, which vanishes at a geometric rate. 

Notably, the term \(\epsilon_k\) in Equation  \ref{eq:unified_error} serves as a unified risk metric for stage $k$, balancing two distinct sources of error. The first component, involving the covering number \(\mathcal{N}_k(1/n_k)\) and the \(\log(1/\delta)\) term, characterizes the \textit{statistical estimation risk}. It reflects the generalization gap inherent in learning from finite samples, where the complexity of the function class \(\mathcal{F}_k\) dictates the variance of the empirical risk minimizer. The second component, represented by \(\epsilon_{\mathrm{approx}, k}\), quantifies the \textit{approximation risk} arising from the potential bias of the function class. By combining these terms under a square root, \(\epsilon_k\) encapsulates the total learning error at each stage, which then propagates through the cycle. This explicit decomposition allows for a more granular understanding of how stage-specific data quality and model selection collectively influence the global suboptimality gap.

\subsection{Finite-Sample Rates on Suboptimality Gap in Besov Spaces}
\label{subsection:besov}

Besov spaces \( B^s_{p,q}(\mathcal{X}) \) provide a refined framework for characterizing function smoothness, encompassing and extending classical Sobolev and Hölder spaces \citep{devore1993constructive}. The primary advantage of Besov regularity in our cyclic reinforcement learning framework is its unique ability to capture diverse and spatially inhomogeneous regularities. Unlike standard smoothness classes, Besov spaces can model functions that exhibit sharp localized variations or decision boundaries in certain regions while remaining highly smooth elsewhere \citep{ciesielski1982spline, hardle2012wavelets}. Such properties are frequently observed in optimal Q-functions, where abrupt state-transitions or policy changes create localized complexity that conventional Sobolev-based analyses often fail to capture accurately.

For parameters \( s > 0 \) and \( 1 \leq p, q \leq \infty \), the space \( B^s_{p,q}(\mathcal{X}) \) consists of functions \( f: \mathcal{X} \to \mathbb{R} \) with finite Besov norm:
\[
\|f\|_{B^s_{p,q}(\mathcal{X})} := \|f\|_{L^p(\mathcal{X})} + \left( \int_0^1 \left( t^{-s} \omega_{k'}(f, t)_p \right)^q \frac{dt}{t} \right)^{1/q} < \infty,
\]
where \( \omega_{k'}(f, t)_p \) is the \( k' \)-th order \( L^p \)-modulus of smoothness for \( k' > s \). The index $s$ quantifies the degree of differentiability, while $p$ and $q$ allow for fine-grained control over integral and oscillatory properties. Assuming the stage-wise constrained optimal Q-functions $Q_k^*$ lie in these spaces enables a precise, non-parametric approximation analysis for diverse ML architectures \citep{siegel2023optimal, wada2023approximation}.

\begin{assumption}[Besov Regularity]\label{assumption:besov}
For each stage \( k \in \{1,\dots,K\} \), the constrained optimal Q-function \( Q_k^* \) satisfies
\[
Q_k^*(\cdot, a_k) \in B^{s_k}_{p_k,q_k}(\mathcal{S}_k), \quad \forall a_k \in \mathcal{A}_k,
\]
for some parameters \( s_k > 0 \), \( 1 \leq p_k, q_k \leq \infty \), with the embedding condition \( s_k > d_k/p_k \).
\end{assumption}

This regularity condition facilitates deriving \textit{finite-sample error rates} by enabling sharp bounds on the approximation error $\epsilon_{\mathrm{approx}, k}$ defined in Equation \ref{eq:approx_error}. The requirement \( s_k > d_k/p_k \) is essential as it ensures functions satisfy the Sobolev embedding condition, allowing for uniform control of approximation errors via the supremum norm \citep{adams2003sobolev}.

\begin{theorem}[Finite Sample Convergence Rate under Besov Regularity]\label{thm:finite_sample_rate}
Consider Algorithm \ref{alg:cfqi} with stage-wise function classes \( \{\mathcal{F}_k\}_{k=1}^K \), where each \( \mathcal{F}_k \) is characterized by a complexity parameter \( D_k \). Suppose Assumptions~\ref{assumption:coverage} and \ref{assumption:besov} hold. Assume that for each stage \( k \in \{1,\dots,K\} \), there exists an approximation exponent \( \alpha_k > 0 \)—determined by the structural properties of \( \mathcal{F}_k \) and the smoothness \( s_k \)—such that the following capacity and approximation conditions are satisfied for some constants \( \mathcal{C}_{k,\text{cap}}, \mathcal{C}_{k,\text{approx}} > 0 \):
\begin{enumerate}[label=(C\arabic*), ref=C\arabic*]
    \item \label{cond:C1}
    \textbf{Metric Entropy (Capacity Control):} For any \( \epsilon > 0 \),
    \[
    \log \mathcal{N}_k(\epsilon, \mathcal{F}_k, \|\cdot\|_\infty)
    \leq \mathcal{C}_{k,\text{cap}} D_k \log\left( \frac{\mathrm{poly}(D_k)}{\epsilon} \right).
    \]

    \item \label{cond:C2}
    \textbf{Approximation Power (Bias Bound):} The approximation error relative to the constrained optimum
    \( Q_k^* \in B^{s_k}_{p_k, q_k} \) satisfies
    \[
    \inf_{f \in \mathcal{F}_k} \|Q_k^* - f\|_{2, \nu_k}
    \leq \mathcal{C}_{k,\text{approx}} D_k^{-\alpha_k}.
    \]
\end{enumerate}

If the stage-wise complexity is scaled as \( D_k \asymp n_k^{1/(2\alpha_k + 1)} \) to balance estimation and approximation risks, then for any iteration count \( M \ge \Omega(\mathrm{poly}(\sum\limits_{k=1}^K n_k)) \), the value vector \( \mathbf{v}^{(M)} \) satisfies, with probability at least \( 1 - \delta \):
\[
\|\mathbf{v}^* - \mathbf{v}^{(M)}\|_\infty = \widetilde{\mathcal{O}} \left( \frac{\sqrt{C} H^2 Y}{(1 - \gamma_{\mathrm{cycle}})^2} \cdot \max_{k \in \{1,\dots,K\}} n_k^{-\frac{\alpha_k}{2\alpha_k+1}} \right),
\]
where \( \mathbf{v}^* \) is the constrained optimal value vector and \( \widetilde{\mathcal{O}}(\cdot) \) hides terms logarithmic with respect to \( n_k, 1/\delta, M, \) and \( K \).
\end{theorem}

Theorem \ref{thm:finite_sample_rate} serves as a blueprint for model selection, demonstrating that the optimal rate is achieved by balancing the sample size \(n_k\) with a complexity \(D_k\) tailored to the stage-specific smoothness \(s_k\).

As summarized in Table \ref{tab:approx_rates_covering}, our framework offers the flexibility to employ a wide range of approximators—from classical bases like B-splines and Wavelets to modern architectures such as ReLU DNNs and Transformers. Appendix B provides a comprehensive analysis of how these diverse function classes satisfy the required Besov regularity conditions. This modularity ensures that the choice of \(\mathcal{F}_k\) can be adapted to the intrinsic complexity of each stage.

Crucially, our analysis confirms that the cyclic structure scales the global error only by a constant factor $(1-\gamma_{\mathrm{cycle}})^{-2}$, preserving the fundamental non-parametric rates of the chosen regressor. This effectively \textit{decouples} the learning difficulties across stages, ensuring that the global convergence rate is dictated by the \textit{worst-case stage-wise complexity}, rather than the accumulated complexity of the entire trajectory.

\begin{table}[!t]
\centering
\caption{Approximation rates (\(\alpha_k\)) for Besov spaces across various function classes.}
\renewcommand{\arraystretch}{1.2}
{\fontsize{8}{10}\selectfont
\begin{tabular}{|c|c|c|c|}
\hline
Function Class $\mathcal{F}_k$ & Hyperparameter $D_k$ & Structural Properties & $\alpha_k$ \\
\hline\hline
\makecell[c]{\vspace{0.3em}B-Spline\\\citep{ciesielski1982spline}\vspace{0.3em}} & $N^{d_k}$ (knots) & Degree $m \ge 0$ & \scriptsize $\min(m+1,s_k)/d_k$ \\
\hline
\makecell[c]{\vspace{0.3em}Wavelet\\\citep{hardle2012wavelets}\vspace{0.3em}} & $2^{J d_k}$ (resolution) & Vanishing moment $r \ge 1$ & \scriptsize $\min(r, s_k)/d_k$ \\
\hline
\makecell[c]{\vspace{0.3em}Radial Basis Function\\\citep{hangelbroek2010nonlinear}\vspace{0.3em}} & $N$ (bases) & $1/q_k \le 1 + s_k/d_k$ & $s_k/d_k$ \\
\hline
\makecell[c]{\vspace{0.3em}ReLU DNN\\\citep{siegel2023optimal}\vspace{0.3em}} & Depth $L \times$ Width $W$ & $W \asymp d_k$ & $2s_k/d_k$ \\
\hline
\makecell[c]{\vspace{0.3em}Transformer\\\citep{takakura2023approximation}\vspace{0.3em}} &
$2^L$ (Depth $L$) &
\makecell[c]{\scriptsize \vspace{0.2em}Width $\asymp d_k 2^L$, \\ Sparsity $\asymp d_k^2 2^L\vspace{0.2em}$} &
$s_k/d_k$ \\
\hline
\end{tabular}
}
\label{tab:approx_rates_covering}
\end{table}

\subsubsection{Mitigating the Curse of Dimensionality via Decomposition}
\label{subsubsec:curse_of_dim}

To demonstrate the practical benefits of the proposed structural decoupling, particularly within the Besov framework, we contrast $\algname$ against a standard \textit{Flattened} baseline that models the joint state space directly.

To quantify the dimensional advantage, we first define the effective input dimensions for the structured and flattened approaches as follows:
\[
d_{\max} := d + \max_{1 \le k \le K} d_k, \quad \text{and} \quad d_{\text{total}} := d + \sum_{k=1}^K d_k.
\]
The distinction between these two quantities is the key driver of the performance gap. Consider a representative balanced setting where all stage-specific dimensions $d_k$ are approximately equal to a constant $\bar{d}$. In this case, the joint dimension $d_{\text{total}}$ scales linearly with the number of stages $K$, whereas the structured dimension $d_{\max}$ remains constant. As $K$ increases, this roughly $K$-fold difference leads to a substantial separation in sample complexity. 

\textit{Setup and Baseline Construction.}
We assume a balanced data set consisting of $n$ transition tuples for each stage, resulting in a total sample size of $N = Kn$. For the proposed $\algname$, the input at any stage $k$ is naturally bounded by dimension $d_{\max}$. In contrast, the flattened baseline operates on the joint state space with dimension $d_{\text{total}}$, embedding stage-specific features into a unified high-dimensional vector via zero-padding.
The baseline treats the problem as a single-stage MDP defined over the disjoint union of stage-wise actions and applies standard Fitted Q-Iteration on this joint domain. For statistical error analysis, the baseline uses the entire pooled data set of size $N$, corresponding to a naive construction where stage modularity is ignored.

\textit{Regularity and Favorable Extension.}
For a consistent comparison, we assume that all stage-wise optimal Q-functions $\{Q_k^*\}_{k=1}^K$ share identical Besov smoothness parameters $s, p, q$ as defined in Assumption~\ref{assumption:besov}. Defining the target regularity for the flattened baseline requires care. To isolate the effect of dimensionality, we adopt a \textit{favorable extension} assumption. We assume the high-dimensional target function preserves the original smoothness parameter \(s\) of the local functions, 
so that any difference in convergence rates can be attributed solely to the increase in input dimension \(d_{\text{total}}\), not to a loss of smoothness. This allows the analysis to focus solely on the impact of the increased input dimension $d_{\text{total}}$ on convergence rates, giving the baseline the benefit of the doubt regarding smoothness.

\textit{Upper Bound Comparison.}
We now explicitly compare convergence rates by specializing the general result of Theorem~\ref{thm:finite_sample_rate}. We assume that the function approximators satisfy the capacity control condition \eqref{cond:C1} and the approximation power condition \eqref{cond:C2}, with a common approximation exponent $\alpha = \rho s / d_{\text{in}}$.

To facilitate a direct comparison based on the per-stage sample size $n$, we treat the number of stages $K$ as a fixed constant in the asymptotic analysis. Consequently, factors depending solely on $K$ are absorbed into the notation, allowing us to express the baseline's rate in terms of $n$ despite its use of $Kn$ samples.

\begin{corollary}[Finite-Sample Error Comparison]
\label{cor:upper_bound_comparison}
Under the assumptions of Theorem~\ref{thm:finite_sample_rate} and conditions \eqref{cond:C1}--\eqref{cond:C2}, let the model capacity $D$ be optimally tuned and $M$ be sufficiently large. Then, with probability at least $1-\delta$, the estimation errors for the two approaches scale as follows:

\begin{enumerate}
    \item \textbf{Structured Approach ($\algname$):} Let $\mathbf{v}^{(M)}_{\mathrm{sep}}$ be the output of Algorithm~\ref{alg:cfqi}. With sample size $n$ per stage, the global error is dominated by the stage with the highest dimension:
    \[
    \|\mathbf{v}^* - \mathbf{v}^{(M)}_{\mathrm{sep}}\|_\infty
    = \widetilde{\mathcal{O}}\left(
    n^{-\frac{\rho s}{2\rho s + d_{\max}}}
    \right).
    \]

    \item \textbf{Flattened Baseline:} Let $\mathbf{v}^{(M)}_{\mathrm{flat}}$ be the output of standard Fitted Q-Iteration on the joint space. With total sample size $Kn$ operating on the joint dimension $d_{\text{total}}$, the error bound is:
    \[
    \|\mathbf{v}^* - \mathbf{v}^{(M)}_{\mathrm{flat}}\|_\infty
    = \widetilde{\mathcal{O}}\left(
    n^{-\frac{\rho s}{2\rho s + d_{\text{total}}}}
    \right).
    \]
\end{enumerate}
Since $d_{\text{total}}$ is significantly larger than $d_{\max}$, the exponent for the structured approach is strictly more favorable. This shows that the gain from dimensionality reduction outweighs the linear increase in sample size $K$, leading to asymptotically superior performance.
\end{corollary}

\textit{Lower Bound Analysis of the Flattened Baseline.}
Having established the achievable upper bounds, we now rigorously verify that the performance limitation of the baseline is intrinsic to its architecture.
Specifically, we analyze the worst-case estimation error of the Q-function itself.

It is important to note that the upper bounds on the value error derived in Corollary \ref{cor:upper_bound_comparison} are a direct consequence of the convergence rates of the underlying Q-function estimators. Therefore, verifying the tightness of these rates requires analyzing the fundamental limits of Q-function estimation. This analysis strips away problem-specific factors like reward structures to reveal the core difficulty of learning in high-dimensional spaces. Since accurate value estimation is the foundation of any Q-learning algorithm, the inability to approximate $Q^*$ serves as a hard bottleneck for the overall performance.

To rigorously evaluate the error, we must define an appropriate norm for the pooled data set. We construct an \textit{effective data distribution} $\bar{\nu}$ as the uniform mixture of the stage-wise distributions.
We admit a slight abuse of notation here: we identify each local state-action pair $(s,a)$ sampled from $\nu_k$ with its zero-padded embedding in the joint space. Under this identification, each \((s,a)\sim\nu_k\) is zero-padded into \([0,1]^{d_{\text{total}}}\), 
so that the \(L^2(\bar{\nu})\)-norm represents the stage-averaged squared error: 
\[
\|f\|_{L^2(\bar{\nu})}^2 := \mathbb{E}_{(s,a) \sim \bar{\nu}} \left[ f(s,a)^2 \right] = \frac{1}{K} \sum_{k=1}^K \mathbb{E}_{\nu_k}\!\left[ f(s,a)^2 \right].
\]
This norm correctly captures the statistical strength of the pooled data set $\mathcal{D}$ of total size $N=Kn$.

Using this metric, we define the \textit{Worst-Case Regression Risk} for the flattened estimator $\hat{Q}$ as:
\[
\mathfrak{R}_{\text{Q}}(\hat{Q}) := \sup_{Q^* \in \mathcal{B}} \mathbb{E}_{\mathcal{D}} \left[ \|\hat{Q} - Q^*\|_{L^2(\bar{\nu})}^2 \right],
\]
where $\mathcal{B}$ denotes the target Besov ball $B^s_{p,q}$ in the joint state space.

Under this formalization, we characterize the fundamental limit using two canonical structural properties. These conditions are not arbitrary constraints but standard requirements in statistical learning theory to quantify the intrinsic complexity of approximating smooth functions:
(1) \textit{Metric Entropy} (Condition \ref{cond:L1}), ensuring the model class is sufficiently expressive to warrant a statistical analysis; and
(2) \textit{Approximation Limit} (Condition \ref{cond:L2}), capturing the unavoidable geometric error arising from compressing smooth Besov functions into a finite-capacity model.

\begin{corollary}[Worst-Case Lower Bound for Q-Function Estimation]
\label{cor:flattened_lower_bound}
Consider the flattened estimator $\hat{Q}$ derived from the model class $\mathcal{F}_D$ trained on the data set $\mathcal{D}$. Let the target function class be the Besov space $\mathcal{B} = B^s_{p,q}(\mathcal{S})$ defined on the joint state space $\mathcal{S} \subset [0,1]^{d_{\text{total}}}$. We focus on the regime satisfying the following structural properties defined with respect to the effective distribution $\bar{\nu}$:

\begin{enumerate}[label=(L\arabic*), ref=L\arabic*]
    \item \label{cond:L1}
    \textit{Metric Entropy.} There exists a constant $c_{\text{ent}} > 0$ such that for any sufficiently small $\delta > 0$, the local $\delta$-packing number satisfies:
    \[
    \log \mathcal{M}(\delta, \mathcal{F}_D, \|\cdot\|_{L^2(\bar{\nu})}) \ge c_{\text{ent}} \cdot D \log \left( \frac{1}{\delta} \right).
    \]

    \item \label{cond:L2}
    \textit{Approximation Limit.} There exists a constant $c_{\text{approx}} > 0$ such that the worst-case approximation error satisfies:
    \[
    \sup_{g \in \mathcal{B}} \inf_{f \in \mathcal{F}_D} \|g - f\|_{L^2(\bar{\nu})} \ge c_{\text{approx}} \cdot D^{-\rho s / d_{\text{total}}}.
    \]
\end{enumerate}

Then, the worst-case root-mean-square estimation error is lower-bounded by:
\[
\sup_{Q^* \in \mathcal{B}} \mathbb{E}_{\mathcal{D}}^{1/2} \left[ \|\hat{Q} - Q^*\|_{L^2(\bar{\nu})}^2 \right] = \Omega\left( n^{-\frac{\rho s}{2\rho s + d_{\text{total}}}} \right),
\]
provided that the capacity $D$ is tuned optimally. We treat the stage count $K$ as a fixed constant absorbed into the asymptotic notation to focus on the per-stage sample size $n$.
\end{corollary}

Comparing this result with the upper bound of our structured approach, we observe a clear asymptotic separation in convergence rates:
\[
\text{Upper Bound}(\algname) \ll \text{Lower Bound}(\text{Flattened FQI}),
\]
since the effective input dimension for \(\algname\) is only \(d_{\max}\), whereas the flattened baseline must contend with the full joint dimension \(d_{\text{total}}\), 
resulting in strictly slower convergence for the baseline under the same per-stage sample size \(n\).

This separation provides a formal guarantee that the proposed structural decomposition asymptotically outperforms the flattened approach, 
demonstrating that the baseline is inherently limited by the full-dimensional state space.

\begin{remark}[Utilization of Structural Information]
This theoretical gain relies on the optimal use of observable context such as stage indices, not on privileged information. $\algname$ exploits this inherent modularity to structurally mitigate the curse of dimensionality, distinct from standard variable selection methods.
\end{remark}

\subsubsection{Expected Finite-Sample Rate with Random Forests.}
While \(\algname\) is primarily analyzed with deterministic function classes \(\mathcal{F}_k\), the framework naturally accommodates stochastic approximation methods. To demonstrate this flexibility, we establish convergence guarantees for Random Forests \citep{dietterich2000experimental}.
For this analysis, we assume the \(\mathcal{U}\)-constrained optimal Q-functions \(Q_k^*\) are Lipschitz continuous, consistent with the Besov regularity in Assumption~\ref{assumption:besov}. Additionally, we require the data distribution \(\nu_k\) to be absolutely continuous \citep{biau2010layered}.

\begin{assumption}[Absolute Continuity]\label{assumption:absolute_continuity}
For each stage \(k\), the distribution \(\nu_k\) is absolutely continuous with respect to the product of the Lebesgue measure on \(\mathcal{S}_k\) and the counting measure on \(\mathcal{A}_k\), with a density bounded from above.
\end{assumption}

\begin{theorem}[Expected Finite-Sample Rate with Random Forests]\label{thm:random_forest}
Suppose Assumptions \ref{assumption:coverage} and \ref{assumption:absolute_continuity} hold, and \(Q_k^*\) are Lipschitz continuous. Under Algorithm~\ref{alg:cfqi} using Random Forest estimators, the expected suboptimality gap satisfies:
\[
\mathbb{E}_{\mathcal{D}} \|\mathbf{v}^* - \mathbf{v}^{(M)}\|_\infty = \mathcal{O} \left( \frac{\sqrt{C} H^2 Y}{(1 - \gamma_{\mathrm{cycle}})^2} \cdot \max_k n_{k}^{-\frac{0.375}{d_k \log 2 + 0.75}} \right).
\]
\end{theorem}
Theorem \ref{thm:random_forest} confirms that the cyclic geometry scales the global constant without altering the fundamental rate exponent, validating that our analysis extends to ensemble methods.

\section{Asymptotic Inference with Sieve Approximations}
\label{sec:inference}

Having established the finite-sample convergence rates, we now address the challenge of statistical inference. Specifically, we aim to construct valid confidence regions for the multivariate value vector $\mathbf{v}^*$ corresponding to the $\mathcal{U}$-constrained optimal policy. Constructing these regions in a cyclic MDP is non-trivial compared to single-stage settings. The cyclic structure induces complex stage-wise correlations, requiring us to characterize the joint \textit{asymptotic} distribution of a \textit{vector-valued} estimator and its cross-stage covariance structure.

\subsection{Sieve-based Estimation Framework}
\label{sec:inference_setup}

To address the challenge of infinite-dimensional Q-functions and their complex cyclic dependencies, we employ a linear sieve estimator such as polynomials, B-splines, or wavelets. This framework approximates the Q-functions using a finite set of basis functions and solves for the parameters jointly through a global system of estimating equations.

\subsubsection{Linear Sieve Approximation}

For each stage $k \in \{1, \dots, K\}$, we represent the state space using a sieve basis vector $\Phi_k(s_k) \in \mathbb{R}^{L_k}$:
\[
    \Phi_k(s_k) = (\Phi_{k,1}(s_k), \dots, \Phi_{k,L_k}(s_k))^{\top}.
\]
We assume the Q-function for action $a_k \in \mathcal{A}_k$ can be approximated linearly as
\[
    Q_k(s_k, a_k) \approx \Phi_k(s_k)^{\top} \beta_{k,a_k},
\]
where $\beta_{k,a_k} \in \mathbb{R}^{L_k}$ is the unknown coefficient vector.

To rigorously define the estimation target, we stack these coefficients into a stage-specific parameter vector $\beta_k \in \mathbb{R}^{L_k A_k}$, where $A_k$ denotes the cardinality of the action space $\mathcal{A}_k$:
\[
    \beta_k = (\beta_{k,1}^{\top}, \dots, \beta_{k,A_k}^{\top})^{\top}.
\]
These are further concatenated into a global parameter vector $\beta$ of total dimension $L_{\text{tot}} = \sum_{k=1}^K L_k A_k$:
\[
    \beta = (\beta_1^{\top}, \dots, \beta_K^{\top})^{\top} \in \mathbb{R}^{L_{\text{tot}}}.
\]
Due to the cyclic linkage where the value of stage $k$ depends on stage $[k+1]$, the global parameter $\beta$ must be estimated jointly rather than sequentially.

\subsubsection{Estimating Equations and Matrix Formulation}

Given a data set $\mathcal{D}$ and a target policy $\pi$, we seek an estimator $\hat{\beta} \in \mathbb{R}^{L_{\text{tot}}}$ that satisfies the orthogonality conditions of the Bellman error.

\textit{Bellman Error Components.}
First, let $T_{k,i} = \mathbb{I}[(s_k^i, a_k^i) \in \mathcal{T}_k]$ be the indicator that the $i$-th state-action pair at stage $k$ is terminal. The sample Bellman error for the linearized Q-function, denoted by $e_{k,i}(\pi, \beta) \in \mathbb{R}$, is defined as:
\begin{multline*}
    e_{k,i}(\pi, \beta) = r_{k}^{i} + \gamma_k T_{k,i} \sum_{a \in \mathcal{A}_{[k+1]}} \Phi_{[k+1]}^{\top}(\phi_k(s_k'^i)) \beta_{[k+1],a} \pi_{[k+1]}(a|\phi_k(s_k'^i)) \\
    + (1-T_{k,i}) \sum_{a \in \mathcal{A}_k} \Phi_k^{\top}(s_k'^i) \beta_{k,a} \pi_k(a|s_k'^i) - \Phi_k^{\top}(s_k^i) \beta_{k,a_k^i}.
\end{multline*}
The estimator $\hat{\beta}$ is found by solving the linear system derived from the following orthogonality condition:
\[
    e_{k,i}(\pi, \beta) \Phi_k (s_k^{i}) = \mathbf{0}_{L_k}, \quad \forall k, i.
\]

\textit{Global Feature Vectors.}
To solve this system efficiently, we define feature vectors in the global parameter space $\mathbb{R}^{L_{\text{tot}}}$. We first define the local feature vector $\psi_k(s_k, a_k) \in \mathbb{R}^{L_k A_k}$ as:
\[
    \psi_k(s_k, a_k) = \left( \Phi_k^{\top}(s_k) \mathbb{I}(a_k=1), \dots, \Phi_k^{\top}(s_k) \mathbb{I}(a_k=A_k) \right)^{\top}.
\]
Similarly, the policy-weighted feature vector $\mathbf{U}_k(s_k) \in \mathbb{R}^{L_k A_k}$ is defined as:
\begin{equation}\label{eq:U_function}
    \mathbf{U}_k(s_k) = \left( \Phi_k^{\top}(s_k) \pi_k(1 \mid s_k), \dots, \Phi_k^{\top}(s_k) \pi_k(A_k \mid s_k) \right)^{\top}.
\end{equation}

For a sample $i$ at stage $k$, we construct three block-sparse global vectors of dimension $L_{\text{tot}}$. The \textit{current feature vector} $\psi_{k,i} \in \mathbb{R}^{L_{\text{tot}}}$ places the local feature in the $k$-th block:
\[
    \psi_{k,i} = \bigl(\mathbf{0}^{\top}, \dots, \underbrace{\psi_k(s_k^i, a_k^i)^{\top}}_{k\text{-th block}}, \dots, \mathbf{0}^{\top}\bigr)^{\top}.
\]
To capture transitions, we define the \textit{policy-weighted next-state features}. The vector $\mathbf{U}_{k,i} \in \mathbb{R}^{L_{\text{tot}}}$ corresponds to the next state within the current stage, which is active when $T_{k,i}=0$:
\begin{equation}\label{eq:def_U_vec}
    \mathbf{U}_{k,i} = \bigl(\mathbf{0}^{\top}, \dots, \underbrace{\mathbf{U}_k(s_k'^i)^{\top}}_{k\text{-th block}}, \dots, \mathbf{0}^{\top}\bigr)^{\top}.
\end{equation}
Conversely, $\mathbf{U}'_{k,i} \in \mathbb{R}^{L_{\text{tot}}}$ corresponds to the transition to the subsequent stage $[k+1]$, which becomes relevant when the transition is terminal ($T_{k,i}=1$):
\begin{equation}\label{eq:def_U_prime_vec}
    \mathbf{U}'_{k,i} = \bigl(\mathbf{0}^{\top}, \dots, \underbrace{\mathbf{U}_{[k+1]}(\phi_k(s_k'^i))^{\top}}_{[k+1]\text{-th block}}, \dots, \mathbf{0}^{\top}\bigr)^{\top}.
\end{equation}

\textit{Global Linear System.}
The system of estimating equations can be written in a compact matrix form as $\widehat{\mathbf{H}} \hat{\beta} = b$. The empirical matrix $\widehat{\mathbf{H}} \in \mathbb{R}^{L_{\text{tot}} \times L_{\text{tot}}}$ and vector $b \in \mathbb{R}^{L_{\text{tot}}}$ are explicitly constructed as:
\begin{align}
    \widehat{\mathbf{H}} &= \frac{1}{n} \sum_{k=1}^K \sum_{i=1}^{n_k} \psi_{k,i} \left( \psi_{k,i} - (1-T_{k,i}) \mathbf{U}_{k,i} - \gamma_k T_{k,i} \mathbf{U}'_{k,i} \right)^{\top}, \label{eq:H_matrix_def} \\[1ex]
    b &= \frac{1}{n} \sum_{k=1}^K \sum_{i=1}^{n_k} \psi_{k,i} \, r_{k}^i, \nonumber
\end{align}
where $n = \sum_{k=1}^K n_k$. The solution is given by $\hat{\beta} = \widehat{\mathbf{H}}^{-1} b$. 

The structure of $\widehat{\mathbf{H}}$ explicitly encodes the cyclic transition dynamics. Specifically, the diagonal blocks capture within-stage relationships via $\psi_{k,i}$ and $\mathbf{U}_{k,i}$, while the off-diagonal blocks $(k, [k+1])$ are populated by $\mathbf{U}'_{k,i}$ to enforce the connectivity between stages.

\subsubsection{Value and Covariance Estimation}

With the estimated parameters $\hat{\beta}$, we proceed to estimate the value vector $\mathbf{v}^* \in \mathbb{R}^K$ and its asymptotic covariance.

\textit{Value Estimator.}
We define the random \textit{policy-weighted feature matrix} $\mathbf{U} \in \mathbb{R}^{L_{\text{tot}} \times K}$ as a block-diagonal matrix constructed from the stage-wise feature vectors:
\[
    \mathbf{U} = \begin{bmatrix}
    \mathbf{U}_1(s_1) & \mathbf{0} & \cdots & \mathbf{0} \\
    \mathbf{0} & \mathbf{U}_2(s_2) & \cdots & \mathbf{0} \\
    \vdots & \vdots & \ddots & \vdots \\
    \mathbf{0} & \mathbf{0} & \cdots & \mathbf{U}_K(s_K)
    \end{bmatrix}.
\]
The estimated value vector $\hat{\mathbf{v}}_{\mathcal{D}}(\pi) \in \mathbb{R}^K$ is obtained by projecting $\hat{\beta}$ onto the expected feature space $\mathbb{E}[\mathbf{U}]$, where the expectation is taken element-wise with respect to the initial state distributions $s_k \sim \eta_k$:
\begin{equation} \label{eq:mean_estimate}
    \hat{\mathbf{v}}_{\mathcal{D}}(\pi) = \mathbb{E}[\mathbf{U}]^\top \hat{\beta}.
\end{equation}

\textit{Covariance Estimator.}
The asymptotic covariance matrix $\widehat{\Sigma}_{\mathcal{D}}(\pi) \in \mathbb{R}^{K \times K}$ is estimated using the sandwich form:
\begin{equation} \label{eq:variance_estimate}
    \widehat{\Sigma}_{\mathcal{D}}(\pi) = \mathbb{E}[\mathbf{U}]^\top \, \widehat{\mathbf{H}}^{-1} \, \widehat{\Omega} \, (\widehat{\mathbf{H}}^{-1})^\top \, \mathbb{E}[\mathbf{U}].
\end{equation}
Here, $\widehat{\Omega} \in \mathbb{R}^{L_{\text{tot}} \times L_{\text{tot}}}$ represents the empirical variance of the Bellman residuals. Let $\hat{e}_{k,i} = e_{k,i}(\pi, \hat{\beta})$ be the sample residual evaluated using the explicit Bellman error formula. The matrix $\widehat{\Omega}$ is defined as:
\[
    \widehat{\Omega} = \frac{1}{n} \sum_{k=1}^K \sum_{i=1}^{n_k} \hat{e}_{k,i}^2 \, \psi_{k,i} \psi_{k,i}^\top.
\]
This formulation accounts for the heteroscedasticity inherent in the cyclic decision process and provides the basis for the valid confidence regions constructed in Algorithm~\ref{alg:ensemble_cfqi}.

\subsubsection{Ensemble Evaluation Procedure}
To ensure the asymptotic independence required for valid inference, we implement these matrix-based estimators using an $N$-fold ensemble procedure (Algorithm~\ref{alg:ensemble_cfqi}).

\begin{algorithm}[ht]
\caption{Ensemble Value Evaluation for Cyclic Fitted Q-Iteration}
\label{alg:ensemble_cfqi}
\begin{algorithmic}[1]
    \State \textbf{Input}: data set $\mathcal{D}$, initial distributions $ \{\eta_k\} $, splits $ N \ge 2 $, sieve bases $ \{\Phi_k\} $.
    \State \textbf{Divide data sets}: Partition $ \mathcal{D} $ into $ N $ disjoint subsets $ \mathcal{D}^1, \dots, \mathcal{D}^N $. Let $ \bar{\mathcal{D}}^n = \bigcup_{m=1}^n \mathcal{D}^m$.
    \State \textbf{Iterative Estimation}: For $ n = 1, \dots, N-1 $:
    \State \quad Learn policy $ \hat{\pi}^n $ using Algorithm~\ref{alg:cfqi} on cumulative data $ \bar{\mathcal{D}}^n $.
    \State \quad Compute estimates $ \hat{\mathbf{v}}_n = \hat{\mathbf{v}}_{\mathcal{D}^{n+1}}(\hat{\pi}^n) \in \mathbb{R}^K$ and $ \widehat{\Sigma}_n = \widehat{\Sigma}_{\mathcal{D}^{n+1}}(\hat{\pi}^n) \in \mathbb{R}^{K \times K}$ using the evaluation data set $ \mathcal{D}^{n+1} $ via Equations \ref{eq:mean_estimate}--\ref{eq:variance_estimate}.
    \State \textbf{Aggregate}: Compute precision-weighted average:
    \[
    \widehat{\Sigma}^{-1/2} = \frac{1}{N-1} \sum_{n=1}^{N-1} \widehat{\Sigma}_n^{-1/2}, \quad
    \hat{\mathbf{v}} = \frac{1}{N-1} \widehat{\Sigma}^{1/2} \sum_{n=1}^{N-1} \widehat{\Sigma}_n^{-1/2} \hat{\mathbf{v}}_n.
    \]
    \State \textbf{Output}: Estimator $ \hat{\mathbf{v}} $ and covariance matrix $ \widehat{\Sigma} $.
\end{algorithmic}
\end{algorithm}

\subsection{Asymptotic Normality and Validity of Inference}
\label{sec:inference_theory}

We now turn to the theoretical analysis of the estimator $\hat{\mathbf{v}}$ produced by the ensemble procedure (Algorithm \ref{alg:ensemble_cfqi}). Although Theorem \ref{thm:finite_sample_rate} provides error bounds, they are insufficient for practical inference for two key reasons. First, the rates may be slower than \(n^{-1/2}\) without strict realizability assumptions. Second, and more importantly, the bounds depend on non-pivotal constants such as distribution coverage factors, making it impossible to construct computable confidence regions.

To overcome these limitations and establish standard \(\sqrt{n}\)-rate asymptotic normality for the ensemble estimator, we require stricter regularity conditions that control the approximation bias and variance. We adapt the framework established by \citet{shi2022statistical} to our cyclic setting, imposing specific constraints on the margin, eigenvalue, and smoothness properties.

We begin by imposing a margin condition to characterize the separation between the optimal and suboptimal action values. This condition ensures that the suboptimality gap induced by the policy estimation decays faster than the standard \(n^{-1/2}\) rate, which is essential for valid inference.

\begin{assumption}[Margin Condition]\label{assumption:margin}
For each stage \(k\), define the Q-margin \(\Delta_k(s_k)\) at state \(s_k\) as:
\[
\Delta_k(s_k) := \max_{a \in \mathcal{A}_k} Q_k^*(s_k, a) - \max_{a \in \mathcal{A}_k \setminus \{ \arg\max_{a'} Q_k^*(s_k, a') \}} Q_k^*(s_k, a).
\]
We assume \(\mathbb{P}[\Delta_k(s_k) \leq \epsilon] = \mathcal{O}(\epsilon^\alpha)\) for some \(\alpha > 0\), holding under both the initial distribution \(\eta_k\) and the Lebesgue measure.
\end{assumption}

In addition to the margin condition, we require the population covariance matrix of the features to be well-behaved. This guarantees that the sieve coefficient matrix $\widehat{\mathbf{H}}$ in the estimation procedure remains invertible with high probability.

\begin{assumption}[Minimum Eigenvalue Condition]\label{assumption:min_eigenvalue}
There exists a constant $c > 0$ such that:
\[
    \lambda_{\min} \left( \sum_{k=1}^K p_k \mathbb{E} \left[ \psi_{k} \psi_{k}^{\top} - (1-T_{k}) \mathbf{u}_{k} \mathbf{u}_{k}^{\top} - \gamma_k^2 T_{k} \mathbf{u}'_{k} (\mathbf{u}'_{k})^{\top} \right] \right) \ge c,
\]
where $p_k = \lim n_k/n$. The expectation is taken with respect to the data generating distribution for $(s_k, a_k, r_k, s'_k, T_k)$. Here, $\psi_k = \psi_{k,i}$ denotes the random feature vector, and the vectors $\mathbf{u}_k, \mathbf{u}'_k \in \mathbb{R}^{L_{\text{tot}}}$ represent the conditional expected feature embeddings defined as:
\begin{align*}
    \mathbf{u}_{k} &= \left( \mathbf{0}^{\top}, \dots, \mathbb{E}[\mathbf{U}_k(s'_k) \mid s_k, a_k]^{\top}, \dots, \mathbf{0}^{\top} \right)^{\top}, \\
    \mathbf{u}'_{k} &= \left( \mathbf{0}^{\top}, \dots, \mathbb{E}[\mathbf{U}_{[k+1]}(\phi_k(s'_k)) \mid s_k, a_k]^{\top}, \dots, \mathbf{0}^{\top} \right)^{\top}.
\end{align*}
The non-zero blocks for $\mathbf{u}_k$ and $\mathbf{u}'_k$ are located at the $k$-th and $[k+1]$-th positions, respectively.
\end{assumption}

Finally, we generalize the Besov regularity assumption. Since our inference procedure involves evaluating estimated policies \(\hat{\pi}\), we require that the Q-functions of any \(\mathcal{U}\)-constrained policy—not just the optimal one—are well-approximated by the sieve basis.

\begin{assumption}[Extended Besov Regularity]\label{assumption:besov_general}
For any \(\mathcal{U}\)-constrained policy \(\pi\), the Q-function \( Q_k^{\pi}(\cdot, a_k) \) belongs to the Besov space \( B^{s_k}_{p_k, q_k}(\mathcal{S}_k) \) with smoothness \(s_k > d_k\).
\end{assumption}

Based on these structural assumptions, we establish the main normality result.

\begin{theorem}[Asymptotic Normality]\label{thm:asymptotic}
Let \(n = \sum n_k\) be the total sample size with \(n_k/n \to r_k \in (0,1)\). Suppose that Assumptions \ref{assumption:coverage}, \ref{assumption:margin}, \ref{assumption:min_eigenvalue}, and \ref{assumption:besov_general} hold.
By choosing the sieve dimension \(L_k \asymp n_k^{d_k / (2s_k + d_k)}\), and provided that the policy estimation error satisfies \(\|\hat{Q} - Q^*\|_{2, \nu} = O_p(n^{-b_*})\) with \(b_* > 1/4\) and the margin exponent satisfies \(\alpha > \frac{2 - 4b_*}{4b_* - 1}\), the estimators \((\hat{\mathbf{v}}, \widehat{\Sigma})\) produced by Algorithm \ref{alg:ensemble_cfqi} satisfy:
\begin{align*}
\sqrt{n(N-1)/N} \: \widehat{\Sigma}^{-1/2} (\hat{\mathbf{v}} - \mathbf{v}^{(n)}) &\xrightarrow{d} \mathcal{N}(\mathbf{0}_K, \mathbf{I}_K), \\
\sqrt{n(N-1)/N} \: \widehat{\Sigma}^{-1/2} (\hat{\mathbf{v}} - \mathbf{v}^*) &\xrightarrow{d} \mathcal{N}(\mathbf{0}_K, \mathbf{I}_K),
\end{align*}
where \(\mathbf{v}^{(n)}\) is the value of the estimated policy and \(\mathbf{v}^*\) is the optimal value.
\end{theorem}

Theorem \ref{thm:asymptotic} enables the construction of a valid \((1-\delta)\) simultaneous confidence region for the constrained optimal value vector \(\mathbf{v}^*\), given by:
\[
\mathcal{C}_{1-\delta} = \left\{ \mathbf{v} \in \mathbb{R}^K : n ( \mathbf{v} - \hat{\mathbf{v}} )^{\top} \widehat{\Sigma}^{-1} ( \mathbf{v} - \hat{\mathbf{v}} ) \leq \chi^2_{K, 1-\delta} \right\}.
\]
The validity of this region relies on the rate conditions \(b_* > 1/4\) and \(\alpha\). These conditions are not arbitrary constraints but standard requirements in non-parametric inference to ensure that the bias induced by policy optimization vanishes asymptotically relative to the variance ($n^{-1/2}$), thereby allowing the estimation error to be safely ignored in the limit.

\section{Experimental Results}
To corroborate our theoretical findings and demonstrate practical utility, we evaluate our proposed algorithm, \(\algname\), against FQI on synthetic and real-world Type 1 Diabetes (T1D) data sets for multi-stage adaptive glucose management. Both methods use random forest regressors, trained for 100 iterations, with tree counts tuned from \{100, 200, 300\}. \(\algname\) trains a separate model per stage, leveraging the cyclic MDP structure. Standard FQI, designed for single MDPs, is adapted by concatenating stage-wise state-action spaces and zero-padding irrelevant entries. This inflates the state-action space, complicating learning and potentially obscuring stage-specific dynamics that \(\algname\) captures. Full experimental details are provided in Appendix~\ref{appendix:experiment}.

\subsection{Simulation Study}\label{subsection:simulation}
\subsubsection{Policy Optimization Performance}
\label{sec:sim_performance}
We simulate a four-stage cyclic MDP to model daily glucose dynamics, with stages defined as morning (6:00--11:00), day (11:00--17:00), evening (17:00--22:00), and night (22:00--6:00). Negative rewards penalize hyperglycemia and hypoglycemia to maintain in-range glucose levels. Both \(\algname\) and FQI are trained on two update sets (All stages, Day/Evening only), using fixed policies that select actions uniformly at random. A discount factor of 1 is applied for morning, day, and evening, and 0.9 for night to reflect daily cyclicity.  See Appendix~\ref{appendix:simulation} for complete simulation details.

Table~\ref{table:synthetic_result} reports cumulative rewards over 50 simulated days, averaged across 100 trials with stage-wise sample sizes of \{100, 200, 500\}. \(\algname\) consistently outperforms FQI and a random policy, adeptly adapting to stage-specific dynamics. Conversely, FQI struggles with stage-wise dynamics, occasionally performing worse than the random policy.

\begin{table}[ht]
\centering
{\fontsize{9}{11}\selectfont
\begin{tabular}{ccccccc} 
\toprule
Method & \multicolumn{2}{c}{$n=100$} & \multicolumn{2}{c}{$n=200$} & \multicolumn{2}{c}{$n=500$} \\ 
\cmidrule(lr){2-3} \cmidrule(lr){4-5} \cmidrule(lr){6-7}  
 & All & Day/Evening & All & Day/Evening & All & Day/Evening \\
\midrule
\(\algname\) & -41.7 (2.7) & \textbf{-39.8} (1.7) & -54.0 (2.1) & -93.8 (2.3) & \textbf{-41.6} (3.4) & -57.0 (4.0) \\
FQI & -350.7 (14.0) & -187.7 (10.8) & -119.6 (6.2) & -96.7 (2.3) & -336.5 (5.0) & -130.6 (3.6) \\
Random & \multicolumn{6}{c}{-259.5 (3.8)} \\
\bottomrule
\end{tabular}
}
\caption{Mean (Standard Error) of simulated cumulative rewards over 50 days for \(\algname\), FQI, and a random policy across stages, over 100 repeats, with sample size \(n\) and update sets.}
\label{table:synthetic_result}
\end{table}

\subsubsection{Validity of Statistical Inference}
\label{sec:sim_inference}

To validate the asymptotic normality results presented in Theorem \ref{thm:asymptotic} and assess the coverage properties of the confidence regions, we conducted a simulation study on a heterogeneous cyclic environment. This experiment is designed to verify the statistical inference procedure in a controlled setting where the convergence rate of the nuisance parameters satisfies the theoretical requirement of $O_p(n^{-\beta})$ for $\beta > 1/4$.

The environment is a 3-stage ($K=3$) cyclic MDP with heterogeneous state dimensions cycling as $d_1=1, d_2=2, d_3=2$. The transition dynamics and rewards follow a linear structure:
\[
    s_{k+1} = A_k s_k + B_k a_k + \xi_k, \quad r_k(s, a) = w_k^\top s + u_{k,a}.
\]
To ensure the system is well-defined, the model coefficients were sampled from fixed distributions as follows:
\begin{align*}
    A_k &\sim \mathcal{U}[-0.3, 0.3]^{d_{k+1} \times d_k}, & B_k &\sim \mathcal{U}[-0.3, 0.3]^{d_{k+1}}, \\
    w_k &\sim \mathcal{U}[-0.5, 0.5]^{d_k}, & u_{k, a} &\sim \mathcal{N}(0, 0.5^2).
\end{align*}
The noise is Gaussian $\xi_k \sim \mathcal{N}(0, 0.1^2 I_{d_{k+1}})$. We employed a quadratic basis 
\[
    \Phi_k(s) = [1, s^\top, \text{vec}(s s^\top)^\top]^\top
\]
for both the policy learning step (Algorithm~\ref{alg:cfqi}) and the subsequent evaluation step (Algorithm~\ref{alg:ensemble_cfqi}). This choice ensures sufficient representation capacity for the value function, thereby controlling the approximation bias as required in Theorem \ref{thm:asymptotic}.

We constructed the data set $\mathcal{D}$ of total size $n$ by independently collecting $n/K$ transition tuples for each stage $k \in \{1, \dots, K\}$. For each sample, the state was first drawn uniformly from a bounded domain, and the action was subsequently selected according to a uniform behavior policy:
\[
    s_k \sim \mathcal{U}[-2, 2]^{d_k}, \quad a_k \sim \text{Bernoulli}(0.5).
\]
We then applied an $N=2$ fold sample splitting procedure to these stage-wise data sets. For every stage, the first fold was used to estimate the target policy, while the second fold was reserved for the inference task. 

The estimation target is the value vector $\mathbf{v}^* = [v^*_1, v^*_2, v^*_3]^\top$, where each component $v^*_k = \mathbb{E}_{s \sim \eta_k}[V^\pi(s)]$ is defined with respect to a specific initial distribution $\eta_k$. In this experiment, we set the evaluation distribution identical to the data sampling distribution, given as $\eta_k = \mathcal{U}[-2, 2]^{d_k}$. To evaluate the accuracy of our estimators, the ground truth $\mathbf{v}^*$ was computed via Monte Carlo estimation using $M=50,000$ trajectories drawn from $\eta_k$.

We examined the joint asymptotic normality of the estimator $\hat{\mathbf{v}} \in \mathbb{R}^3$ over $T=200$ independent trials. The effective sample size used for variance estimation is given by $n (N-1)/N$, which amounts to $n/2$ in this configuration. Accordingly, we computed the Mahalanobis distance statistic:
\[
    D^2 = \frac{n(N-1)}{N} (\hat{\mathbf{v}} - \mathbf{v}^*)^\top \widehat{\Sigma}^{-1} (\hat{\mathbf{v}} - \mathbf{v}^*),
\]
where $\widehat{\Sigma}$ is the estimated asymptotic covariance matrix. Under the null hypothesis, $D^2$ follows a Chi-squared distribution with $K=3$ degrees of freedom ($\chi^2_3$).

Figure \ref{fig:inference_plots} presents the diagnostic plots for the statistical inference with a sample size of $n=2400$. The left panel displays the Q-Q plot of the empirical $D^2$ values against the theoretical quantiles of the $\chi^2_3$ distribution. The plot shows a close alignment with the diagonal line, confirming that the sampling distribution of the estimator is well-approximated by the predicted Gaussian distribution. The right panel shows the scatter plot of the estimated values $\hat{\mathbf{v}}$ projected onto the first two dimensions ($v_1, v_2$). The estimates are distributed around the ground truth $\mathbf{v}^*$ (blue star), and the empirical mean (green star) coincides closely with the true parameter, illustrating the unbiasedness of the estimator.

Table \ref{tab:coverage_results} summarizes the empirical joint coverage probabilities and Mean Squared Errors (MSE) for total sample sizes $n \in \{600, 1200, 1500, 2400\}$. We observe that the empirical coverage, calculated as the proportion of trials where $D^2 \le \chi^2_{3, 0.95}$, consistently approaches the nominal 95\% confidence level as the sample size increases. Furthermore, the MSE, defined as the squared Euclidean norm of the estimation error $\|\hat{\mathbf{v}} - \mathbf{v}^*\|_2^2$, exhibits a steady decrease as the sample size grows, demonstrating the consistency of our estimator.

\begin{figure}[t]
    \centering
    \includegraphics[width=0.95\textwidth]{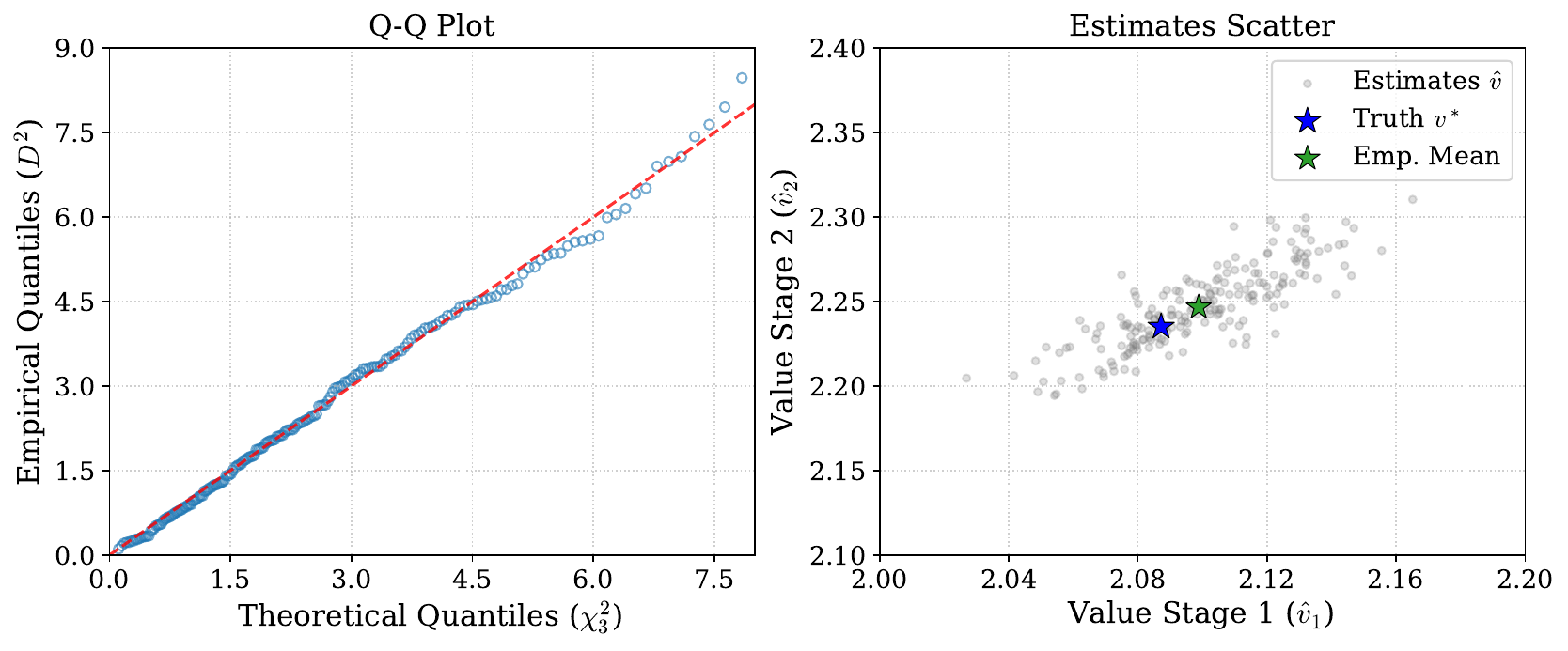}
    \caption{Diagnostic plots for statistical inference based on $n=2400$ samples, each estimated across 200 independent trials. The left panel presents a Q-Q plot of the empirical squared error statistic $D^2$ against the theoretical $\chi^2_3$ quantiles. The right panel shows a scatter plot of the estimates $(\hat{v}_1, \hat{v}_2)$ in relation to the ground truth $\mathbf{v}^*$ (blue star) and the empirical mean (green star).}

    \label{fig:inference_plots}
\end{figure}

\begin{table}[ht]
\centering
\setlength{\tabcolsep}{18pt} 
\renewcommand{\arraystretch}{1.2} 
\begin{tabular}{ccc}
\toprule
Total Sample Size ($n$) & Coverage (\%) & MSE ($\times 10^{-3}$) \\ 
\midrule
600  & 99.5 & 2.8225 \\
1200 & 97.0 & 1.5138 \\
1500 & 96.0 & 1.2529 \\
2400 & 95.0 & 0.7117 \\
\bottomrule
\end{tabular}
\caption{Empirical joint coverage probabilities and Mean Squared Error (MSE) under varying sample sizes, computed over $T=200$ independent trials. Coverage reports the proportion of trials where the ground truth falls within the 95\% confidence region ($D^2 \le \chi^2_{3, 0.95}$), and MSE measures the estimation accuracy of the value vector.}
\label{tab:coverage_results}
\end{table}

\subsection{Real-World Experiments: T1D Behavioral and Clinical data set}\label{subsection:real_experiment}
To assess our method on real-world data, we used a de-identified data set from an IRB-approved study at an anonymous U.S. medical center, which collected data from 25 elderly T1D patients over 33 non-consecutive days. The data set includes 24-hour dietary and activity recalls, glucose and insulin levels, accelerometer readings, and clinical and demographic information.

For offline RL, we divide each day into four stages (as in the simulation) with a discount factor of 0.9 at night and negative rewards for out-of-range glucose. Stage-specific states include demographic (sex, weight), CGM readings (glucose level, rate of change), macronutrient intake, and actions (meals, insulin doses, physical activity above MET 3). Decision points occur every 30 minutes for morning, day, and evening (actions affecting 1 hour), and once at 22:00 for night, transitioning to 6:00 the next day. The data set is split 8:2 into training and test sets. Details are described in Appendix \ref{appendix:real}.

Figure~\ref{fig:experiment_values} shows box-plots comparing the distribution of estimated values \(\widehat{V}_k(s_k)\) for \(\algname\) and Flattened FQI, trained on the training set and evaluated on test set states across four stages under four update sets: All, Morning/Day, Evening/Night, and None (non-updated stages follow random policies). \(\algname\) achieves higher estimated values, particularly in the `All' update setting, while Flattened FQI converges to lower values, indicating suboptimality in finite-sample real-world scenarios.

To evaluate \(\algname\)'s potential despite limited samples and no online testing, we conduct 100 bootstrap iterations on the training set. Figure~\ref{fig:bootstrap_histogram} displays the distribution of estimated values \(\widehat{V}_k\) on test samples, with initial state distributions \(\eta_k\) set to stage start times (e.g., 6:00 for morning) and other variables matching the test set. We calculate average stage-wise rewards from the test set, applying discounted sums to estimate cumulative rewards under clinically guided human policies (red line). The bootstrap mean (green line) exceeds test rewards, with 90\% intervals (black lines) suggesting potential to improve beyond clinical baselines, except the morning stage has greater variability.

\begin{figure}[t!]
    \centering
    \includegraphics[width=0.99\textwidth]{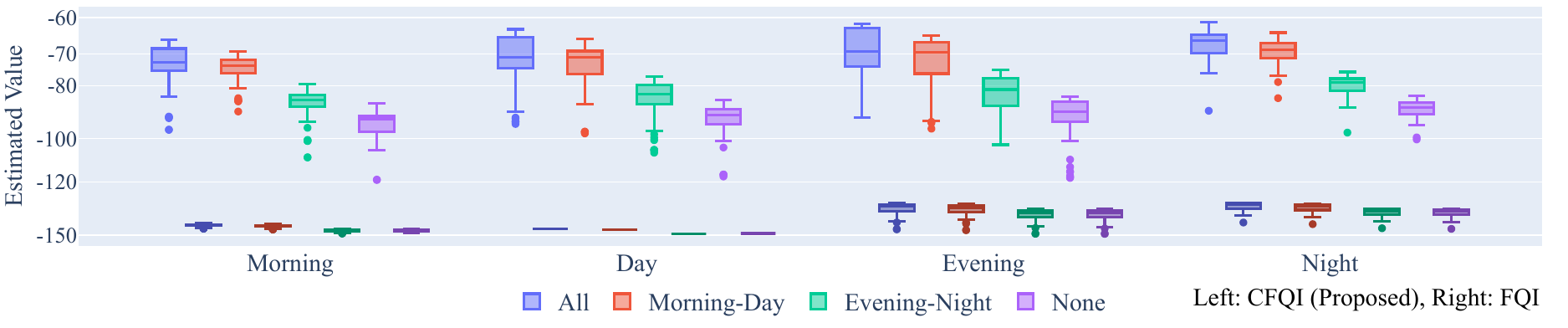}
    \caption{Box plots of estimated values $\widehat{V}_k(s_k)$ at test sample states for each stage under four update sets in the T1D data set. 
Colors indicate update sets. For each time period and update set, two box plots are shown for comparison; the algorithm corresponding to each box is indicated in the legend (left: $\algname$; right: flattened FQI).}
    \label{fig:experiment_values}
\end{figure}

\begin{figure}[t!]
    \centering
    \includegraphics[width=0.99\textwidth]{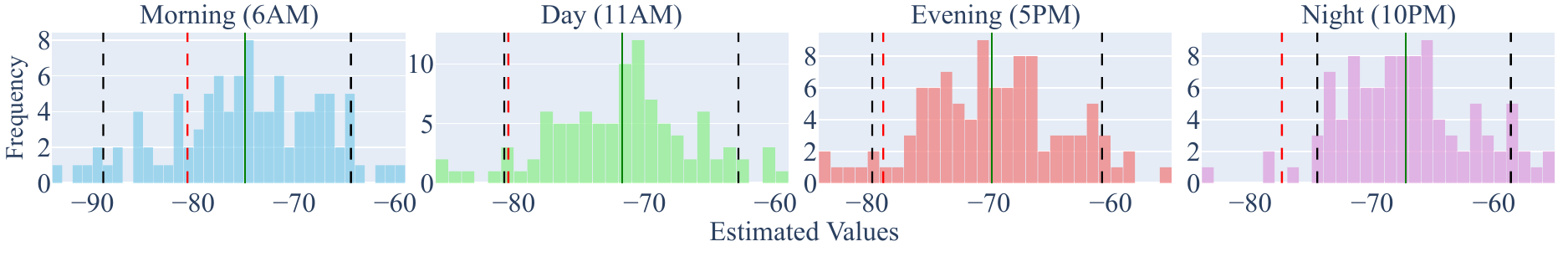}
    \caption{Bootstrap distribution of estimated values \(\widehat{V}_k\) at stage starts in T1D analysis. Red line: average observed test cumulative reward; green line: bootstrap mean; black lines: 90\% intervals.}
    \label{fig:bootstrap_histogram}
\end{figure}

\section{Discussion and Conclusion}

Our analysis relies on the uniform coverage assumption (Assumption \ref{assumption:coverage}), which generalizes standard concentrability coefficients widely established in the offline RL literature \citep{chen2019information, munos2003error, munos2008finite} to the cyclic setting. This assumption serves as a canonical baseline to characterize the fundamental learnability of Cyclic MDPs without the confounding factors of exploration. While recent advances in single-stage RL have explored relaxing coverage requirements via pessimism, establishing the minimax rates under this standard assumption provides the necessary theoretical foundation for the proposed framework. Future work may extend this analysis by incorporating pessimistic penalties to address scenarios with partial coverage, leveraging the modular structure of \(\algname\).

In this work, we introduced \(\algname\), an offline reinforcement learning algorithm designed for a newly formalized class of structured cyclic Markov Decision Processes (MDPs), characterized by heterogeneous stages with distinct dynamics, rewards, and state-action spaces. \(\algname\) employs a vector of interconnected, stage-specific Q-functions optimized through a set of coupled Bellman equations. This architecture provides the flexibility to learn policies for all stages simultaneously or to target specific subsets while conditioning on fixed policies for the remaining stages, thereby improving practical applicability in settings with partial system controllability.

We provided comprehensive theoretical guarantees for \(\algname\), including finite-sample suboptimality bounds and convergence rates under Besov regularity assumptions for various function approximators, such as random forests. Additionally, we extended a sieve-based methodology for asymptotic inference to our setting, establishing the normality of value estimates and enabling the construction of confidence regions. Empirical evaluations on both synthetic and real-world Type 1 Diabetes data sets underscore the effectiveness and adaptability of \(\algname\) in capturing complex cyclic dependencies and optimizing policies within multi-stage environments.



\acks{The first and last authors were supported in part by the National Center for Advancing Translational Sciences (NCATS) of the National Institutes of Health under Grant Award Number UM1TR004406. The second author was supported by the National Institute of Diabetes and Digestive and Kidney Diseases under Grant Award Number 1R21DK125033. The last author was also supported in part by the National Science Foundation under Grant DMS-2210659. The content is solely the responsibility of the authors and does not necessarily represent the official views of the National Institutes of Health or the National Science Foundation.}



\newpage
\appendix

\section{Discussion on Approximation Error \texorpdfstring{\(\epsilon_{\mathrm{approx},k}\)}{e approx,k} and Table \ref{tab:approx_rates_covering}}\label{sec:discussion_approx}

\textbf{Approximation Error \(\epsilon_{\mathrm{approx},k}\).}
This section discusses the context and implications of the approximation error analysis and the rates presented in Table~\ref{tab:approx_rates_covering}. The \(L_2\)-approximation error is defined in Equation~\ref{eq:approx_error} as
\[
\epsilon_{\mathrm{approx}, k} := \inf_{f \in \mathcal{F}_k} \normtwo[,\nu_k]{Q_k^* - f},
\]
where \(\nu_k\) denotes the stage-specific data distribution. The dependence on this distribution \(\nu_k\) makes direct analysis using standard approximation theory challenging.

To facilitate theoretical analysis, we therefore shift our focus towards controlling a distribution-independent uniform error over the state space, across all actions. Specifically, we often seek to bound the best uniform approximation achievable by a function \(f\) within the class \(\mathcal{F}_k\), represented by quantities like
\[
\inf_{f \in \mathcal{F}_k} \sup_{a_k \in \action_k} \norminf{f(\cdot , a_k) - Q_k^*(\cdot, a_k)}.
\]
This type of uniform control is feasible because the state space \(\state_k \subseteq [0,1]^{d_k}\) is compact and the action space \(\action_k\) is finite.

Assumption~\ref{assumption:besov} is crucial for this approach. It ensures that for each fixed action \(a_k\), the function \(s_k \mapsto Q_k^*(s_k, a_k)\) possesses sufficient smoothness, lying in a Besov space \(B^{s_k}_{p_k, q_k}(\state_k)\). To rigorously justify relating sup-norm approximation results to \(L_\infty\) control via function space embeddings, we typically require that the smoothness and integrability parameters satisfy the Sobolev embedding condition \(\dfrac{1}{p_k} < \dfrac{s_k}{d_k}\)~\citep{adams2003sobolev}.

\textbf{\(\epsilon\)-covering Number \(\mathcal{N}_{\mathcal{F}_k}(\epsilon)\).} Assumption \ref{assumption:besov} enables bounding the covering number for the functions for a fixed action \(a_k\), considering them as functions on the state space, applied for the function approximation methods described in Table \ref{tab:approx_rates_covering}.

To determine the covering number for the entire function class \(\mathcal{F}_k\) over the state-action space, we then combine these covering number bounds derived for each fixed action. Given the finite number of actions at stage \(k\), the log-covering number bound for \(\mathcal{F}_k\) is obtained by multiplying the corresponding bound applicable to fixed-action functions defined on the state space by the number of actions at stage \(k\).

\textbf{Discussion on Table \ref{tab:approx_rates_covering}.}
The approximation rates \(\epsilon_{\mathrm{approx},k} = \mathcal{O}(D_k^{-\alpha_k})\) presented in Table~\ref{tab:approx_rates_covering} are typically derived under such uniform error control frameworks. Here, \(D_k\) refers to a complexity parameter of the function class \(\mathcal{F}_k\), such as its degrees of freedom or model size, for example, the number of basis functions or network parameters. The term \(\alpha_k\) represents the rate exponent determined by the interplay of function smoothness (related to \(s_k, p_k\)), input dimension (\(d_k\)), and the chosen approximation method.

For classical approximation techniques such as algebraic or trigonometric polynomials, B-splines, and wavelets, the approximation rate in Besov or Sobolev spaces is well-studied~\citep{ciesielski1982spline, daubechies1992ten, devore1993constructive, hardle2012wavelets, lorentz1996constructive, oswald1990degree}. These methods typically yield rates for the sup-norm error of the form \(\mathcal{O}(D_k^{-s_k/d_k})\), where \(D_k\) often corresponds to the number of basis functions. The log-covering number for such function classes is often proportional to \(D_k\), assuming bounded coefficients for the basis expansion.

This classical approximation theory also provides the foundation for the linear sieve approximators used for Q-functions within our inference procedure detailed in Section~\ref{sec:inference}.

Gaussian radial basis functions (RBFs) offer an alternative approximation approach. These methods allow flexible placement of basis centers and control over scales within compact domains. Approximation guarantees for RBFs are sometimes derived via connections to wavelet theory~\citep{hangelbroek2010nonlinear}, and their log-covering numbers also typically grow linearly in the number of basis functions \(D_k\).

Recent studies have extended approximation error analysis to encompass deep learning models. For fully connected networks using the ReLU activation function, reference \citep{siegel2023optimal} demonstrates approximation capabilities for Besov functions, potentially achieving rate exponents \(\alpha_k = 2s_k/d_k\), influenced by network depth. Complementary work by \citep{ou2024coveringnumbersdeeprelu} analyzes the log-covering numbers of these ReLU networks, finding scaling proportional to the network width squared times its depth, a quantity related to the total number of parameters \(D_k\).

Transformer architectures have also been analyzed within this approximation theory framework. Research such as \citep{takakura2023approximation, wada2023approximation} investigates their approximation power and covering numbers on Besov spaces, typically under specific assumptions regarding data sparsity, coefficient norm constraints, network width, and depth. It is noteworthy that when configured appropriately, for instance by setting the input sequence length to 1, Transformer networks can function as powerful universal approximators, aligning with the approximation regimes discussed in Table~\ref{tab:approx_rates_covering}.

\textbf{Extension beyond deterministic approximators.}
Although our core analysis focuses on deterministic approximation using functions from well-defined classes \(\mathcal{F}_k\), alternative stochastic approximation methods like Random Forests can also be employed in practice. These methods often operate without explicit reliance on a pre-specified function class structure and may still achieve consistent policy learning under appropriate conditions. As noted following Theorem~\ref{thm:random_forest}, the proposed \(\algname\) algorithm maintains compatibility with such flexible learners, enhancing its practical applicability.

\section{Experimental Details}\label{appendix:experiment}

\subsection{Simualtional Study Environment Details}\label{appendix:simulation}
The simulational study from subsection \ref{subsection:simulation} uses a synthetic environment simulating a physiological process with \( K = 4 \) daily stages (Morning, Day, Evening, Night) to evaluate reinforcement learning algorithms for glycemic control. The environment is characterized by stage-specific state-action dynamics, designed to be complex and nonlinear, to assess algorithm performance under these varied conditions. While grounded in physiological principles, certain glucose-related physiological responses are simplified, or their fluctuations are intentionally magnified. This design choice aims to establish pronounced scenarios, such as clear hypoglycemic or hyperglycemic risks, thereby enabling a examination of the algorithms' performance in managing glycemic variability.

The following sections describe the state variables, action variables, meal definitions, stage-specific glucose dynamics, and reward function.

\subsubsection{State Variables}
The state at time \( t \) is defined by variables listed in Table \ref{tab:state_variables}. Cumulative nutrient values (e.g., \( C_M \)) reset to 0 at 6 AM daily.

\begin{table}[htbp]
\centering
\caption{State Variables for Simulational Study}
\label{tab:state_variables}
\begin{tabular}{@{}lll@{}}
\toprule
Symbol & Description & Unit / Range \\ \midrule
\( t \) & Time of day & 6--22 (hours) \\
\( G_t \) & Glucose level & mg/dL \\
\( \Delta G_t \) & Glucose rate of change (past hour) & mg/dL/hr \\
\( C_M, P_M, F_M \) & Morning carbohydrates, protein, fat & grams \\
\( C_D, P_D, F_D \) & Day carbohydrates, protein, fat & grams \\
\( C_E, P_E, F_E \) & Evening carbohydrates, protein, fat & grams \\
\( S_{sex} \) & Sex (0 or 1) & Categorical \\
\( W \) & Weight & kg \\ \bottomrule
\end{tabular}
\end{table}

\subsubsection{Action Variables for Simulational Study}
The agent selects actions at each time step from a 6-component vector, detailed in Table \ref{tab:action_variables}. Not all components affect every stage; unused components have no impact on that stage’s glucose dynamics.

\begin{table}[htbp]
\centering
\caption{Action Variables}
\label{tab:action_variables}
\begin{tabular}{@{}lll@{}}
\toprule
Symbol & Description & Values \\ \midrule
\( A_I \) & Insulin action & 0 (no dose), 1 (dose) \\
\( A_M \) & Meal type & 0 (no meal), 1 (low calorie), 2 (high calorie) \\
\( A_P \) & Physical activity & 0 (none), 1 (light), 2 (moderate/intense) \\
\( A_S \) & Sleep time (night only) & 22.0--24.0 (hours) \\
\( A_{SR} \) & Stress reduction & 0 (no), 1 (yes) \\
\( A_H \) & Hydration & 0 (no/low), 1 (yes/adequate) \\ \bottomrule
\end{tabular}
\end{table}

\subsubsection{Meal Definitions}
Meal types (\( A_M \)) have the following nutritional content:
\begin{itemize}[noitemsep]
    \item No Meal (Type 0): 0g carbohydrates, 0g protein, 0g fat.
    \item Low Calorie (Type 1): 30g carbohydrates, 10g protein, 10g fat.
    \item High Calorie (Type 2): 70g carbohydrates, 25g protein, 25g fat.
\end{itemize}

\subsubsection{Stage-Specific Glucose Dynamics}
Glucose levels (\( G_{t+1} \)) are computed using stage-specific equations based on the current state \( s_t \) and action \( a_t \).  Weight (\( W \)) is in kg, and current nutrients from \( A_M \) are \( C_{curr}, P_{curr}, F_{curr} \). Cumulative nutrients (e.g., \( C_M(t) \)) reflect amounts at the start of hour \( t \).

\textit{1.4.1 Morning Dynamics (6 AM -- 10 AM)}
Morning features higher insulin resistance and sensitivity to food intake. The state update is \( s'_{M} = \text{morning\_dynamic}(s_t, A_I, A_M, A_P, A_S, A_{SR}, A_H) \), where \( A_P \) and \( A_S \) have no effect.

\begin{align*}
    G_{t+1} = & \min(\max(G_{calc}, 50), 450) \\
    G_{calc} = & (10.0 + \text{early\_boost}) + 0.93 \cdot G_t \\
    & + (0.50 + 0.002 \cdot \max(0, G_t - 120)) \cdot \text{carb\_mod} \cdot C_{curr} \cdot (70/W) \\
     & + 0.10 \cdot P_{curr} + 0.03 \cdot F_{curr} \\
    & - 55.0 \cdot A_I \cdot \text{clip}(\text{eff\_mod}_{I,M}, 0.05, 1.5) \\
    & - 4.0 \cdot A_{SR} - 2.0 \cdot A_H \\
    & + 0.0018 \cdot C_M(t) \cdot W + 0.04 \cdot P_M(t) - 0.02 \cdot F_M(t) \\
    & + N(0, 5.5^2)
\end{align*}
\textbf{Details:}
\begin{itemize}[noitemsep]
    \item \( \text{early\_boost} = 5.0 \) if \( t < 8 \), else 0 (higher baseline glucose before 8 AM).
    \item \( \text{carb\_mod} = (1.1 \text{ if } A_H=1 \text{ else } 0.8) \cdot (1.1 \text{ if } A_{SR}=1 \text{ else } 0.7) \) adjusts carbohydrate impact based on hydration and stress reduction.
    \item Insulin effectiveness: \( \text{eff\_mod}_{I,M} = \text{time\_6am\_factor} \cdot \text{h\_factor} \cdot \text{sr\_factor} / (\text{fat\_res} \cdot \text{carb\_res}) \), where:
        \begin{itemize}[noitemsep]
            \item \( \text{time\_6am\_factor} = 0.25 \) if \( t=6 \) and \( A_M=0 \), else 1.0 (reduced insulin effect at 6 AM without a meal).
            \item \( \text{h\_factor} = 1.1 \) if \( A_H=1 \), else 0.8.
            \item \( \text{sr\_factor} = 1.1 \) if \( A_{SR}=1 \), else 0.7.
            \item \( \text{fat\_res} = 1.0 + 0.07 \cdot F_M(t) \) (morning fat increases insulin resistance).
            \item \( \text{carb\_res} = 1.0 + 0.005 \cdot C_{curr} \) (current carbohydrates slightly reduce insulin effect).
        \end{itemize}
    \item Cumulative morning nutrients (\( C_M, P_M, F_M \)) moderately affect glucose. New nutrients update \( C_M(t+1) \), etc. Day and evening cumulatives remain 0.
\end{itemize}

\textit{1.4.2 Day Dynamics (11 AM -- 4 PM)}
Daytime has higher insulin sensitivity and significant physical activity effects. The state update is 
\[ s'_{D} = \text{day\_dynamic}(s_t, A_I, A_M, A_P, A_S, A_{SR}, A_H), \]
where \( A_S \) is unused.

\begin{align*}
    G_{t+1} = & \min(\max(G_{calc}, 50), 450) \\
    G_{calc} = & 5.0 + 0.94 \cdot G_t \\
    & + (0.42 \cdot \text{carb\_proc\_eff}) \cdot C_{curr} \cdot (70/W) + 0.09 \cdot P_{curr} \\
    & - 80.0 \cdot A_I \cdot \text{clip}(\text{eff\_mod}_{I,D}, 0.1, 1.3) \\
    & + (\text{activity\_base} + \text{act\_carb\_syn}) \cdot A_P \cdot \text{act\_sr\_h\_mod} \\
    & - 8.0 \cdot A_I \cdot \text{bool}(A_P>0) - 6.0 \cdot A_{SR} - 4.0 \cdot A_H \\
    & + 0.0013 \cdot C_D(t) \cdot W + 0.020 \cdot C_M(t) + \text{other\_cumulatives} \\
    & + N(0, 4.5^2)
\end{align*}
\textbf{Details:}
\begin{itemize}[noitemsep]
    \item \( \text{carb\_proc\_eff} = (0.8 \text{ if } A_{SR}=0 \text{ else } 1.0) \cdot (0.85 \text{ if } A_H=0 \text{ else } 1.0) \) reflects reduced carbohydrate processing without stress reduction or hydration.
    \item Insulin effectiveness: \( \text{eff\_mod}_{I,D} = (1.0 - (0.002 C_M(t) + 0.004 F_M(t))) \cdot \text{sr\_h\_mod} \), where \( \text{sr\_h\_mod} = (1.1 \text{ if } A_{SR}=1 \text{ else } 0.8) \cdot (1.15 \text{ if } A_H=1 \text{ else } 0.85) \).
    \item Activity effects: \( \text{activity\_base} = -20.0 \), \( \text{act\_carb\_syn} = -0.15 \cdot (C_{curr}/50.0) \cdot A_P \), enhanced by \( \text{act\_sr\_h\_mod} = (1.1 \text{ if } A_{SR}=1 \text{ else } 0.9) \cdot (1.2 \text{ if } A_H=1 \text{ else } 0.8) \).
    \item The term \( -8.0 \cdot A_I \cdot \text{bool}(A_P>0) \) adds a synergistic glucose drop with insulin and activity.
    \item \( \text{other\_cumulatives} \) includes minor effects from \( P_D(t), F_D(t), P_M(t), F_M(t) \). Updates \( C_D, P_D, F_D \); evening cumulatives remain 0.
\end{itemize}

\textit{1.4.3 Evening Dynamics (5 PM -- 9 PM)}
Evening metabolism slows, with dietary fat having a stronger impact. The state update is
\[ s'_{E} = \text{evening\_dynamic}(s_t, A_I, A_M, A_P, A_S, A_{SR}, A_H), \]
where \( A_S \) is unused.

\begin{align*}
    G_{t+1} = & \min(\max(G_{calc}, 50), 450) \\
    G_{calc} = & (10.0 + 0.003(C_M(t)+C_D(t))) + 0.94 \cdot G_t \\
    & + (0.40 \cdot \text{carb\_proc\_eff}) \cdot C_{curr} \cdot (70/W) + 0.10 \cdot P_{curr} + 0.18 \cdot F_{curr} \\
    & - 55.0 \cdot A_I \cdot \text{clip}(\text{eff\_mod}_{I,E}, 0.1, 1.0) \\
    & + \text{activity\_eff}_{E} - 3.0 \cdot A_{SR} - 1.5 \cdot A_H \\
    & + 0.0016 \cdot C_E(t) \cdot W + \text{other\_cumulatives} \\
    & + N(0, 6.0^2)
\end{align*}
\textbf{Details:}
\begin{itemize}[noitemsep]
    \item Baseline glucose increases slightly with morning and day carbohydrates.
    \item Fat (\( F_{curr} \)) raises glucose more significantly.
    \item Insulin effectiveness: \( \text{eff\_mod}_{I,E} = \text{sr\_h\_mod} / (1.0 + \text{fat\_res}) \), where:
        \begin{itemize}[noitemsep]
            \item \( \text{sr\_h\_mod} = (0.8 \text{ if } A_{SR}=0 \text{ else } 1.0) \cdot (0.85 \text{ if } A_H=0 \text{ else } 1.0) \).
            \item \( \text{fat\_res} = 0.18 \cdot F_{curr} + 0.010 \cdot F_D(t) + 0.008 \cdot F_M(t) \).
        \end{itemize}
    \item Activity: \( \text{activity\_eff}_{E} = (-10.0 \cdot A_P \cdot \text{act\_sr\_h\_mod}) \), where \( \text{act\_sr\_h\_mod} = (0.7 \text{ if } A_{SR}=0 \text{ else } 1.0) \cdot (0.75 \text{ if } A_H=0 \text{ else } 1.0) \). If \( A_P > 0 \) and (\( C_{curr} > 45 \) or \( F_{curr} > 12 \)), this effect is multiplied by 0.4.
    \item \( \text{other\_cumulatives} \) includes effects from \( P_E(t), F_E(t), C_D(t), F_D(t), C_M(t), P_D(t) \). Updates \( C_E, P_E, F_E \).
\end{itemize}

\textit{1.4.4 Night Dynamics (10 PM -- 6 AM)}
Night models overnight glucose changes. The state update is \( s'_{N} = \text{night\_dynamic}(s_t, A_I, A_M, A_P, A_S, A_{SR}, A_H) \), where \( A_M, A_P \) are unused, and \( A_S \) (sleep time) is primary.

\begin{align*}
    G_{t+1} = & \min(\max(G_{calc}, 50), 450) \\
    G_{calc} = & (5.0 + \text{eod\_push}) + 0.97 \cdot G_t \\
    & - 30.0 \cdot A_I \cdot \text{clip}(\text{eff\_mod}_{I,N}, 0.05, 1.0) \\
    & + (\text{sleep\_eff\_val} \cdot \text{eff\_sleep\_hrs} \cdot \text{sleep\_qual}) \\
    & + \text{linger\_effects} - 0.5 \cdot A_{SR} - 0.2 \cdot A_H \\
    & + N(0, 6.0^2)
\end{align*}
\textbf{Details:}
\begin{itemize}[noitemsep]
    \item \( \text{eod\_push} = 0.15 F_E(t) - 0.08 P_E(t) + (1.5 \text{ if } A_{SR}=0 \text{ else } -1.0) + (1.0 \text{ if } A_H=0 \text{ else } -0.5) \).
    \item Insulin effectiveness: \( \text{eff\_mod}_{I,N} = (1.0 - \text{resist}) \), where \( \text{resist} = 0.003 \cdot \max(0, G_t-150) + 0.01 \cdot F_E(t) + 0.002 \cdot C_E(t) + (0.3 \text{ if } A_{SR}=0 \text{ else } -0.1) + (0.2 \text{ if } A_H=0 \text{ else } -0.05) \).
    \item Sleep: \( \text{sleep\_eff\_val} = -2.5 \), \( \text{eff\_sleep\_hrs} = \max(0, A_S - 22.0) \), \( \text{sleep\_qual} = (1 \text{ if } A_{SR}=1 \text{ else } 0.7) \cdot (1 \text{ if } A_H=1 \text{ else } 0.8) \cdot \max(0.5, 1 - 0.004 \cdot \max(0, G_t-140)) \).
    \item \( \text{linger\_effects} = 0.03 C_E(t) + 0.025 F_E(t) + 0.02 P_E(t) + 0.001 C_M(t) + 0.002 F_D(t) \).
    \item At 6 AM, \( t \) resets to 6, and all cumulative nutrients reset to 0.
\end{itemize}

\subsubsection{Reward Function}
The reward \( R_t \) penalizes time spent in undesirable glucose ranges during the transition from \( G_t \) to \( G_{t+1} \) over 1 hour (morning, day, evening) or 8 hours (night). Times \( T_{<70} \), \( T_{180-250} \), and \( T_{>250} \) (in mg/dL) are estimated by sampling 10 uniform points between \( G_t \) and \( G_{t+1} \):
\[ R_t = -3 \cdot T_{<70} - 2 \cdot T_{>250} - 1 \cdot T_{180-250} \]
Hypoglycemia (\( <70 \)) is penalized most, followed by severe hyperglycemia (\( >250 \)), then mild hyperglycemia (\( 180-250 \)).

\subsubsection{Hyperparameter Tuning and Evaluation}\label{appendix:hyperparameter}
We tuned and evaluated two reinforcement learning methods, the proposed \(\algname\) and FQI, using random forests as the function approximator. The hyperparameter tuned was the number of trees in the random forest, tested at 100, 200, and 300.

\textit{Tuning Process}
For both \(\algname\) and FQI, tuning was conducted with a fixed seed of 0. We generated 200 samples per stage (Morning, Day, Evening, Night), resulting in 800 samples total. Training ran for 100 iterations, with update set
\[\mathcal{U}=\{\mathrm{morning, day, evening, night}\}.\]
Hyperparameter selection was based on evaluation with a seed of 1. We generated 100 trajectories, each spanning 10 days, and computed the average cumulative reward over these trajectories. The Random Forest tree count yielding the highest average cumulative reward was selected as the optimal hyperparameter for each method. The results are as below.

\begin{table}[htbp]
\centering
\caption{Hyperparameter Tuning Results: Average Cumulative Reward with Standard Deviation over 10-Day Trajectories with seed 1 and 100 Trajectories.}
\label{tab:tuning_results}
\begin{tabular}{@{}cccc@{}}
\toprule
Number of Trees & \multicolumn{1}{c}{100} & \multicolumn{1}{c}{200} & \multicolumn{1}{c}{300} \\ \midrule
\(\algname\) & \textbf{-104.86} (36.57) & -133.19 (53.73) & -248.04 (24.23) \\
FQI & \textbf{-198.69} (69.82) & -373.97 (98.78) & -277.06 (17.87) \\ \bottomrule
\end{tabular}
\end{table}

\textit{Final Training and Validation}
Using the selected hyperparameters, we trained \(\algname\) and FQI with a seed of 1000, generating 100, 200 and 500 samples across all stages. For validation, we used a seed of 1001 to generate samples and computed the distribution of cumulative rewards over 50-day trajectories. 

\subsection{Type 1 Diabetes Patient Study Details}\label{appendix:real}
This section details the framework for analyzing real-world patient data from  subsection~\ref{subsection:real_experiment}. The data is organized into state-action-reward tuples to evaluate glycemic control across four daily stages: Morning, Day, Evening, and Night.

\subsubsection{Data Collection}\label{appendix:real_data_collection}
Participants were recruited through an anomymous medical center’s electronic health records and endocrinology clinic using purposive sampling to ensure diversity in demographic and health-related characteristics. Additional clinical and demographic data were collected to support the analysis.

\subsubsection{State Variables}\label{appendix:real_state_variables}
State variables are defined for each time point and are consistent across all stages, with stage-specific extensions. The common state features include:
\begin{itemize}
    \item Sex
    \item Age
    \item Current blood glucose level (mg/dL)
    \item Rate of change of blood glucose over the preceding hour (mg/dL/hr)
    \item Current hour of the day
    \item Cumulative carbohydrate, protein, and fat intake since the start of the morning period (in grams)
\end{itemize}

Stage-specific states incorporate additional cumulative macronutrient information, reset daily at 6 AM:
\begin{itemize}
    \item \textbf{Morning State:} Includes only the common state features.
    \item \textbf{Day State:} Extends the common features with cumulative carbohydrate, protein, and fat intake since the start of the Day period.
    \item \textbf{Evening State:} Includes the common features, cumulative macronutrients from the Day period, and cumulative carbohydrate, protein, and fat intake since the start of the Evening period.
    \item \textbf{Night State:} Identical to the Evening State.
\end{itemize}

\subsubsection{Action Definitions}\label{appendix:real_action_definitions}
Actions are defined for each stage based on logged or inferred patient behaviors, with combinations varying by stage.

\textbf{Morning Actions (6 AM -- 10 AM):}
Two binary actions yield $2^2 = 4$ combinations:
\begin{itemize}
    \item Insulin administered (Yes/No)
    \item Meal consumed (Yes/No)
\end{itemize}

\textbf{Day Actions (11 AM -- 4 PM):}
Three binary actions yield $2^3 = 8$ combinations:
\begin{itemize}
    \item Insulin administered (Yes/No)
    \item Meal consumed (Yes/No)
    \item Physical activity (METs $> 3$) performed (Yes/No)
\end{itemize}

\textbf{Evening Actions (5 PM -- 9 PM):}
Identical to Day actions, yielding $2^3 = 8$ combinations:
\begin{itemize}
    \item Insulin administered (Yes/No)
    \item Meal consumed (Yes/No)
    \item Physical activity (METs $> 3$) performed (Yes/No)
\end{itemize}

\textbf{Night Actions (Decision at 10 PM):}
One categorical and one binary action yield $3 \times 2 = 6$ combinations:
\begin{itemize}
    \item Chosen bedtime (10 PM, 11 PM, or 0 AM)
    \item Insulin administered around bedtime (Yes/No)
\end{itemize}

\subsubsection{Reward Function}\label{appendix:real_reward_function}
The reward function mirrors that of the simulation study, penalizing time spent in undesirable glycemic ranges. The penalties are:
\begin{itemize}
    \item Time with glucose $< 70$ mg/dL ($T_{<70}$): $-3$ per hour
    \item Time with glucose $> 250$ mg/dL ($T_{>250}$): $-2$ per hour
    \item Time with glucose between $180$ and $250$ mg/dL (inclusive, $T_{180-250}$): $-1$ per hour
\end{itemize}
The duration is 1 hour for daytime stages and approximately 8 hours for the Night stage, assuming linear glucose changes between observed points. The total reward $R_t$ is:
\[
R_t = -3 \cdot T_{<70} - 2 \cdot T_{>250} - 1 \cdot T_{180-250}
\]

\section{Limitations and Broader Impacts}
This study proposes an off-policy algorithm tailored for cyclic multi-stage decision-making. While it does not directly target online learning scenarios, the foundational structure and theoretical guarantees laid out here provide a natural pathway for future extensions to such settings. The analysis centers on least-squares regression over a broad class of function approximators, which is computationally efficient in most cases. For highly expressive models, additional refinements may further enhance scalability, presenting a promising direction for future work.

This research is primarily focused on methodological contributions and theoretical insights, with no foreseeable negative societal impacts. Rather, it contributes positively by enabling more effective sequential decision-making in domains characterized by cyclic dynamics, such as healthcare and urban systems, supporting the development of adaptive and data-driven solutions.

\section{Technical Lemmas}

\begin{lemma}[Regression Error Bound]\label{lemma:regression_bound}
Let \(\mathcal{D} = \{ (x_i, y_i) \}_{i=1}^n\) be a data set where each pair \((x_i, y_i)\) is independently drawn from a distribution \(\nu\).  

Assume \(f^*(x) = \mathbb{E}[y \mid x]\), and that \(y_i \in [0, Y]\) almost surely,  
with \(f(x) \in [0, Y]\) for all \(f \in \mathcal{F} \cup \{ f^* \}\).  

Let \(\epsilon_{\mathrm{approx}} \ge 0\) satisfy
\[
\inf_{f \in \mathcal{F}} \normtwo[,\nu]{f^* - f} \le \epsilon_{\mathrm{approx}}.
\]

Define the least-squares empirical risk minimizer
\[
\hat{f} = \argmin_{f \in \mathcal{F}} \frac{1}{n} \sum_{i=1}^n (f(x_i) - y_i)^2,
\]
and let \(\mathcal{N}(\epsilon) = \mathcal{N}(\epsilon, \mathcal{F}, \norminf{\cdot})\)  
denote the \(\epsilon\)-covering number of \(\mathcal{F}\) with respect to the supremum norm.

Then, for any \(\delta \in (0, 1)\), with probability at least \(1 - \delta\),
\[
\normtwo[,\nu]{\hat{f} - f^*}^2 \le 
\frac{45Y^2}{n} \log\left( \frac{2\mathcal{N}(1/n)}{\delta} \right)
+ 40 \epsilon_{\mathrm{approx}}^2.
\]
\end{lemma}

\bigskip

\begin{lemma}[Uniform Error Bounds for Iterative Vector Regression]\label{lemma:iterative_regression}
Let $\{ \mathbf{f}^j \}_{j=0}^M$, where $\mathbf{f}^j = (f_1^j, \dots, f_K^j) \in \mathcal{F}_{\mathrm{vec}} := \prod_{k=1}^K \mathcal{F}_k$, be a sequence generated iteratively starting from $\mathbf{f}^0$.

Let $\mathbf{T}: \mathcal{F}_{\mathrm{vec}} \to \mathcal{F}_{\mathrm{vec}}$ be an operator such that for each $j = 0, \dots, M-1$ and $k = 1, \dots, K$, the target function $(\mathbf{T}\mathbf{f}^j)_k \in [0, Y]$ almost surely with respect to distribution $\nu_k$.

For each $j = 0, \dots, M-1$ and $k = 1, \dots, K$, the function $f_k^{j+1} \in \mathcal{F}_k$ is obtained as the least-squares empirical risk minimizer over $\mathcal{F}_k$, targeting the conditional expectation defined by $(\mathbf{T}\mathbf{f}^j)_k$, using $n_k$ i.i.d. samples drawn from $\nu_k$. Assume that functions in $\mathcal{F}_k$ satisfy $f(x) \in [0, Y]$ for all $x$ in the support of $\nu_k$.

Assume that for each $j = 0, \dots, M-1$ and $k = 1, \dots, K$, the regression problem satisfies the conditions in Lemma~\ref{lemma:regression_bound}. Specifically, there exists an approximation error $\epsilon_{\mathrm{approx}, k} \ge 0$ such that:
\[
\inf_{f \in \mathcal{F}_k} \normtwo[,\nu_k]{(\mathbf{T}\mathbf{f}^j)_k - f} \le \epsilon_{\mathrm{approx}, k}.
\]

Let $\mathcal{N}_k(\epsilon) := \mathcal{N}(\epsilon, \mathcal{F}_k, \norminf{\cdot})$ denote the $\epsilon$-covering number of $\mathcal{F}_k$ with respect to the supremum norm. Then, for any $\delta \in (0, 1)$, with probability at least $1 - \delta$, the following bound holds uniformly over all $j = 0, \dots, M-1$ and $k = 1, \dots, K$:
\[
\normtwo[,\nu_k]{f_k^{j+1} - (\mathbf{T}\mathbf{f}^j)_k}^2 \le \frac{45Y^2}{n_k} \log\left( \frac{2MK \mathcal{N}_k(1/n_k)}{\delta} \right) + 40 \epsilon_{\mathrm{approx}, k}^2 := \epsilon_k^2.
\]
\end{lemma}

\bigskip

\begin{lemma}[Error Bound for Q-Function Estimation]\label{lemma:q_error_bound}
Let \(\beta_k \in \Delta(\mathcal{S}_k \times \mathcal{A}_k)\) be any admissible distribution induced by some policy for stage \(k\).Let \(\mathbf{Q}^{(m)} = (Q_1^{(m)}, \dots, Q_K^{(m)})\) denote the Q-functions estimated by \(\algname\) after \(m\) iterations for a given update set $\mathcal{U}$. Let $Q_k^*$ be the corresponding $\mathcal{U}$-constrained optimal Q-function for stage $k$.

Under Assumption~\ref{assumption:coverage}, the following uniform error bound holds with  probability at least \(1 - \delta\):
\[
\normtwo[,\beta_k]{Q_k^* - Q_k^{(m)}} \le \frac{\sqrt{C}H}{1 - \gamma_{\mathrm{cycle}}} \varepsilon + \gamma_{\mathrm{cycle}}^{\lfloor m/H \rfloor} Y,
\]
where \(H = \sum_{k=1}^K H_k\), \(\varepsilon = \max_k \epsilon_k\) with $\epsilon_k$ being the per-stage regression error bound derived from Lemma~\ref{lemma:iterative_regression}, and \(\gamma_{\mathrm{cycle}} = \prod_{k=1}^K \gamma_k\).
\end{lemma}

\bigskip

\begin{lemma}[Performance Difference for Cyclic MDPs]\label{lemma:performance_difference}
Let \(\pi\) and \(\pi'\) be any two composite policies, and fix any stage \(k\) with an initial state distribution \(\mu_k \in \Delta(\mathcal{S}_k)\).  
Let \(\tau = (s_0, a_0, s_1, a_1, \dots)\) be a trajectory generated under policy \(\pi\), where at each time step \(t \ge 0\), the following hold:
\[
s_0 \sim \mu_k, \quad
a_t \sim \pi_{j_t}(\cdot \mid s_t), \quad
s_t' \sim P_{j_t}(\cdot \mid s_t, a_t), \quad
s_{t+1} = 
\begin{cases}
\phi_{j_t}(s_t') & \text{if } (s_t, a_t) \in \mathcal{T}_{j_t}, \\
s_t' &\text{if } (s_t, a_t) \notin \mathcal{T}_{j_t}.
\end{cases}
\]
Here, \(j_t\) denotes the stage index such that the state \(s_t\) belongs to stage \(j_t\), that is, \(s_t \in \mathcal{S}_{j_t}\).

Then the performance difference is given by:
\[
v_k^\pi(\mu_k) - v_k^{\pi'}(\mu_k) 
= \sum_{t=0}^\infty \gamma_{\mathrm{cycle}}^{\lfloor t/H \rfloor} 
\mathbb{E}_{s_0 \sim \mu_k} 
\mathbb{E}_{\tau \sim \Pr^\pi(\tau \mid s_0)} 
\left[ A_{j_t}^{\pi'}(s_t, a_t) \right],
\]
where \(A_{j_t}^{\pi'}(s_t, a_t) := Q_{j_t}^{\pi'}(s_t, a_t) - V_{j_t}^{\pi'}(s_t)\), and \(\Pr^\pi(\tau \mid s_0)\) denotes the distribution over trajectories induced by policy \(\pi\) starting from initial state \(s_0\).
\end{lemma}

\section{Proofs of Propositions, Corollaries and Lemmas}

\subsection{Proof of Proposition \ref{prop:contraction}}

\begin{proof}
For simplicity, we denote the operator \(\mathbf{T}_{\mathcal{U}}\) by \(\mathbf{T}\) throughout.

\textbf{Step 1: Non-Expansive Property}

Define the elementwise error function \(\mathbf{h} \coloneqq |\mathbf{f} - \mathbf{g}|\), where \(\mathbf{f}, \mathbf{g} \in \prod_{k=1}^K L_\infty(\mathcal{S}_k \times \mathcal{A}_k)\), so \(\mathbf{h} \geq 0\) elementwise. Introduce the auxiliary operator \(\mathbf{P}\) via:
\[
\mathbf{P}\mathbf{f} \coloneqq \mathbf{T}|\mathbf{f}| - \mathbf{T}\mathbf{0}.
\]
This operator is monotone by its definition: if \(|\mathbf{f}| \leq |\mathbf{g}|\) elementwise, then \(\mathbf{P}\mathbf{f} \leq \mathbf{P}\mathbf{g}\).

And also from the definitions, we have:
\[
|\mathbf{T}\mathbf{f} - \mathbf{T}\mathbf{g}| \leq \mathbf{P}\mathbf{h} \quad \text{elementwise}.
\]
By definition, \(|\mathbf{P}\mathbf{h}| \leq \mathbf{h}\) elementwise, so:
\[
\|\mathbf{P}\mathbf{h}\|_\infty \leq \|\mathbf{h}\|_\infty.
\]
Thus:
\[
\|\mathbf{T}\mathbf{f} - \mathbf{T}\mathbf{g}\|_\infty \leq \|\mathbf{P}\mathbf{h}\|_\infty \leq \|\mathbf{h}\|_\infty = \|\mathbf{f} - \mathbf{g}\|_\infty.
\]
This establishes that \(\mathbf{T}\) is non-expansive:
\[
\|\mathbf{T}\mathbf{f} - \mathbf{T}\mathbf{g}\|_\infty \leq \|\mathbf{f} - \mathbf{g}\|_\infty.
\]

\medskip

\textbf{Step 2: $H$-Step Contraction Property}

We aim to show:
\[
\|\mathbf{T}^H \mathbf{f} - \mathbf{T}^H \mathbf{g}\|_\infty \leq \gamma_{\mathrm{cycle}} \|\mathbf{f} - \mathbf{g}\|_\infty,
\]
where \(H = \sum_{k=1}^K H_k\) and \(\gamma_{\mathrm{cycle}} = \prod_{k=1}^K \gamma_k\).

By induction, we prove:
\[
|\mathbf{T}^n \mathbf{f} - \mathbf{T}^n \mathbf{g}| \leq \mathbf{P}^n \mathbf{h} \quad \text{(elementwise for all } n \geq 1\text{)}.
\]
\begin{itemize}
\item \textbf{Base case (\(n=1\))}: From Step 1, \(|\mathbf{T}\mathbf{f} - \mathbf{T}\mathbf{g}| \leq \mathbf{P}\mathbf{h}\).
\item \textbf{Inductive step}: Assume the inequality holds for some \(n \geq 1\): \(|\mathbf{T}^n \mathbf{f} - \mathbf{T}^n \mathbf{g}| \leq \mathbf{P}^n \mathbf{h}\). Then:
\[
\begin{aligned}
|\mathbf{T}^{n+1}\mathbf{f} - \mathbf{T}^{n+1}\mathbf{g}| &= |\mathbf{T}(\mathbf{T}^n\mathbf{f}) - \mathbf{T}(\mathbf{T}^n\mathbf{g})| \\
&\leq \mathbf{P}\bigl(|\mathbf{T}^n\mathbf{f} - \mathbf{T}^n\mathbf{g}|\bigr) \quad (\text{using } |\mathbf{T}\mathbf{f}' - \mathbf{T}\mathbf{g}'| \leq \mathbf{P}(|\mathbf{f}' - \mathbf{g}'|)) \\
&\leq \mathbf{P}\bigl(\mathbf{P}^n \mathbf{h}\bigr) \quad (\text{by induction hypothesis and monotonicity}) \\
&= \mathbf{P}^{n+1} \mathbf{h}.
\end{aligned}
\]
\end{itemize}
Thus, by induction, the inequality holds for all \(n \geq 1\), implying:
\[
\|\mathbf{T}^n\mathbf{f} - \mathbf{T}^n\mathbf{g}\|_\infty \leq \|\mathbf{P}^n \mathbf{h}\|_\infty.
\]

We now bound \(\|\mathbf{P}^H \mathbf{h}\|_\infty\). Fix a stage \(k \in \{1, \dots, K\}\) and an initial state-action pair \((s_1, a_1) \in \mathcal{S}_k \times \mathcal{A}_k \). Construct a stochastic state-action sequence:
\begin{itemize}
\item Start with \((s_1, a_1)\).
\item For each step \(i \geq 1\), sample \(s_{i}' \sim P_k(\cdot | s_i, a_i)\).
\item If \((s_i, a_i) \in \mathcal{T}_k\), terminate the sequence.
\item If \((s_i, a_i) \notin \mathcal{T}_k\), set \(s_{i+1} = s_i'\) and select \(a_{i+1}\):
\begin{itemize}
\item If \(k \in \mathcal{U}\), choose
\(
a_{i+1} = \arg\max_{a \in \mathcal{A}_k} (\mathbf{P}^{H_k - i} h_k)(s_{i+1}, a).
\)
\item If \(k \notin \mathcal{U}\), sample \(a_{i+1} \sim \pi_k^{\circ}(\cdot \mid s_{i+1})\).
\end{itemize}
\item Define the stopping time \(\tau = \min\{n \geq 1 : (s_n, a_n) \in \mathcal{T}_k\}\).
\end{itemize}
The finite-horizon assumption ensures \(\tau \leq H_k\) almost surely.

For \(k \in \mathcal{U}\), expand \((\mathbf{P}^{H_k} h)_k(s_1, a_1)\):
\[
\begin{aligned}
(\mathbf{P}^{H_k} h)_k(s_1, a_1) &\leq \mathbb{E}_{s_1' \sim P_k(\cdot \mid s_1, a_1)} \Biggl[ \max_{a_2 \in \mathcal{A}_k}(\mathbf{P}^{H_k-1} h_k)(s_2, a_2) \, \mathbb{I}((s_1, a_1) \notin \mathcal{T}_k) \\
&\quad + \gamma_k \max_{a' \in \mathcal{A}_{[k+1]}} (\mathbf{P}^{H_k-1} h_{[k+1]})(\phi_k(s_1'), a') \, \mathbb{I}((s_1, a_1) \in \mathcal{T}_k) \Biggr].
\end{aligned}
\]
Consider the two cases:
\begin{itemize}
\item \textbf{Terminal case (\((s_1, a_1) \in \mathcal{T}_k\), \(\tau = 1\))}:

The first term is zero, and the second term is:
\begin{multline*}
\mathbb{E}_{s_1'} \bigl[ \gamma_k \max_{a_2 \in \action_{[k+1]}} (\mathbf{P}^{H_k-1} h_{[k+1]})(\phi_k(s_1'), a_2) \bigr] \mathbb{I}(\tau=1) \\
\leq \gamma_k \|\mathbf{P}^{H_k-1} h_{[k+1]}\|_\infty \mathbb{P}(\tau = 1) \leq \gamma_k \|h_{[k+1]}\|_\infty \mathbb{P}(\tau = 1),    
\end{multline*}
since \(\|\mathbf{P}^n \mathbf{h}\|_\infty \leq \|\mathbf{h}\|_\infty\).
\item \textbf{Non-terminal case (\((s_1, a_1) \notin \mathcal{T}_k\), \(\tau > 1\))}:

The second term is zero, and the first term is:
\[
\mathbb{E}_{s_2 \sim P_k(\cdot \mid s_1, a_1)} \left[ \max_{a_2 \in \mathcal{A}_k} (\mathbf{P}^{H_k-1} h_k)(s_2, a_2)  \mathbb{I}(\tau > 1)\right] .
\]
Recursively apply the bound, conditioning on \(\tau = 2\) or \(\tau > 2\):
\[
\begin{aligned}
&\mathbb{E}_{s_2 \sim P_k( \cdot \mid s_1, a_1)}\bigl[\max_{a_2 \in \action_{[k+1]}} (\mathbf{P}^{H_k-1} h_k)(s_2, a_2) \mathbb{I}(\tau > 1) \bigr]\\
&\quad \leq \gamma_k \|h_{[k+1]}\|_\infty \mathbb{P}(\tau = 2) + \mathbb{E}\bigl[\max_{a_3} (\mathbf{P}^{H_k-2} h_k)(s_3, a_3)  \mathbb{I}(\tau > 2) \bigr].
\end{aligned}
\]
\end{itemize}

For \(k \notin \mathcal{U}\), replace \(\max_{a}\) with \(\mathbb{E}_{a \sim \pi_k^{\circ}}\), yielding the same bound structure.

Summing over all termination times \(n=1, \dots, H_k\), since \(\sum_{n=1}^{H_k} \mathbb{P}(\tau = n) = 1\) we have:
\[
(\mathbf{P}^{H_k} h)_k(s_1, a_1) \leq \sum_{n=1}^{H_k} \gamma_k \|h_{[k+1]}\|_\infty \mathbb{P}(\tau = n) = \gamma_k \|h_{[k+1]}\|_\infty.
\]
Taking the supremum:
\[
\|(\mathbf{P}^{H_k} h)_k\|_\infty \leq \gamma_k \|h_{[k+1]}\|_\infty.
\]

Apply this bound across stages \(k=1, \dots, K\) for a full cycle (\(H = \sum H_k\)):
\[
\begin{aligned}
\| (\mathbf{P}^H \mathbf{h})_1 \|_\infty &= \| (\mathbf{P}^{H_K} \circ \dots \circ \mathbf{P}^{H_1} \mathbf{h})_1 \|_\infty \\
&\leq \gamma_1 \|(\mathbf{P}^{H_K} \circ \dots \circ \mathbf{P}^{H_2} \mathbf{h})_2\|_\infty \\
&\leq \gamma_1 \gamma_2 \|(\mathbf{P}^{H_K} \circ \dots \circ \mathbf{P}^{H_3} \mathbf{h})_3\|_\infty \\
&\quad \vdots \\
&\leq \left( \prod_{j=1}^K \gamma_j \right) \|\mathbf{h}_{1}\|_\infty = \gamma_{\mathrm{cycle}} \|\mathbf{h}_1\|_\infty.
\end{aligned}
\]
Since the starting stage is arbitrary, we have:
\[
\|\mathbf{P}^H \mathbf{h}\|_\infty \leq \gamma_{\mathrm{cycle}} \|\mathbf{h}\|_\infty.
\]
Thus:
\[
\|\mathbf{T}^H \mathbf{f} - \mathbf{T}^H \mathbf{g}\|_\infty \leq \|\mathbf{P}^H \mathbf{h}\|_\infty \leq \gamma_{\mathrm{cycle}} \|\mathbf{f} - \mathbf{g}\|_\infty.
\]

\medskip

\textbf{Step 3: Existence and Uniqueness of the Fixed Point}

The space
\[
\mathcal{F} = \prod_{k=1}^K L_\infty(\mathcal{S}_k \times \mathcal{A}_k)
\]
is a complete Banach space under the sup-norm.

Let \(\mathbf{f}_n = \mathbf{T}^n \mathbf{f}_0\) for some arbitrary \(\mathbf{f}_0 \in \mathcal{F}\), and define the subsequence \(\mathbf{f}_{kH} = \mathbf{T}^{kH} \mathbf{f}_0\). Since \(\mathbf{T}^H\) is a contraction with constant \(\gamma_{\mathrm{cycle}} < 1\), we have
\[
\|\mathbf{f}_{(k+1)H} - \mathbf{f}_{kH}\|_\infty \le \gamma_{\mathrm{cycle}}^k \|\mathbf{T}^H \mathbf{f}_0 - \mathbf{f}_0\|_\infty,
\]
so \(\{\mathbf{f}_{kH}\}\) is Cauchy and converges to some \(\mathbf{f}^* \in \mathcal{F}\).

To show the full sequence \(\{\mathbf{f}_n\}\) converges to the same limit, we bound the difference between \(\mathbf{f}_n\) and \(\mathbf{f}_{kH}\), where \(n = kH + r\) for some \(0 \le r < H\). Using the non-expansiveness of \(\mathbf{T}\), we have:

\begin{align*}
\|\mathbf{f}_n - \mathbf{f}_{kH}\|_\infty
&= \|\mathbf{T}^{kH + r} \mathbf{f}_0 - \mathbf{T}^{kH} \mathbf{f}_0\|_\infty
\\&= \|\mathbf{T}^{kH}(\mathbf{T}^r \mathbf{f}_0) - \mathbf{T}^{kH}\mathbf{f}_0)\|_\infty
\\&\le \gamma_{\mathrm{cycle}}^k \|\mathbf{T}^r \mathbf{f}_0 - \mathbf{f}_0\|_\infty.
\end{align*}

Taking the maximum over all \(0 \le r < H\), this gives:
\[
\|\mathbf{f}_n - \mathbf{f}_{kH}\|_\infty \le \gamma_{\mathrm{cycle}}^k \cdot \max_{0 \le r < H} \|\mathbf{T}^r \mathbf{f}_0 - \mathbf{f}_0\|_\infty.
\]

Similarly, for any \(m = \ell H + s\),
\[
\|\mathbf{f}_m - \mathbf{f}_{\ell H}\|_\infty \le \gamma_{\mathrm{cycle}}^\ell \cdot \max_{0 \le r < H} \|\mathbf{T}^r \mathbf{f}_0 - \mathbf{f}_0\|_\infty.
\]

Thus, for \(n = kH + r\) and \(m = \ell H + s\), we have:
\[
\|\mathbf{f}_n - \mathbf{f}_m\|_\infty \le \|\mathbf{f}_{kH} - \mathbf{f}_{\ell H}\|_\infty + \|\mathbf{f}_n - \mathbf{f}_{kH}\|_\infty + \|\mathbf{f}_m - \mathbf{f}_{\ell H}\|_\infty.
\]

As \(k,\ell \to \infty\), each term on the right vanishes, so \(\{\mathbf{f}_n\}\) is Cauchy and converges to \(\mathbf{f}^*\).

To see that \(\mathbf{f}^*\) is a fixed point:
\[
\mathbf{T} \mathbf{f}^* = \lim_{n \to \infty} \mathbf{T} \mathbf{f}_n = \lim_{n \to \infty} \mathbf{f}_{n+1} = \mathbf{f}^*,
\]
by continuity of \(\mathbf{T}\) implied by non-expansiveness.

For uniqueness, suppose another fixed point \(\mathbf{f}^\dagger \neq \mathbf{f}^*\) exists. Then
\[
\|\mathbf{f}^* - \mathbf{f}^\dagger\|_\infty = \|\mathbf{T}^H \mathbf{f}^* - \mathbf{T}^H \mathbf{f}^\dagger\|_\infty \le \gamma_{\mathrm{cycle}} \|\mathbf{f}^* - \mathbf{f}^\dagger\|_\infty,
\]
which implies \(\mathbf{f}^* = \mathbf{f}^\dagger\), a contradiction.

Hence, \(\mathbf{f}^*\) is the unique fixed point. This argument holds for any update set \(\mathcal{U} \subseteq \{1, \dots, K\}\).

\end{proof}

\bigskip

\subsection{Proof of Corollary \ref{cor:upper_bound_comparison}}

\begin{proof}
The first claim for $\algname$ follows directly from Theorem~\ref{thm:finite_sample_rate}.
By specifying the smoothness parameters $p_k = p, q_k = q, s_k = s$ and the approximation exponent $\alpha_k = \rho s / (d + d_k)$ for all $k \in \{1, \dots, K\}$, the conditions of Theorem~\ref{thm:finite_sample_rate} are satisfied. The global error is dominated by the term with the largest exponent, which corresponds to the maximum dimension $d_{\max} = d + \max_k d_k$.

For the second claim concerning the flattened baseline, we treat the problem as a single-stage MDP equivalent defined on the joint state space.
In this setting, the effective input dimension is the cumulative sum:
\[
d_{\text{total}} = d + \sum_{k=1}^K d_k.
\]
Applying Theorem~\ref{thm:finite_sample_rate} to this joint domain with the pooled sample size $N = Kn$ yields the rate $\widetilde{\mathcal{O}}((Kn)^{-\frac{\rho s}{2\rho s + d_{\text{total}}}})$. Absorbing the constant $K$ into the asymptotic notation yields the stated upper bound.
\end{proof}

\subsection{Proof of Corollary \ref{cor:flattened_lower_bound}}

\begin{proof}
We analyze the worst-case estimation error of the specific flattened estimator $\hat{Q}$ produced by the baseline algorithm. Let $\mathcal{B} = B^s_{p,q}(\mathbb{R}^{d_{\text{total}}})$ denote the target Besov space and $\mathcal{F}_D$ denote the neural network model class with capacity $D$. The quantity of interest is the worst-case risk:
\[
\mathfrak{R}_{\text{Q}}(\hat{Q}) := \sup_{Q^* \in \mathcal{B}} \mathbb{E}_{\mathcal{D}_N} \left[ \|\hat{Q} - Q^*\|_{L^2(\bar{\nu})}^2 \right].
\]

\textit{Step 1: Construction of the Hard Instance.}
To establish a rigorous lower bound, we construct an MDP instance strictly respecting the bounded reward assumption while maintaining statistical hardness. We set $\gamma_k = 0$ for all stages, implying $Q^*(s,a) = r(s,a)$.
To ensure the reward lies in $[0, R_{\max}]$, we assume the target function $Q^*$ is centered at $R_{\max}/2$. The observed reward is generated as:
\[
R_1 = Q^*(s,a) + \xi, \quad \xi \sim \mathcal{D}(0, \sigma^2), \quad \xi \in \left[-\frac{R_{\max}}{2}, \frac{R_{\max}}{2}\right].
\]
Here, $\xi$ follows a bounded distribution (e.g., a Truncated Gaussian) with mean zero and a fixed variance parameter $\sigma^2 > 0$. Since the noise is bounded, it naturally satisfies the sub-Gaussian property with parameter $\sigma^2$. This setup is equivalent to regressing $Q^*$ from a data set $\mathcal{D}_N$ of size $N=Kn$.

Since $\gamma_k=0$, policy optimization plays no role, and the problem reduces
to pure nonparametric regression; thus any lower bound here applies a fortiori
to the flattened baseline in the full RL setting.

\textit{Step 2: Fundamental Limits.}
We define two intrinsic properties of the model class $\mathcal{F}_D$ that serve as unavoidable lower bounds.
First, the \textit{Approximation Limit} $\mathbb{B}^2(\mathcal{F}_D)$ represents the worst-case geometric distance from the target space $\mathcal{B}$ to the model manifold $\mathcal{F}_D$:
\[
\mathbb{B}^2(\mathcal{F}_D) := \sup_{Q^* \in \mathcal{B}} \inf_{f \in \mathcal{F}_D} \|f - Q^*\|_{L^2(\bar{\nu})}^2.
\]
Second, the \textit{Statistical Limit} $\mathbb{V}_{\min}(\mathcal{F}_D)$ represents the minimum risk over the realizable subclass restricted to the subset $\mathcal{G} := \mathcal{B} \cap \mathcal{F}_D$ of functions perfectly representable by the model:
\[
\mathbb{V}_{\min}(\mathcal{F}_D) := \inf_{\Psi} \sup_{Q^* \in \mathcal{G}} \mathbb{E}_{\mathcal{D}_N} \left[ \|\Psi - Q^*\|_{L^2(\bar{\nu})}^2 \right],
\]
where the infimum is taken over all possible measurable estimators $\Psi$.

\textit{Step 3: Structural Decomposition of Risk.}
We strictly prove that the risk of our specific estimator $\hat{Q}$ is lower-bounded by the maximum of these limits.
For the geometric barrier, since $\hat{Q} \in \mathcal{F}_D$, the error for any target $Q^*$ is bounded below by the distance to the model class. Taking the supremum over $Q^*$ yields $\mathfrak{R}_{\text{Q}}(\hat{Q}) \ge \mathbb{B}^2(\mathcal{F}_D)$.

For the statistical barrier, we use the subset inclusion $\mathcal{G} \subset \mathcal{B}$. The worst-case risk over the full space dominates the risk over the subset. Moreover, since $\hat{Q}$ is an empirical risk minimizer constrained to $\mathcal{F}_D$, it is subject to the class's inherent statistical limits and cannot outperform the theoretically optimal minimax estimator $\Psi$ on this subset:
\[
\mathfrak{R}_{\text{Q}}(\hat{Q}) \ge \sup_{Q^* \in \mathcal{G}} \mathbb{E} \|\hat{Q} - Q^*\|^2 \ge \inf_{\Psi} \sup_{Q^* \in \mathcal{G}} \mathbb{E} \|\Psi - Q^*\|^2 = \mathbb{V}_{\min}(\mathcal{F}_D).
\]
Combining these, we obtain the decomposition:
\[
\mathfrak{R}_{\text{Q}}(\hat{Q}) \ge \max\left( \mathbb{B}^2(\mathcal{F}_D), \mathbb{V}_{\min}(\mathcal{F}_D) \right).
\]

\textit{Step 4: Quantifying the Approximation Limit.}
Condition \ref{cond:L2} (Approximation Limit) implies that the finite capacity $D$ imposes a hard limit on accuracy. The metric entropy of the Besov space ensures that the worst-case approximation error scales polynomially with $D$:
\[
\mathbb{B}^2(\mathcal{F}_D) \ge c_{\text{approx}} \cdot D^{-\frac{2\rho s}{d_{\text{total}}}}.
\]

\textit{Step 5: Quantifying the Statistical Limit.}
We analyze $\mathbb{V}_{\min}(\mathcal{F}_D)$ using Fano's method. This involves relating the packing number of the model class to the information capacity of the channel.

\textbf{1. Local Packing Construction:}
Using Condition \ref{cond:L1} (Metric Entropy), we focus on a \textit{locally rich subset} of $\mathcal{G}$ to construct a maximal $\delta$-packing set $\mathcal{M} = \{Q_1, \dots, Q_M\} \subset \mathcal{G}$. This set satisfies:
\[
\log M \ge c_{\text{ent}} \cdot D \quad \text{and} \quad \|Q_i - Q_j\|_{L^2(\bar{\nu})} \ge \delta \quad \forall i \neq j.
\]

\textbf{2. Fano's Inequality Bound:}
Let $J$ be a uniform random variable on $\{1, \dots, M\}$. Identifying the true function index $J$ from data $\mathcal{D}_N$ is a lower bound for the estimation error. Fano's inequality states:
\[
\mathbb{V}_{\min}(\mathcal{F}_D) \ge \frac{\delta^2}{2} \left( 1 - \frac{I(J; \mathcal{D}_N) + \log 2}{\log M} \right).
\]

\textbf{3. Mutual Information and Noise Structure.}
We upper-bound the mutual information $I(J; \mathcal{D}_N)$ induced by the regression model.
Conditioned on $J=j$, the data $\mathcal{D}_N = \{(S_i,A_i,R_i)\}_{i=1}^N$ are generated as:
\[
R_i = Q_j(S_i,A_i) + \xi_i,
\qquad
\xi_i \ \text{i.i.d., mean zero, sub-Gaussian with parameter } \sigma^2.
\]
Since the samples are independent, the joint likelihood factorizes as:
\[
P_{Q_j}(\mathcal{D}_N)
=
\prod_{i=1}^N P_{Q_j}(R_i \mid S_i,A_i).
\]
For sub-Gaussian noise, the KL divergence between two such regression models admits the standard quadratic upper bound:
\[
D_{\mathrm{KL}}(P_{Q_i} \,\|\, P_{Q_j})
\;\le\;
\frac{N}{2\sigma^2}
\mathbb{E}_{(S,A)\sim\bar{\nu}}
\!\left[
\bigl(Q_i(S,A)-Q_j(S,A)\bigr)^2
\right]
=
\frac{N}{2\sigma^2}
\|Q_i-Q_j\|_{L^2(\bar{\nu})}^2 .
\]
Using the standard mutual information bound under a uniform prior over the packing set $\mathcal{M}=\{Q_1,\dots,Q_M\}$:
\[
I(J;\mathcal{D}_N)
\le
\frac{1}{M^2}
\sum_{i,j}
D_{\mathrm{KL}}(P_{Q_i}\,\|\,P_{Q_j}),
\]
we obtain:
\[
I(J;\mathcal{D}_N)
\le
\frac{N}{2\sigma^2}
\sup_{i\neq j}
\|Q_i-Q_j\|_{L^2(\bar{\nu})}^2 .
\]
Restricting the packing to a local neighborhood of radius $\delta$, such that $\|Q_i-Q_j\|_{L^2(\bar{\nu})}^2 \le c\,\delta^2$ for all $i\neq j$, yields:
\[
I(J;\mathcal{D}_N)
\le
\frac{c N \delta^2}{2\sigma^2}.
\]

\textbf{4. Critical Resolution.}
To ensure a nontrivial Fano lower bound, we choose $\delta$ so that the mutual information is bounded by a constant fraction of the metric entropy:
\[
\frac{c N \delta^2}{2\sigma^2}
\le
\frac{1}{2}\log M.
\]
Invoking Condition \ref{cond:L1}, which guarantees $\log M \asymp D$, we obtain the critical resolution:
\[
\delta^2 \asymp \frac{\sigma^2 D}{N}.
\]
Substituting this choice into Fano’s inequality yields the statistical lower bound:
\[
\mathbb{V}_{\min}(\mathcal{F}_D)
\ge
c_{\text{stat}}\frac{D}{N}
=
c_{\text{stat}}\frac{D}{K n}.
\]

\textit{Step 6: Optimal Rate Derivation.}
Substituting the bounds from Steps 4 and 5 into the decomposition:
\[
\mathfrak{R}_{\text{Q}}(\hat{Q}) \gtrsim \max \left( D^{-\frac{2\rho s}{d_{\text{total}}}}, \quad \frac{D}{Kn} \right).
\]
The optimal lower bound is achieved when the bias and variance are balanced. Solving for the optimal capacity $D^* \asymp (Kn)^{\frac{d_{\text{total}}}{2\rho s + d_{\text{total}}}}$. Treating the number of stages $K$ as a fixed constant and substituting $D^*$ back yields the final rate:
\[
\mathfrak{R}_{\text{Q}}(\hat{Q}) = \Omega \left( n^{-\frac{\rho s}{2\rho s + d_{\text{total}}}} \right).
\]
\end{proof}

\subsection{Proof of Lemma \ref{lemma:regression_bound}}
\begin{proof}
Let \( f^*(x) = \mathbb{E}[y \mid x] \) denote the true regression function. For any function \( f \in \mathcal{F} \), define the random variable \( e_i^f \) as the difference between the squared error of \( f \) and that of \( f^* \):
\[
e_i^f = (f(x_i) - y_i)^2 - (f^*(x_i) - y_i)^2,
\]
where the pair \( (x_i, y_i) \) is drawn independently from the distribution \( \nu \).

\medskip

\textbf{1. First Moment Calculation.}

We compute the expectation of \( e_i^f \). Using the law of total expectation, \( \mathbb{E}_{\nu}[\cdot] = \mathbb{E}_{x_i \sim \nu}[\mathbb{E}[\cdot \mid x_i]] \), we obtain:
\begin{align*}
\mathbb{E}_{\nu}[e_i^f] &= \mathbb{E}_{\nu} \left[ \mathbb{E}[ (f(x_i) - y_i)^2 - (f^*(x_i) - y_i)^2 \mid x_i] \right] \\
&= \mathbb{E}_{\nu} \left[ \mathbb{E}[ f(x_i)^2 - 2f(x_i)y_i + y_i^2 - (f^*(x_i)^2 - 2f^*(x_i)y_i + y_i^2) \mid x_i] \right] \\
&= \mathbb{E}_{\nu} \left[ f(x_i)^2 - 2f(x_i)\mathbb{E}[y_i \mid x_i] + \mathbb{E}[y_i^2 \mid x_i] - f^*(x_i)^2 + 2f^*(x_i)\mathbb{E}[y_i \mid x_i] - \mathbb{E}[y_i^2 \mid x_i] \right] \\
&= \mathbb{E}_{\nu} \left[ f(x_i)^2 - 2f(x_i)f^*(x_i) - f^*(x_i)^2 + 2f^*(x_i)f^*(x_i) \right] \quad \text{(since } \mathbb{E}[y_i \mid x_i] = f^*(x_i)\text{)} \\
&= \mathbb{E}_{\nu} \left[ (f(x_i) - f^*(x_i))^2 \right] \\
&= \normtwo[,\nu]{f - f^*}^2.
\end{align*}
Thus, the expected value of \( e_i^f \) is the squared \( L_2(\nu) \) distance between \( f \) and \( f^* \).

\medskip

\textbf{2. Second Moment and Range Bounding.}

Next, we bound the second moment of \( e_i^f \):
\begin{align*}
\mathbb{E}_{\nu}[(e_i^f)^2] &= \mathbb{E}_{\nu} \left[ \mathbb{E}[ ((f(x_i) - y_i)^2 - (f^*(x_i) - y_i)^2)^2 \mid x_i] \right] \\
&= \mathbb{E}_{\nu} \left[ \mathbb{E}[ ((f(x_i) - f^*(x_i))(f(x_i) + f^*(x_i) - 2y_i))^2 \mid x_i] \right] \\
&= \mathbb{E}_{\nu} \left[ (f(x_i) - f^*(x_i))^2 \mathbb{E}[ (f(x_i) + f^*(x_i) - 2y_i)^2 \mid x_i] \right].
\end{align*}
Given the boundedness assumptions in the lemma by
\[|f(x_i)| \leq Y, \;\; |f^*(x_i)| \leq Y  \;\; \text{and} \;\; |y_i| \leq Y, \]
we have
\[
\mathbb{E}[ (f(x_i) + f^*(x_i) - 2y_i)^2 \mid x_i] \leq 4Y^2.
\]
Thus:
\[
\mathbb{E}_{\nu}[(e_i^f)^2] \leq 4Y^2 \mathbb{E}_{\nu}[ (f(x_i) - f^*(x_i))^2 ] = 4Y^2 \normtwo[,\nu]{f - f^*}^2.
\]
Similarly, the range of \( e_i^f \) is bounded by:
\[
|e_i^f| = |(f(x_i) - f^*(x_i))(f(x_i) + f^*(x_i) - 2y_i)| \leq 2Y^2.
\]

\medskip

\textbf{3. Covering Numbers and \(\epsilon\)-Net Construction and Error Decomposition.}

Let \( \mathcal{F}_\epsilon \) be an \( \epsilon \)-net for the function class \( \mathcal{F} \) with respect to the \( \norminf{\cdot} \) norm, with size \( \mathcal{N}_{\mathcal{F}}(\epsilon) = \mathcal{N}(\epsilon, \mathcal{F}, \norminf{\cdot}) \).
Since the empirical risk minimizer \( \hat{f} \in \mathcal{F} \), there exists \( f_\epsilon \in \mathcal{F}_\epsilon \) such that \[ \norminf{\hat{f} - f_\epsilon} \leq \epsilon. \]

We decompose the target error \( \normtwo[,\nu]{\hat{f} - f^*}^2 \) using the triangle inequality and the inequality \( (a + b)^2 \leq 2a^2 + 2b^2 \):
\begin{align}
\normtwo[,\nu]{\hat{f} - f^*}^2 &\leq \left( \normtwo[,\nu]{\hat{f} - f_\epsilon} + \normtwo[,\nu]{f_\epsilon - f^*} \right)^2 \nonumber \\
&\leq 2 \normtwo[,\nu]{\hat{f} - f_\epsilon}^2 + 2 \normtwo[,\nu]{f_\epsilon - f^*}^2 \nonumber \\
&\leq 2 \epsilon^2 + 2 \normtwo[,\nu]{f_\epsilon - f^*}^2. \label{eq:error_decomp}
\end{align}
The goal is to establish a high-probability bound for \( \normtwo[,\nu]{f_\epsilon - f^*}^2 \).

\medskip

\textbf{4. Application of Bernstein's Inequality over the Net}

For any fixed \( f \in \mathcal{F} \), the random variables \( e_1^f, \dots, e_n^f \) are independent, with mean \( \mathbb{E}_{\nu}[e_i^f] = \normtwo[,\nu]{f - f^*}^2 \) and variance bounded by \( \mathrm{var}_{\nu}(e_i^f) \leq \mathbb{E}_{\nu}[(e_i^f)^2] \leq 4Y^2 \normtwo[,\nu]{f - f^*}^2 \).

Applying Bernstein's inequality to the empirical average \( \displaystyle\frac{1}{n} \sum_{i=1}^n e_i^f \) for each \( f \in \mathcal{F}_\epsilon \), and using a union bound over all \( \mathcal{N}(\epsilon, \mathcal{F}, \norminf{\cdot}) \) functions, we set \( \delta' = \delta/2 \).

With probability at least \( 1 - \delta' \), the following holds for all \( f \in \mathcal{F}_\epsilon \):
\begin{equation} \label{eq:bernstein_net}
\normtwo[,\nu]{f - f^*}^2 \leq \frac{1}{n} \sum_{i=1}^n e_i^f + \sqrt{\frac{8Y^2 \normtwo[,\nu]{f - f^*}^2 \log(\mathcal{N}_{\mathcal{F}}(\epsilon)/\delta')}{n}} + \frac{4Y^2 \log(\mathcal{N}_{\mathcal{F}}(\epsilon)/\delta')}{3n}.
\end{equation}

\textbf{5. Bounding the Empirical Term Involving \(\hat{f}\).}

Let \[ \bar{f} = \arg\min_{f \in \mathcal{F}} \normtwo[,\nu]{f - f^*}^2 \] 
be the best \( L_2(\nu) \)-approximating function in \( \mathcal{F} \), with \[\normtwo[,\nu]{\bar{f} - f^*}^2 \leq \epsilon_{\mathrm{approx}}^2. \]

Applying Bernstein's inequality to the empirical average \( \frac{1}{n} \sum_{i=1}^n e_i^{\bar{f}} \), with probability at least \( 1 - \delta' \):
\[
\frac{1}{n} \sum_{i=1}^n e_i^{\bar{f}} \leq \normtwo[,\nu]{\bar{f} - f^*}^2 + \sqrt{\frac{8Y^2 \normtwo[,\nu]{\bar{f} - f^*}^2 \log(1/\delta')}{n}} + \frac{4Y^2 \log(1/\delta')}{3n}.
\]
If \( \displaystyle\frac{1}{n} \sum_{i=1}^n e_i^{\bar{f}} \geq \frac{Y^2 \log(1/\delta')}{n} \), this implies:
\[
\displaystyle\frac{1}{n} \sum_{i=1}^n e_i^{\bar{f}} \leq \normtwo[,\nu]{\bar{f} - f^*}^2 + \sqrt{\frac{8 \normtwo[,\nu]{\bar{f} - f^*}^2 \left( \frac{1}{n} \sum_{i=1}^n e_i^{\bar{f}} \right)}{n}} + \frac{4Y^2 \log(1/\delta')}{3n}.
\]
Solving yields:
\[
\displaystyle\frac{1}{n} \sum_{i=1}^n e_i^{\bar{f}} \leq 10 \normtwo[,\nu]{\bar{f} - f^*}^2 + \frac{8Y^2 \log(1/\delta')}{3n}.
\]
If \( \displaystyle\dfrac{1}{n} \sum_{i=1}^n e_i^{\bar{f}} < \dfrac{Y^2 \log(1/\delta')}{n} \), the same bound holds. Thus, with probability at least \( 1 - \delta' \):
\[
\frac{1}{n} \sum_{i=1}^n e_i^{\bar{f}} \leq 10 \normtwo[,\nu]{\bar{f} - f^*}^2 + \frac{8Y^2 \log(1/\delta')}{3n} \leq 10 \epsilon_{\mathrm{approx}}^2 + \frac{8Y^2 \log(1/\delta')}{3n}.
\]

\medskip

\textbf{6. Bounding \(\normtwo[,\nu]{f_\epsilon - f^*}^2\).}

From inequality \eqref{eq:bernstein_net}, for \( f_\epsilon \), with probability at least \( 1 - \delta' \):
\[
\normtwo[,\nu]{f_\epsilon - f^*}^2 \leq \frac{1}{n} \sum_{i=1}^n e_i^{f_\epsilon} + \sqrt{\frac{8Y^2 \normtwo[,\nu]{f_\epsilon - f^*}^2 \log(\mathcal{N}_{\mathcal{F}}(\epsilon)/\delta')}{n}} + \frac{4Y^2 \log(\mathcal{N}_{\mathcal{F}}(\epsilon)/\delta')}{3n}.
\]
Since \( \displaystyle\frac{1}{n} \sum_{i=1}^n e_i^{\hat{f}} \leq \frac{1}{n} \sum_{i=1}^n e_i^{\bar{f}} \), and:\[
\frac{1}{n} \sum_{i=1}^n e_i^{f_\epsilon} - \frac{1}{n} \sum_{i=1}^n e_i^{\hat{f}} = \frac{1}{n} \sum_{i=1}^n (\hat{f}(x_i) - f_\epsilon(x_i))(\hat{f}(x_i) + f_\epsilon(x_i) - 2y_i) \leq 2Y \epsilon,
\]
we substitute the bound for \( \bar{f} \):
\begin{multline*}
\normtwo[,\nu]{f_\epsilon - f^*}^2 \leq  \left( 2Y \epsilon + 10 \epsilon_{\mathrm{approx}}^2 + \frac{8Y^2 \log(2/\delta)}{3n} \right) \\
  + \sqrt{\frac{8Y^2 \normtwo[,\nu]{f_\epsilon - f^*}^2 \log(2\mathcal{N}_{\mathcal{F}}(\epsilon)/\delta)}{n}} + \frac{4Y^2 \log(2\mathcal{N}_{\mathcal{F}}(\epsilon)/\delta)}{3n}.    
\end{multline*}
Solving for \( \normtwo[,\nu]{f_\epsilon - f^*}^2 \), with probability at least \( 1 - \delta \):
\[
\normtwo[,\nu]{f_\epsilon - f^*}^2 \leq \frac{22 Y^2 \log(2\mathcal{N}_{\mathcal{F}}(\epsilon)/\delta)}{n} + 4Y \epsilon + 20 \epsilon_{\mathrm{approx}}^2.
\]

\medskip

\textbf{7. Final Combination.}

Substituting into \eqref{eq:error_decomp}, with probability at least \( 1 - \delta \):
\begin{align*}
\normtwo[,\nu]{\hat{f} - f^*}^2 &\leq 2 \epsilon^2 + 2 \normtwo[,\nu]{f_\epsilon - f^*}^2 \\
&\leq 2 \epsilon^2 + 2 \left( \frac{22 Y^2 \log(2\mathcal{N}(\epsilon, \mathcal{F}, \norminf{\cdot})/\delta)}{n} + 4Y \epsilon + 20 \epsilon_{\mathrm{approx}}^2 \right) \\
&\leq 9Y \epsilon + \frac{44 Y^2 \log(2\mathcal{N}(\epsilon, \mathcal{F}, \norminf{\cdot})/\delta)}{n} + 40 \epsilon_{\mathrm{approx}}^2 \\
&\leq \frac{45 Y^2 \log(2\mathcal{N}(1/n, \mathcal{F}, \norminf{\cdot})/\delta)}{n} + 40 \epsilon_{\mathrm{approx}}^2,
\end{align*}
by choosing \( \epsilon = 1/n \) for large \( n \), yielding the lemma's result.
\end{proof}

\bigskip

\subsection{Proof of Lemma \ref{lemma:iterative_regression}}
\begin{proof}
Our goal is to establish that the error bound holds simultaneously for all \( M \times K \) pairs \((j, k)\),
where \( j \in \{0, \dots, M-1\} \) and \( k \in \{1, \dots, K\} \),
with probability at least \( 1 - \delta \).

To achieve this, we apply Lemma~\ref{lemma:regression_bound} to each pair and employ a union bound,
a process which assigns a failure probability of \( \delta' = \delta / (MK) \) to each pair.

For a fixed pair \((j, k)\), the regression problem satisfies the conditions of Lemma~\ref{lemma:regression_bound}.
Specifically, the least-squares empirical risk minimizer is \( f_k^{j+1} \in \mathcal{F}_k \),
which targets the conditional expectation defined by \( (\mathbf{T}\mathbf{f}^j)_k \).
This target, \( (\mathbf{T}\mathbf{f}^j)_k \), satisfies
\[ (\mathbf{T}\mathbf{f}^j)_k(x) = \mathbb{E}[y \mid x] \;\; \textrm{for}  \;\; x \sim \nu_k, \]
ensuring it plays the role of the true regression function \( f^* \) in Lemma~\ref{lemma:regression_bound}.

The regression is performed using \( n_k \) i.i.d. samples drawn from the distribution \( \nu_k \),
with the function class \( \mathcal{F}_k \).
The \(\epsilon\)-covering number for this setup is given by \( \mathcal{N}_k(\epsilon) = \mathcal{N}(\epsilon, \mathcal{F}_k, \norminf{\cdot}) \).

Additionally, the boundedness conditions hold: functions in \( \mathcal{F}_k \) and the target \( (\mathbf{T}\mathbf{f}^j)_k \)
take values in \( [0, Y] \) almost surely under \( \nu_k \), as guaranteed by Lemma~\ref{lemma:iterative_regression}.

Crucially, the approximation error condition is satisfied. 
Lemma~\ref{lemma:iterative_regression} assumes that for each \( j \) and \( k \), 
there exists an approximation error \( \epsilon_{\mathrm{approx}, k} \ge 0 \) such that:
\[
\inf_{f \in \mathcal{F}_k} \normtwo[,\nu_k]{(\mathbf{T}\mathbf{f}^j)_k - f} \le \epsilon_{\mathrm{approx}, k}.
\]
This matches the approximation error condition in Lemma~\ref{lemma:regression_bound}, 
where \( \inf_{f \in \mathcal{F}} \normtwo[,\nu]{f^* - f} \le \epsilon_{\mathrm{approx}} \), 
with
\[ f^* = (\mathbf{T}\mathbf{f}^j)_k , \; \mathcal{F} = \mathcal{F}_k , \;
\nu = \nu_k , \;\; \text{and} \;\;  \epsilon_{\mathrm{approx}} = \epsilon_{\mathrm{approx}, k}. \] 

Thus, all assumptions of Lemma~\ref{lemma:regression_bound} are met.

Applying Lemma~\ref{lemma:regression_bound} to the pair \((j, k)\) with failure probability \( \delta' = \delta / (MK) \), 
we obtain, with probability at least \( 1 - \delta' \):
\[
\normtwo[,\nu_k]{f_k^{j+1} - (\mathbf{T}\mathbf{f}^j)_k}^2 
\leq \frac{45 Y^2}{n_k} \log\left( \frac{2 \mathcal{N}_k(1/n_k)}{\delta'} \right) + 40 \epsilon_{\mathrm{approx}, k}^2.
\]

Next, we simplify the logarithmic term. 
Since \( \delta' = \delta / (MK) \), we have:
\[
\log\left( \frac{2 \mathcal{N}_k(1/n_k)}{\delta'} \right) 
= \log\left( \frac{2 \mathcal{N}_k(1/n_k) \cdot MK}{\delta} \right).
\]
Substituting this into the bound, we get:
\[
\normtwo[,\nu_k]{f_k^{j+1} - (\mathbf{T}\mathbf{f}^j)_k}^2 
\leq \frac{45 Y^2}{n_k} \log\left( \frac{2 MK \mathcal{N}_k(1/n_k)}{\delta} \right) + 40 \epsilon_{\mathrm{approx}, k}^2.
\]

Finally, we apply the union bound over all \( M \times K \) pairs. 
Since each pair \((j, k)\) has a failure probability of at most \( \delta' \), 
the probability that the bound holds simultaneously for all pairs is at least:
\[
1 - MK \delta' = 1 - MK \cdot \frac{\delta}{MK} = 1 - \delta.
\]
Thus, the lemma is proved.
\end{proof}

\bigskip

\subsection{Proof of Lemma \ref{lemma:q_error_bound}}  
\begin{proof}

We begin by denoting the trajectory of all random variables that arise during an episode under the given transition rules, starting at certain state-action pair \((s_0, a_0)\), stage \(j_0=k\). 

Specifically, let
\[\tau_m = (s_0, a_0, j_0, \gamma_0', s_1, a_1,  j_1, \gamma_1', \dots, s_m, a_m, j_m, \gamma_m' ) \]
represent a trajectory of length \(m\), where each component is determined recursively as follows.

Define \( j_t \) as the stage index associated with state \( s_t \). The single-step discount factor \( \gamma'_{t+1} \) is set to \( \gamma_{j_t} \) if a stage transition occurs at time \(t\) by \( (s_t, a_t) \in \mathcal{T}_{j_t} \), and to \( 1 \) otherwise.

The next state \( s_{t+1} \) is determined by:
\[
s_{t+1} = 
\begin{cases}
\phi_{j_t}(s_t') & \text{if } (s_t, a_t) \in \mathcal{T}_{j_t}, \text{ where } s_t' \sim P_{j_t}(\cdot \mid s_t, a_t), \\
s_t' & \text{otherwise, where } s_t' \sim P_{j_t}(\cdot \mid s_t, a_t).
\end{cases}
\]

The next action \( a_{t+1} \in \mathcal{A}_{j_{t+1}} \) is selected depending on whether the current stage \( j_{t+1} \) is in the update set \( \mathcal{U} \). If \( j_{t+1} \in \mathcal{U} \), the action is chosen to maximize the Q-function error:
\[
a_{t+1} = \argmax_{a \in \mathcal{A}_{j_{t+1}}} \left| Q_{j_{t+1}}^{(m-1-t)}(s_{t+1}, a) - Q_{j_{t+1}}^*(s_{t+1}, a) \right|.
\]
Otherwise, the action is sampled from a fixed policy:
\[
a_{t+1} \sim \pi_{j_{t+1}}^{\circ}(\cdot \mid s_{t+1}).
\]

Let \( \beta_k \in \Delta(\mathcal{S}_k \times \mathcal{A}_k) \) denote an arbitrary admissible distribution over the initial state-action pair \( (s_0, a_0) \) at stage \(k\),
and let \( \Gamma_t = \prod_{i=1}^t \gamma'_i \) represent the cumulative discount factor, with \( \Gamma_0 = 1 \).

Since each stage \( k \) has a finite horizon \( H_k \), and $H = \sum H_k$, we have \( \Gamma_m \leq \gamma_{\mathrm{cycle}}^{\lfloor m / H \rfloor} \) almost surely.
 
Define the error terms: single-step error function
\[ \epsilon_{k,m} := \left| Q_k^{(m)} - (\mathbf{T}_{\mathcal{U}}\mathbf{Q}^{(m-1)})_k \right| \]
and total error function
\[ \delta_{k,m} := \left| Q_k^{(m)} - Q_k^* \right|, \]
where $Q^*$ denotes the $\mathcal{U}$-constrained optimal Q-function.

From Lemma~\ref{lemma:iterative_regression}, for any admissible \( \beta_k \), we have \( \|\epsilon_{k,m}\|_{2,\beta_k} \leq \sqrt{C} \epsilon_k \) with high probability, where $\epsilon_k$ is defined from that Lemma. Also, \( \delta_{k,m} \leq Y \) almost surely.

Next, we decompose the total error \( \delta_{k,m} \):
\begin{align*}
& \; \delta_{k,m}(s_0, a_0) \\
&\leq \left| Q_k^{(m)}(s_0, a_0) - (\mathbf{T}_{\mathcal{U}}\mathbf{Q}^{(m-1)})(s_0, a_0) \right| + \left| (\mathbf{T}_{\mathcal{U}}\mathbf{Q}^{(m-1)})(s_0, a_0) - Q_k^*(s_0, a_0) \right| \\
&\leq \epsilon_{k,m}(s_0, a_0) + \mathbb{E} \left[ \gamma'_1 \delta_{j_1, m-1}(s_1, a_1) \mid (s_0, a_0) \right] \\
&\leq \epsilon_{k,m}(s_0, a_0) + \mathbb{E} \left[ \gamma'_1 \epsilon_{j_1, m-1}(s_1, a_1) \mid (s_0, a_0) \right] + \mathbb{E} \left[ \gamma'_1 \gamma'_2 \delta_{j_2, m-2}(s_2, a_2) \mid (s_0, a_0) \right] \\
&= \epsilon_{k,m}(s_0, a_0) + \mathbb{E} \left[ \Gamma_1 \epsilon_{j_1, m-1}(s_1, a_1) \mid (s_0, a_0) \right] + \mathbb{E} \left[ \Gamma_2 \delta_{j_2, m-2}(s_2, a_2) \mid (s_0, a_0) \right] \\
&\quad \vdots \\
&\leq \sum_{t=0}^{m-1} \mathbb{E} \left[ \Gamma_t \epsilon_{j_t, m-t}(s_t, a_t) \mid (s_0, a_0) \right] + \mathbb{E} \left[ \Gamma_m \delta_{j_m, 0}(s_m, a_m) \mid (s_0, a_0) \right] \\
&\leq \sum_{t=0}^{m-1} \mathbb{E} \left[ \Gamma_t \epsilon_{j_t, m-t}(s_t, a_t) \mid (s_0, a_0) \right] + \gamma_{\mathrm{cycle}}^{\lfloor m / H \rfloor} Y,
\end{align*}
where the expectation \( \mathbb{E}[\cdot \mid (s_0, a_0)] \) represents the expectation over the trajectory \( \tau_t = (s_1, a_1, \dots, s_t, a_t) \) conditional on \( (s_0, a_0) \).

Let \( \varepsilon := \max_k \epsilon_k \). Then, bounding the norm of the conditional expectation term as shown in the provided derivation:
\begin{align*}
\left\| \mathbb{E} (\epsilon_{j_t, m-t}(s_t, a_t) \mid (s_0, a_0) ) \right\|_{2,\beta_k}
&\leq \sqrt{ \mathbb{E}_{(s_0, a_0) \sim \beta_k} \left[ \mathbb{E} ( \epsilon^2_{j_t, m-t}(s_t, a_t) \mid (s_0, a_0) ) \right] } \\
&\leq \sqrt{ \mathbb{E}_{(s_0, a_0) \sim \beta_k} [ C \epsilon_{j_t}^2 ] } \\
&\leq \sqrt{C} \varepsilon.
\end{align*}

The first inequality follows from Jensen’s inequality. For the second inequality, observe that:
\begin{align*}
\mathbb{E} ( \epsilon^2_{j_t, m-t}(s_t, a_t) \mid (s_0, a_0) )  
&= \sum\limits_{k=1}^K \mathbb{E} \left( \epsilon^2_{k, m-t}(s_t, a_t) \mathbb{I}(j_t = k) \mid (s_0, a_0) \right) \\
&= \sum\limits_{k=1}^K \mathbb{E}_{\beta_{k,t}} \left( \epsilon^2_{k, m-t}(s_t, a_t) \mid (s_0, a_0) \right) \mathbb{P}(j_t = k \mid (s_0, a_0)),
\end{align*}
where \(\beta_{k,t}\) denotes the conditional distribution over the state-action pair at time step \(t\), given that the initial state-action pair is \((s_0, a_0)\) and that \(j_t = k\).

Taking expectation over \((s_0, a_0) \sim \beta_k\), we have:
\begin{multline*}
\mathbb{E}_{(s_0, a_0) \sim \beta_k} \left[ \sum\limits_{k'=1}^K \mathbb{E}_{\beta_{k',t}} ( \epsilon^2_{k', m-t}(s_t, a_t) \mid (s_0, a_0) ) \mathbb{P}(j_t = k' \mid (s_0, a_0)) \right] \\
\leq \sum\limits_{k'=1}^K \mathbb{E}_{\beta'_{k',t}} ( \epsilon^2_{k', m-t}(s_t, a_t) ) \mathbb{P}_{\beta_k}(j_t = k') \leq \sum\limits_{k'=1}^K C \epsilon_{k'}^2 \mathbb{P}_{\beta_k}(j_t = k') = C \varepsilon^2,
\end{multline*}
where \(\beta'_{k,t}\) is the distribution of the state-action pair on stage \(k\) at time \(t\), conditioned on the event that the initial pair is drawn from \(\beta_k\) and that \(j_t = k\). This distribution is admissible, and the bound follows from Assumption~\ref{assumption:coverage} and Lemma~\ref{lemma:iterative_regression}, applied as \(\mathbf{f}^j = Q^{(j)}, \mathbf{T} = \mathbf{T}_{\mathcal{U}}\).

Substituting this bound into the previous estimate for \( \delta_{k,m} \), we conclude that for any admissible \( \beta_k \), with probability at least \( 1 - \delta \),
\[
\left\| Q_k^* - Q_k^{(m)} \right\|_{2,\beta_k} \leq \sum_{t=0}^{m-1} \gamma_{\mathrm{cycle}}^{\lfloor t / H \rfloor} \sqrt{C} \varepsilon + \gamma_{\mathrm{cycle}}^{\lfloor m / H \rfloor} Y \leq \frac{\sqrt{C} H \varepsilon}{1 - \gamma_{\mathrm{cycle}}} + \gamma_{\mathrm{cycle}}^{\lfloor m / H \rfloor} Y.
\]
\end{proof}

\bigskip

\subsection{Proof of Lemma \ref{lemma:performance_difference}}
\begin{proof}
Let \( \Pr^\pi(\tau \mid s_0 = s) \) denote the distribution over trajectories
\[
\tau = (s_0, a_0, j_0, \gamma'_0, s_1, a_1, j_1, \gamma'_1, \dots)
\]
induced by the composite policy \( \pi \) across \(K\) stages, starting from \( s_0 = s \in \mathcal{S}_k \).

Here, \( j_t \) denotes the stage index associated with state \( s_t \), and \( \gamma'_{t+1} \) is the one-step discount factor defined as \( \gamma_{j_t} \) if a stage transition occurs between \( s_t \) and \( s_{t+1} \), and \( 1 \) otherwise.

The cumulative discount factor is given by
\(
\Gamma_t = \prod_{i=1}^t \gamma'_i,
\quad \text{with } \Gamma_0 = 1.
\)

The expectation \( \mathbb{E}_{\tau \sim \Pr^\pi(\tau \mid s_0 = s)} \) is taken over all random variables \( (s_t, a_t, j_t, \Gamma_t) \) along the trajectory generated by \( \pi \).

For a fixed starting state \( s \), we apply a telescoping argument:
\begin{align*}
&\;\; V_k^\pi(s) - V_k^{\pi'}(s) \\
&= \mathbb{E}_{\tau \sim \Pr^\pi(\tau \mid s_0 = s)} \left[ \sum_{t=0}^\infty \Gamma_t r_{j_t}(s_t, a_t) \right] - V_k^{\pi'}(s) \\
&= \mathbb{E}_{\tau \sim \Pr^\pi(\tau \mid s_0 = s)} \left[ \sum_{t=0}^\infty \Gamma_t \left( r_{j_t}(s_t, a_t) + V_{j_t}^{\pi'}(s_t) - V_{j_t}^{\pi'}(s_t) \right) \right] - V_k^{\pi'}(s) \\
&\stackrel{(a)}{=} \mathbb{E}_{\tau \sim \Pr^\pi(\tau \mid s_0 = s)} \left[ \sum_{t=0}^\infty \Gamma_t \left( r_{j_t}(s_t, a_t) + \gamma'_{t+1} V_{j_{t+1}}^{\pi'}(s_{t+1}) - V_{j_t}^{\pi'}(s_t) \right) \right] \\
&\stackrel{(b)}{=} \mathbb{E}_{\tau \sim \Pr^\pi(\tau \mid s_0 = s)} \left[ \sum_{t=0}^\infty \Gamma_t \left( r_{j_t}(s_t, a_t) + \mathbb{E} \left[ \gamma'_{t+1} V_{j_{t+1}}^{\pi'}(s_{t+1}) \mid s_t, a_t \right] - V_{j_t}^{\pi'}(s_t) \right) \right] \\
&\stackrel{(c)}{=} \mathbb{E}_{\tau \sim \Pr^\pi(\tau \mid s_0 = s)} \left[ \sum_{t=0}^\infty \Gamma_t \left( Q_{j_t}^{\pi'}(s_t, a_t) - V_{j_t}^{\pi'}(s_t) \right) \right] \\
&= \mathbb{E}_{\tau \sim \Pr^\pi(\tau \mid s_0 = s)} \left[ \sum_{t=0}^\infty \Gamma_t A_{j_t}^{\pi'}(s_t, a_t) \right],
\end{align*}
where,

step (a) uses the Bellman structure of \( V^{\pi'} \) to rearrange terms via telescoping;

step (b) applies the law of iterated expectations;

and step (c) follows from the definition of \( Q_{j_t}^{\pi'} \). The final equality uses the advantage function defined in the lemma.

Taking the expectation over the initial state distribution \( s \sim \mu_k \):
\begin{align*}
v_k^\pi(\mu_k) - v_k^{\pi'}(\mu_k) 
&= \mathbb{E}_{s \sim \mu_k} \left[ V_k^\pi(s) - V_k^{\pi'}(s) \right] \\
&= \mathbb{E}_{s \sim \mu_k} \left[ \mathbb{E}_{\tau \sim \Pr^\pi(\tau \mid s_0 = s)} \left[ \sum_{t=0}^\infty \Gamma_t A_{j_t}^{\pi'}(s_t, a_t) \right] \right].
\end{align*}
Since each stage \( m \) has a finite horizon \( H_m \), we have \( \Gamma_{t+H} \leq \gamma_{\mathrm{cycle}} \Gamma_t \) almost surely. This property, involving the cycle-based discount term \( \gamma_{\mathrm{cycle}}^{\lfloor t/H \rfloor} \), yields the result stated in the lemma.
\end{proof}

\bigskip

\section{Proofs of the Main Theorems}

\subsection{Proof of Theorem \ref{thm:main_convergence}}
\begin{proof}
We begin by invoking Lemma~\ref{lemma:performance_difference}, the performance difference lemma tailored to our multi-stage cyclic setting.

Let $v_k^*$ denote the optimal value for stage $k$, defined with respect to the fixed update set $\mathcal{U}$ and the baseline policies $\pi_k^\circ$. The lemma states:
\[
v_k^* - v_k^{\pi^{(M)}} = \sum_{t=0}^\infty \gamma_{\mathrm{cycle}}^{\lfloor t/H \rfloor} \mathbb{E}_{s \sim \eta_k} \mathbb{E}_{\tau \sim \Pr^{\pi^{(M)}}(\tau \mid s_0 = s)} \left[ A_{j_t}^{*} (s_t, a_t) \right].
\]
Here the action $a_t$ is sampled according to $a_t \sim \pi_{j_t}^{(M)}(\cdot \mid s_t)$, where $j_t$ is the stage index at time step $t$, and the advantage function relative to the $\mathcal{U}$-constrained optimal policy $\pi_{\mathcal{U}}^*$ is denoted by $A_{j_t}^{*}(s_t, a_t)$ and defined as:
$$ A_{j_t}^{*}(s_t, a_t) = Q_{j_t}^*(s_t, a_t) - V_{j_t}^*(s_t).$$

Note that for stages outside the constrained set with $k \notin \mathcal{U}$, the policies coincide: $\pi_k^{(M)} = \pi_k^\circ = \pi_{k,\mathcal{U}}^*$.

We now analyze the expected advantage term $\mathbb{E}[A_{j_t}^{\pi_{\mathcal{U}}^*}(s_t, a_t)]$ appearing in the summation. When $j_t \notin \mathcal{U}$, the executed policy $\pi_{j_t}^{(M)}$ coincides with $\pi_{j_t}^\circ$, which in turn is equal to the optimal constrained policy $\pi_{j_t,\mathcal{U}}^*$.

Thus, the action $a_t$ is sampled from the optimal action distribution under $\pi_{\mathcal{U}}^*$, and the expected advantage vanishes:
\begin{align*}
\mathbb{E}_{a_t \sim \pi_{j_t}^{(M)}(\cdot \mid s_t)}[A_{j_t}^*(s_t, a_t) ] \mathbb{I}(j_t \notin \mathcal{U}) &= \mathbb{E}_{a_t \sim \pi_{j_t}^{\circ}(\cdot \mid s_t)} [Q_{j_t}^*(s_t, a_t) - V_{j_t}^*(s_t)] \mathbb{I}(j_t \notin \mathcal{U}) \\
&= [V_{j_t}^*(s_t) - V_{j_t}^*(s_t)] \mathbb{I}(j_t \notin \mathcal{U}) \\
&= 0.
\end{align*}
Hence, the summation in the performance difference expression effectively reduces to terms where $j_t \in \mathcal{U}$. For these terms, using the identity $V_{j_t}^*(s_t) = Q_{j_t}^{*}(s_t, \pi_{j_t}^{*}(s_t))$, we rewrite the advantage as:
\[
Q_{j_t}^{*}(s_t, \pi_{j_t}^{*}(s_t)) - Q_{j_t}^{*}(s_t, \pi_{j_t}^{(M)}(s_t)).
\]
The performance difference then becomes:
\begin{multline*}
v_k^* - v_k^{\pi^{(M)}} \\
= \sum_{t=0}^\infty \gamma_{\mathrm{cycle}}^{\lfloor t/H \rfloor} \mathbb{E}_{s \sim \eta_k} \mathbb{E}_{\tau \sim \Pr^{\pi^{(M)}}(\tau \mid s_0 = s)} \left[ \left( Q_{j_t}^{*}(s_t, \pi_{j_t}^{*}(s_t)) - Q_{j_t}^{*}(s_t, \pi_{j_t}^{(M)}(s_t)) \right) \mathbb{I}(j_t \in \mathcal{U}) \right].
\end{multline*}

For each term with $j_t \in \mathcal{U}$, we introduce and subtract the estimated Q-function \(Q^{(M)}\):
\begin{align*}
Q_{j_t}^{*}(s_t, \pi_{j_t}^{*}(s_t)) - Q_{j_t}^{*}(s_t, \pi_{j_t}^{(M)}(s_t))
&= \left[ Q_{j_t}^{*}(s_t, \pi_{j_t}^{*}(s_t)) - Q_{j_t}^{(M)}(s_t, \pi_{j_t}^{*}(s_t)) \right] \\
&\quad + \left[ Q_{j_t}^{(M)}(s_t, \pi_{j_t}^{(M)}(s_t)) - Q_{j_t}^{*}(s_t, \pi_{j_t}^{(M)}(s_t)) \right] \\
&\quad + \left[ Q_{j_t}^{(M)}(s_t, \pi_{j_t}^{*}(s_t)) - Q_{j_t}^{(M)}(s_t, \pi_{j_t}^{(M)}(s_t)) \right].
\end{align*}
Since $\pi^{(M)}$ is greedy with respect to $Q^{(M)}$ for all $j_t \in \mathcal{U}$, the final term is non-positive. Therefore,
\begin{multline*}
Q_{j_t}^{*}(s_t, \pi_{j_t}^{*}(s_t)) - Q_{j_t}^{*}(s_t, \pi_{j_t}^{(M)}(s_t)) \\
\leq \left| Q_{j_t}^{*}(s_t, \pi_{j_t}^{*}(s_t)) - Q_{j_t}^{(M)}(s_t, \pi_{j_t}^{*}(s_t)) \right| + \left| Q_{j_t}^{*}(s_t, \pi_{j_t}^{(M)}(s_t)) - Q_{j_t}^{(M)}(s_t, \pi_{j_t}^{(M)}(s_t)) \right|.
\end{multline*}
Substituting this bound back into the performance difference expression, we obtain:
\begin{multline*}
v_k^* - v_k^{\pi^{(M)}} \leq \sum_{t=0}^\infty \gamma_{\mathrm{cycle}}^{\lfloor t/H \rfloor} \mathbb{E}_{s \sim \eta_k} \mathbb{E}_{\tau \sim \Pr^{\pi^{(M)}}(\tau \mid s_0 = s)} \Big[ \left| Q_{j_t}^{*} - Q_{j_t}^{(M)} \right|(s_t, \pi_{j_t}^{*}(s_t)) \\
 + \left| Q_{j_t}^{*} - Q_{j_t}^{(M)} \right|(s_t, \pi_{j_t}^{(M)}(s_t)) \Big] \mathbb{I}(j_t \in \mathcal{U}).
\end{multline*}

Let $\beta_{k,t}^{(M)*}$ denote the conditional distribution over $(s_t, a_t)$, givne that $j_t = k$ and the state $s_t$ is generated under $\pi^{(M)}$ from $\eta_k$, and the action is $a_t = \pi_k^*(s_t)$. Similarly, define $\beta_{k,t}^{(M)(M)}$ where $a_t \sim \pi_k^{(M)}(\cdot \mid s_t)$.

Applying $L_1$ norms, switching to \(L_2\) norm using the Cauchy–Schwarz inequality, and the high-probability bound from Lemma~\ref{lemma:q_error_bound}, we have:
\begin{align*}
&\mathbb{E}_{\tau \sim \Pr^{\pi^{(M)}}} \left[  \left| Q_{j_t}^{*} - Q_{j_t}^{(M)} \right|(s_t, \pi_{j_t}^{*}(s_t))
 + \left| Q_{j_t}^{*} - Q_{j_t}^{(M)} \right|(s_t, \pi_{j_t}^{(M)}(s_t)) \right] \mathbb{I}(j_t \in \mathcal{U}) \\
&\le \sum_{k \in \mathcal{U}} \left( \| Q_k^* - Q_k^{(M)} \|_{1, \beta_{k,t}^{(M)*}} + \| Q_k^* - Q_k^{(M)} \|_{1, \beta_{k,t}^{(M)(M)}} \right) \mathbb{P}(j_t = k \mid s_0 = s) \\
&\le \sum_{k \in \mathcal{U}} \left( \| Q_k^* - Q_k^{(M)} \|_{2, \beta_{k,t}^{(M)*}} + \| Q_k^* - Q_k^{(M)} \|_{2, \beta_{k,t}^{(M)(M)}} \right) \mathbb{P}(j_t = k \mid s_0 = s) \\
&\le \sum_{k \in \mathcal{U}} 2 \left( \frac{\sqrt{C} H \varepsilon}{1 - \gamma_{\mathrm{cycle}}} + \gamma_{\mathrm{cycle}}^{\lfloor M/H \rfloor} Y \right) \mathbb{P}(j_t = k \mid s_0 = s).
\end{align*}

Plugging this into the bound for the performance difference:
\begin{align*}
v_k^* - v_k^{\pi^{(M)}} &\le \sum_{t=0}^\infty \gamma_{\mathrm{cycle}}^{\lfloor t/H \rfloor} \mathbb{E}_{s \sim \eta_k} \left[ \sum_{k \in \mathcal{U}} 2 \left( \frac{\sqrt{C} H \varepsilon}{1 - \gamma_{\mathrm{cycle}}} + \gamma_{\mathrm{cycle}}^{\lfloor M/H \rfloor} Y \right) \mathbb{P}(j_t = k \mid s_0 = s) \right] \\
&\le \sum_{t=0}^\infty \gamma_{\mathrm{cycle}}^{\lfloor t/H \rfloor} \cdot 2 \left( \frac{\sqrt{C} H \varepsilon}{1 - \gamma_{\mathrm{cycle}}} + \gamma_{\mathrm{cycle}}^{\lfloor M/H \rfloor} Y \right).
\end{align*}

Using the geometric series bound
\(
\displaystyle\sum_{t=0}^\infty \gamma_{\mathrm{cycle}}^{\lfloor t/H \rfloor} \le \frac{H}{1 - \gamma_{\mathrm{cycle}}},
\)
we finally obtain:
\[
v_k^* - v_k^{\pi^{(M)}} \le \frac{2 H}{1 - \gamma_{\mathrm{cycle}}} \left[ \frac{\sqrt{C} H \varepsilon}{1 - \gamma_{\mathrm{cycle}}} + \gamma_{\mathrm{cycle}}^{\lfloor M/H \rfloor} Y \right].
\]
This bound holds uniformly across all initial stages $k \in \{1,\dots,K\}$, proving the theorem for $\|\mathbf{v}^* - \mathbf{v}^{(M)}\|_\infty$.
\end{proof}

\bigskip

\subsection{Proof of Theorem \ref{thm:finite_sample_rate}}
\begin{proof}
We aim to establish the finite-sample error rate for the uniform suboptimality gap, as stated in Theorem~\ref{thm:finite_sample_rate}.

According to Theorem~\ref{thm:main_convergence}, the iteration-dependent term \( \gamma_{\mathrm{cycle}}^{M/H} \) vanishes when \( M \asymp \mathrm{poly}(n) \), where \( n = \sum_{k=1}^K n_k \) is the total sample size. 
Thus, in this regime, the leading-order terms dominate the overall error, and we analyze these terms under Assumption~\ref{assumption:besov}.

Under this assumption, the complexity of the function class \( \mathcal{F}_k \) is characterized by a model-specific parameter \( D_k \), which governs both the covering number \( \mathcal{N}_k(\epsilon) \) and the approximation error \( \epsilon_{\mathrm{approx}, k} \). 

In particular, we have the expressions:
\[
\log \mathcal{N}_k(\epsilon) = \mathcal{O}\left( D_k \log\left( \frac{\mathrm{poly}(D_k)}{\epsilon} \right) \right), 
\quad \epsilon_{\mathrm{approx}, k} = \mathcal{O}(D_k^{-\alpha_k}),
\]
where \( \alpha_k > 0 \) denotes a regularity parameter associated with \( \mathcal{F}_k \).

\smallskip

\noindent
\textit{Note that} these expressions do not impose any additional conditions beyond Assumption~\ref{assumption:besov}; they are automatically satisfied by selecting an appropriate approximation class. 
Table~\ref{tab:approx_rates_covering} provides examples of such classes, and further details can be found in Appendix~\ref{sec:discussion_approx}.

\smallskip

To obtain the error rate, we analyze the term \( \epsilon_k \) from Theorem~\ref{thm:main_convergence}, which consists of a regression error
\[
\sqrt{\dfrac{1}{n_k} \log \mathcal{N}_k(1/n_k)}
\quad \text{and} \quad
\epsilon_{\mathrm{approx}, k}.
\]
Substituting the complexity expressions above, we find
\[
\log \mathcal{N}_k(1/n_k) = \mathcal{O}\left( D_k \log\left( \mathrm{poly}(D_k) \cdot n_k \right) \right),
\]
so that the regression error becomes
\[
\sqrt{\frac{1}{n_k} \log \mathcal{N}_k(1/n_k)} 
= \mathcal{O}\left( \sqrt{\frac{D_k \log(\mathrm{poly}(D_k) \cdot n_k)}{n_k}} \right).
\]
Meanwhile, the approximation error remains
\[
\epsilon_{\mathrm{approx}, k} = \mathcal{O}(D_k^{-\alpha_k}).
\]

To minimize the total error, we balance the two terms by setting them to be of comparable order:
\[
\sqrt{\frac{D_k \log(\mathrm{poly}(D_k) \cdot n_k)}{n_k}} \asymp D_k^{-\alpha_k}.
\]
Neglecting logarithmic factors, this yields the simplified relation:
\[
\sqrt{\frac{D_k}{n_k}} \asymp D_k^{-\alpha_k},
\]
which implies
\[
D_k \asymp n_k^{1/(2\alpha_k + 1)}.
\]

Substituting this back into the expressions for both terms, we obtain the following rate:
\[
\epsilon_k = \mathcal{O}\left( n_k^{-\alpha_k / (2\alpha_k + 1)} \right).
\]
Since the uniform suboptimality gap \( \varepsilon \) is determined by the worst-case error across all stages, we have
\[
\varepsilon = \mathcal{O}\left( \max_k n_k^{-\alpha_k / (2\alpha_k + 1)} \right).
\]

Finally, the constants hidden in the big-\( \mathcal{O} \) notation depend only on the environment and the function classes \( \mathcal{M}_k \) and \( \mathcal{F}_k \), and not on the sample sizes \( n_k \) or the complexity parameters \( D_k \), which completes the proof.
\end{proof}

\begin{remark}
Note that although we use big-\( \mathcal{O} \) notation for simplicity, the result is non-asymptotic and remains valid for any finite sample size \( n_k \).
\end{remark}

\bigskip

\subsection{Proof of Theorem \ref{thm:random_forest}}

\begin{proof}
We begin by recalling that, under the Lipschitz continuity assumption on the optimal Q-functions, the $\mathcal{U}$-constrained optimal Q-function $Q_k^*$ for each stage $k$ and action $a \in \mathcal{A}_k$ is Lipschitz continuous over the domain $\mathcal{S}_k \subseteq [0,1]^{d_k}$.

This continuity is characterized by a uniform upper bound $L$ on the Lipschitz constant across all stages and actions.

In Assumption \ref{assumption:absolute_continuity}, we further assume that the data distribution \(\nu_k\) at each stage \(k\) is absolutely continuous with respect to the product measure \(\lambda_k \times \mu_k\), where \(\lambda_k\) denotes the Lebesgue measure on the state space \(\mathcal{S}_k\), and \(\mu_k\) is the counting measure on the finite action space \(\mathcal{A}_k\).

This absolute continuity ensures that \(\nu_k\) admits a density, so that marginal and conditional distributions are well-defined.

\textbf{1. Sufficient samples for each action.}

For each action \(a \in \mathcal{A}_k\), let \(n_{k,a}\) denote the number of samples corresponding to this action within the stage-\(k\) data set \(\mathcal{D}_k\). The total number of samples at this stage is
\[
n_k = \sum_{a \in \mathcal{A}_k} n_{k,a}.
\]

The marginal probability of selecting action \(a\) under the data distribution \(\nu_k\) is given by
\[
p_{k,a} = \int_{\mathcal{S}_k} \nu_k(s_k, a) \, d\lambda_k(s_k).
\]

Under Assumption \ref{assumption:coverage}, which considers an admissible state-action distribution induced by a uniformly random policy over \(\mathcal{A}_k\), we obtain a lower bound on the marginal probabilities:
\[
p_{k,a} \ge p_{\min,k} \coloneqq \frac{1}{C |\mathcal{A}_k|} > 0, \quad \forall a \in \mathcal{A}_k.
\]

This guarantees a minimum level of exploration for each action within the data set \(\mathcal{D}_k\).

Given the total sample size \(n_k\) and the marginal action probabilities \((p_{k,a})_{a \in \mathcal{A}_k}\), the sample counts \((n_{k,a})_{a \in \mathcal{A}_k}\) follow a multinomial distribution:
\[
(n_{k,a})_{a \in \mathcal{A}_k} \sim \text{Multi}(n_k, (p_{k,a})_{a \in \mathcal{A}_k}).
\]

We are interested in the case where each action is sufficiently represented. Define the “bad event” \(E_k\) as the event that there exists at least one action with a sample count below half of its expected proportion:
\[
E_k \coloneqq \left\{ \exists a \in \mathcal{A}_k \text{ such that } n_{k,a} < \frac{n_k p_{\min,k}}{2} \right\} = \left\{ \min_{a \in \mathcal{A}_k} n_{k,a} < \frac{n_k p_{\min,k}}{2} \right\}.
\]

Using standard concentration inequalities for multinomial distributions, there exist constants \(C_{1,k}, C_{2,k} > 0\) such that the probability of this undesirable event is bounded by
\begin{equation} \label{eq:bad_event_prob_rf_v4}
\mathbb{P}(E_k) \le C_{1,k} \exp(-C_{2,k} n_k p_{\min,k}).
\end{equation}

Thus, with high probability, that is, on the complement event \(E_k^c\), we have that \(n_{k,a} \ge n_k p_{\min,k}/2\) for all \(a \in \mathcal{A}_k\), ensuring that each action receives a sufficient number of samples.

Moreover, the absolute continuity of \(\nu_k\) with respect to \(\lambda_k \times \mu_k\) implies that the conditional distribution \(\nu_k(\cdot|a)\) admits a density with respect to \(\lambda_k\) for every action \(a \in \mathcal{A}_k\) such that \(p_{k,a} > 0\).

\textbf{2. \(L_2\)-error bound on Q-function.}

We now consider the update step in Algorithm \ref{alg:cfqi}, where Random Forest is used to fit Q-functions at each stage \(k \in \{1,\dots,K\}\), action \(a \in \mathcal{A}_k\), and iteration step \(m \in \{0,\dots,M-1\}\).

According to Proposition 2.2 of \cite{biau2012analysis}, given the sample size \(n_{k,a}\) for a fixed stage and action, we may select the number of terminal nodes in the Random Forest proportional to
\[
\Theta\left(n_{k,a}^{- \frac{1}{1+0.75/(d_k \log 2)}}\right),
\]
so that the expected squared error satisfies:
\[
\mathbb{E}_{\mathcal{D}_{k,a}} \left[ \left( Q_k^{(m)}(\cdot, a) - (\mathbf{TQ}^{(m-1)})_k (\cdot, a) \right)^2 \mid n_{k,a} \right] = \mathcal{O}\left( n_{k,a}^{-\frac{0.75}{d_k \log 2 + 0.75}} \right).
\]

Since on the event \(E_k^c\) we have \(n_{k,a} \ge n_k p_{\min,k}/2\), this implies that
\[
\mathcal{O}\left(n_k^{-\frac{0.75}{d_k \log 2 + 0.75}}\right)
\]
is a uniform bound for all actions.

Taking the expectation over the full data set \(\mathcal{D}_k\), and accounting for the small probability of \(E_k\), we obtain the unconditional bound:
\[
\mathbb{E}_{\nu_k} \left[ \left( Q_k^{(m)} - (\mathbf{TQ}^{(m-1)})_k \right)^2 \right] = \mathcal{O}\left( n_k^{-\frac{0.75}{d_k \log 2 + 0.75}} \right).
\]

Using the error propagation argument from Lemma \ref{lemma:q_error_bound}, adapted to expected bounds, we have
\[
\sqrt{\mathbb{E}_{\nu_k} \left\| Q_k^* - Q_k^{(m)} \right\|^2 } = \dfrac{\sqrt{C}H \epsilon_k}{1-\gamma_{\mathrm{cycle}}} + \gamma_{\mathrm{cycle}}^{\lfloor m/H \rfloor} Y,
\]
where
\[
\epsilon_k = \mathcal{O}\left( n_k^{-\frac{0.375}{d_k \log 2 + 0.75}} \right).
\]

\textbf{3. Expected suboptimality gap for \(\algname\) with Random Forest.}

Substituting the above result into Lemma \ref{lemma:performance_difference} and applying Theorem \ref{thm:main_convergence}, we obtain the following bound on the expected value error:
\[
\mathbb{E}_{\mathcal{D}}\|\mathbf{v}^* - \mathbf{v}^{(M)}\|_\infty \leq \frac{2 \sqrt{C} H^2 \varepsilon}{(1 - \gamma_{\mathrm{cycle}})^2} + \frac{2H Y}{1 - \gamma_{\mathrm{cycle}}} \cdot \gamma_{\mathrm{cycle}}^{\lfloor M/H \rfloor},
\]
where \(H = \sum_{i=1}^K H_i\) and
\[
\varepsilon = \max_k \epsilon_k = \mathcal{O}\left( Y \max_k n_k^{-\frac{0.375}{d_k \log 2 + 0.75}} \right).
\]

Here, the expectation \(\mathbb{E}_{\mathcal{D}}\) is over the entire data set \(\mathcal{D} = \{\mathcal{D}_k\}_{k=1}^K\), with each \(\mathcal{D}_k\) sampled from \(\nu_k\).
\end{proof}

\begin{remark}
Although we have assumed Lipschitz continuity of the optimal Q-function, corresponding to a Besov space of smoothness order 1, this assumption is common in the analysis of Random Forests.

It is likely that the results extend to broader function classes such as higher-order Sobolev or Besov spaces, though such generalizations are beyond the scope of this work.
\end{remark}

\bigskip

\subsection{Proof of Theorem \ref{thm:asymptotic}}

The proof adapts the approach from Theorems 1 and 2 of~\cite{shi2022statistical}, 
extending it to address the multi-stage structure of our setting.

Specifically, we generalize the analysis from a scalar value and its standard deviation 
to a vector of values and the corresponding covariance matrix. 
This multivariate extension poses new challenges, 
as it is not a direct application of the original results.

\begin{proof}

\textbf{Step 1: Analysis of the Coefficient Matrix.}

For each \( n = 1, \dots, N-1 \),
we estimate the policy \(\hat{\pi}^{n}\) using Algorithm~\ref{alg:cfqi}
on the aggregated data set \(\bar{\mathcal{D}}^n\),
formed by combining the first \( n \) data subsets.
We then evaluate \(\hat{\pi}^n\) using the disjoint \((n+1)\)-th data subset \(\mathcal{D}^{n+1}\),
following the evaluation method in~\eqref{eq:mean_estimate} and~\eqref{eq:variance_estimate}.

During evaluation, we compute the sample coefficient matrix \(\widehat{\mathbf{H}}_n\)
using~\eqref{eq:H_matrix_def}.
Its conditional expectation, \(\mathbf{H}_n = \mathbb{E}[\widehat{\mathbf{H}}_n \mid \bar{\mathcal{D}}^{n-1}]\),
satisfies
\[
\max_{1 \leq n < N} \|\mathbf{H}_n^{-1}\|_2 \leq 2c^{-1},
\]
by Assumption~\ref{assumption:min_eigenvalue}.
Lemma 3 of~\cite{shi2022statistical} develops a comparable bound for similar matrix structures, particularly under conditions analogous to the \( T=1 \) case in that work.

Moreover, under Assumption~\ref{assumption:absolute_continuity}, and adapting arguments from Lemma 2 of~\cite{shi2022statistical}, there exists a constant \( c_* > 0 \) such that for standard sieve bases—such as B-splines (following the proof of Theorem 3.3 from \cite{burman1989nonparametric}) and wavelets (following the proof of Theorem 5.1 from \cite{chen2015optimal})—the concatenated basis vector \(\Phi = (\Phi_1(s_1)^{\top}, \dots, \Phi_K(s_K)^{\top})^{\top}\) satisfies the following:
the eigenvalues of
\[
\int \bigl[ \Phi \Phi^{\top} \bigr] \, d\lambda_1 \dots d\lambda_K
\]
lie within \( \left( \dfrac{1}{c_*}, c_* \right) \),
and the \( L_2 \)-norm is uniformly bounded by \( \sup\|\Phi\|_2 \le c^* L \),
where \( L = \sum_{k=1}^K L_k \) is the total number of basis functions
and \(\lambda_k\) denotes the Lebesgue measure over \(\state_k\).

Further, to analyze the concentration of \( \widehat{\mathbf{H}}_n \) around \( \mathbf{H}_n \), we apply matrix concentration arguments, such as inequalities for sums of independent random matrices or matrix martingales, drawing upon the approach used for Lemma 3 of~\cite{shi2022statistical}. This yields, with probability at least \( 1 - O\left(\dfrac{1}{n^2}\right) \),
\[
\max_{1 \leq n \leq N-1} \|\widehat{\mathbf{H}}_n - \mathbf{H}_n\|_2 \preceq \sqrt{\frac{L}{n}} \log n,
\]
where \(\|\cdot\|_2\) denotes the matrix Frobenius norm.

Under Assumption~\ref{assumption:besov_general} for Theorem~\ref{thm:asymptotic},
where \( L_k \asymp n_k^{d_k/(s_k + d_k)} \) with \( s_k > d_k \),
we have \( L_k = o(\dfrac{\sqrt{n_k}}{\log n_k} ) \).
This condition, combined with the application of the aforementioned matrix concentration principles, allows us to bound the inverse sample coefficient matrices, applying these principles as guided by Lemma 3 of~\cite{shi2022statistical}.
Consequently, with probability at least \( 1 - O\left(\dfrac{1}{n^2}\right) \),
\begin{equation}\label{eq:coefficient_matrix_inverse_bound}
\max_{1 \leq n < N} \|\widehat{\mathbf{H}}_n^{-1}\|_2 \leq \dfrac{3}{c}, \quad
\max_{1 \leq n < N} \|\widehat{\mathbf{H}}_n^{-1} - \mathbf{H}_n^{-1}\|_2 \preceq \sqrt{\frac{L}{n}} \log n.    
\end{equation}

\medskip

\textbf{Step 2: Decomposition of Coefficient Estimation Error.}

By Assumption~\ref{assumption:besov_general}, 
an extension of Assumption~\ref{assumption:besov} to arbitrary constrained policies, 
the \( Q \)-functions \( Q_k^{\hat{\pi}^n} \), 
for all stages \( k \) and the \( n \)-th estimated policy, 
belong to a Besov space with smoothness parameter \( s_k \).

For suitable sieve bases, such as B-splines or wavelets of sufficient resolution 
(see Table~\ref{tab:approx_rates_covering} and Appendix~\ref{sec:discussion_approx}), 
there exist coefficient vectors \(\{\beta_{k,a}^n\}_{a \in \action_k}\) 
such that the global approximation error is bounded by
\[
\max_{1 \leq n < N} \left\| Q_k^{\hat{\pi}^n}(\cdot, \cdot) - \Phi_k^{\top}(\cdot) \beta_{k,a}^n \right\|_\infty \leq C L_k^{-\alpha_k},
\]
where \(\alpha_k = s_k / d_k\), for some constant \( C > 0 \), 
and for each stage \( k \).

We define \(\xi_{k,i}^n\) as the approximation error for the \( i \)-th sample at stage \( k \), 
representing the difference between the target value from the Bellman equation, 
using the true \( Q \)-functions for \(\hat{\pi}^n\), 
and the value implied by the sieve approximation coefficients \(\beta^n\). 
This error is given by
\begin{multline*}
\xi_{k,i}^n = \Biggl[ \gamma_k T_{k,i} \sum_{a \in \action_{[k+1]}} \bigl( \Phi_{[k+1]}^{\top}(\phi_k(s_k'^i)) \beta_{[k+1],a}^n - Q_{[k+1]}^{\hat{\pi}^n}(\phi_k(s_k'^i), a) \bigr) \hat{\pi}_{[k+1]}^n(a \mid \phi_k(s_k'^i)) \\
+ (1 - T_{k,i}) \sum_{a \in \action_k} \bigl( \Phi_k^{\top}(s_k'^i) \beta_{k,a}^n - Q_k^{\hat{\pi}^n}(s_k'^i, a) \bigr) \hat{\pi}_k^n(a \mid s_k'^i) \Biggr] \\
- \bigl( \Phi_k^{\top}(s_k^i) \beta_{k,a_k^i}^n - Q_k^{\hat{\pi}^n}(s_k^i, a_k^i) \bigr).
\end{multline*}
From the \( Q \)-function approximation bound, 
it follows that, for each stage \( k \),
\begin{equation}\label{eq:xi_bound}
\max_{1 \leq n < N, i} |\xi_{k,i}^n| \leq 2 C L_k^{-\alpha_k}.
\end{equation}

Following the approach in Theorem 1 of~\cite{shi2022statistical}, 
we decompose the error in the estimated sieve coefficients for the \( m \)-th policy evaluation step as
\begin{equation}\label{eq:proof_asymptotic_coeff}
\begin{aligned}
\hat{\beta}^m - \beta^m
&= \widehat{\mathbf{H}}_m^{-1} \left( \dfrac{N}{n} \sum_{k=1}^K \sum_{i \in \mathcal{D}_k^m} \psi_{k,i} (\epsilon_{k,i}^m - \xi_{k,i}^m) \right) \\
&= \mathbf{H}_m^{-1} \left( \dfrac{N}{n} \sum_{k=1}^K \sum_{i \in \mathcal{D}_k^m} \psi_{k,i} \epsilon_{k,i}^m \right) \\
&\quad + \left( \widehat{\mathbf{H}}_m^{-1} - \mathbf{H}_m^{-1} \right) \left( \dfrac{N}{n} \sum_{k=1}^K \sum_{i \in \mathcal{D}_k^m} \psi_{k,i} \epsilon_{k,i}^m \right) 
- \widehat{\mathbf{H}}_m^{-1} \left( \dfrac{N}{n} \sum_{k=1}^K \sum_{i \in \mathcal{D}_k^m} \psi_{k,i} \xi_{k,i}^m \right) \\
&= \mathbf{H}_m^{-1} \left( \dfrac{N}{n} \sum_{k=1}^K \sum_{i \in \mathcal{D}_k^m} \psi_{k,i} \epsilon_{k,i}^m \right) 
+ \mathcal{O}_p \left( \sum_{k=1}^K \dfrac{L_k \log n_k}{n_k} \right) 
+ \mathcal{O}_p \left( \sum_{k=1}^K L_k^{-\alpha_k} \right),
\end{aligned}
\end{equation}
where \(\epsilon_{k,i}^m\) is the sample Bellman error for policy \(\hat{\pi}^m\), defined as
\begin{multline}\label{eq:sample_bellman_error}
\epsilon_{k,i}^m = r_{k,i}^m + \gamma_k T_{k,i} \sum_{a \in \action_{[k+1]}} Q_{[k+1]}^{\hat{\pi}^m}(\phi_k(s_k'^i), a) \hat{\pi}_{[k+1]}^m(a \mid \phi_k(s_k'^i)) \\
+ (1 - T_{k,i}) \sum_{a \in \action_k} Q_k^{\hat{\pi}^m}(s_k'^i, a) \hat{\pi}_k^m(a \mid s_k'^i) - Q_k^{\hat{\pi}^m}(s_k^i, a_k^i),
\end{multline}
and \( i \in \mathcal{D}_k^m \), with a slight abuse of notation, 
indicates that the sample \( (s_k^i, a_k^i, r_k^i, s_k'^i) \) at stage \( k \) belongs to \(\mathcal{D}_k^m\).

The second term in Equation~\eqref{eq:proof_asymptotic_coeff} is controlled 
using the bound from Equation~\eqref{eq:coefficient_matrix_inverse_bound} 
and the independence and conditional mean-zero property of the Bellman error terms \(\epsilon_{k,i}^m\). 
The third term is bounded by adapting the arguments of Lemma 4 in~\cite{shi2022statistical}, 
which relies on Assumption~\ref{assumption:min_eigenvalue}.

For later use, we denote the leading term as
\[
\zeta_m = \mathbf{H}_m^{-1} \left( \dfrac{N}{n} \sum_{k=1}^K \sum_{i \in \mathcal{D}_k^m} \psi_{k,i} \epsilon_{k,i}^m \right).
\]

\medskip

\textbf{Step 3: Asymptotic Decomposition.} 

Our primary goal is to establish the asymptotic normality of the estimator 
from Algorithm~\ref{alg:ensemble_cfqi}.

Specifically, for the first statement of Theorem~\ref{thm:asymptotic}, 
we seek to derive the asymptotic distribution of
\begin{equation}\label{eq:proof_asymptotic_goal}
\sqrt{\frac{n(N-1)}{N}} \widehat{\Sigma}^{-1/2} (\hat{\mathbf{v}} - \mathbf{v}^{\hat{\pi}}) 
= \sqrt{\frac{n}{N(N-1)}} \sum_{m=1}^{N-1} \widehat{\Sigma}_m^{-1/2} (\hat{\mathbf{v}}_{\mathcal{D}^{m+1}}(\hat{\pi}^m) - \mathbf{v}^{\hat{\pi}}) 
\xrightarrow{d} \mathcal{N}(\mathbf{0}_K, \mathbf{I}_K),
\end{equation}
where \(\hat{\pi} = \hat{\pi}^N\) denotes the final policy estimated using the entire data set \(\mathcal{D}\).

Drawing on arguments analogous to the proof of Theorem 1 in~\cite{shi2022statistical}, 
and leveraging the coefficient error decomposition from Equation~\eqref{eq:proof_asymptotic_coeff}, using the fact that \(L_k^{-\alpha_k} = o (n_k^{-1/2})\) by the choice of \(L_k\),
we obtain
\begin{multline}\label{eq:proof_asymptotic_long_error}
\sqrt{\frac{n}{N(N-1)}} \left\| \sum_{m=1}^{N-1} \widehat{\Sigma}_m^{-1/2} \left( \hat{\mathbf{v}}_{\mathcal{D}^{m+1}}(\hat{\pi}^m) - \mathbf{v}^{\hat{\pi}} \right) 
- \sum_{m=1}^{N-1} \widehat{\Sigma}_m^{-1/2} \left( \mathbf{v}^{\hat{\pi}^m} + \mathbb{E}[\mathbf{U}^m]^{\top} \zeta_m - \mathbf{v}^{\hat{\pi}} \right) \right\|_\infty 
\\ = o_p(1),
\end{multline}
as \( n \to \infty \).

Here, \(\mathbf{U}^m\) is a block diagonal matrix defined as
\[
\mathbf{U}^m = \begin{bmatrix}
\mathbf{U}_1^m & \mathbf{0} & \cdots & \mathbf{0} \\
\mathbf{0} & \mathbf{U}_2^m & \cdots & \mathbf{0} \\
\vdots & \vdots & \ddots & \vdots \\
\mathbf{0} & \mathbf{0} & \cdots & \mathbf{U}_K^m
\end{bmatrix}
\in \mathbb{R}^{\left( \sum_{k=1}^K L_k A_k \right) \times K},
\]
where each block \(\mathbf{U}_k^m(s_k) = \mathbf{U}_k(s_k)\) corresponds to stage \( k \) under policy \(\hat{\pi}^m\). 
The policy-weighted feature vector \(\mathbf{U}_k\) is defined in Equation~\eqref{eq:U_function} 
from subsection~\ref{sec:inference_setup}.

Next, we aim to show that the expected suboptimality of the estimated policies, 
\(\mathbb{E}\|\mathbf{v}^{\hat{\pi}^m} - \mathbf{v}^*\|_\infty\), 
converges at a rate \( O(n^{-b_0}) \) with \( b_0 > 1/2 \).
To this end, we use the \( L_2 \)-norm error bounds for the estimated \( Q \)-functions 
from Lemma~\ref{lemma:q_error_bound}, 
which decay at rate \( n_k^{-b_*} \) for \( b_* > 1/4 \), by assumption. 
We consider two cases based on \( b_* \).

If \( b_* > 1/2 \), 
the desired rate follows directly from the proofs of 
Theorem~\ref{thm:main_convergence} and Theorem~\ref{thm:finite_sample_rate}, 
confirming \( b_0 > 1/2 \).

If \( b_* \leq 1/2 \), 
we incorporate the margin condition from Assumption~\ref{assumption:margin} 
for Theorem~\ref{thm:asymptotic}, 
which requires the exponent \(\alpha\) to satisfy
\[
\alpha > \dfrac{2 - 4b_*}{4b_* - 1}.
\]
Using an argument analogous to Theorem 4 of~\cite{shi2022statistical}, 
which leverages the faster convergence of the integrated value suboptimality gap 
relative to the \( Q \)-function under the margin condition, 
we combine the \( Q \)-function rate \( b_* \) and the margin parameter \(\alpha\) to define
\[
b_0 = b_* \dfrac{2 + 2\alpha}{2 + \alpha}.
\]
This ensures \( b_0 > 1/2 \). 
Thus, for all \( m = 1, \dots, N \),
\[
\|\mathbf{v}^{\hat{\pi}^m} - \mathbf{v}^*\|_\infty = O_p(n^{-b_0}),
\]
and the expected suboptimality is
\[
\mathbb{E}\|\mathbf{v}^{\hat{\pi}^m} - \mathbf{v}^*\|_\infty = O(n^{-b_0}),
\]
since the values are bounded. 
In both cases, the rate \( b_0 > 1/2 \) is achieved.

As \(\hat{\pi} = \hat{\pi}^N\) uses the full data set \(\mathcal{D} = \bar{\mathcal{D}}_N\), 
this bound applies to the final policy as well. 
Consequently, the average difference between the values of intermediate policies \(\hat{\pi}^m\) 
and the final policy \(\hat{\pi}\) vanishes at rate \( n^{-b_0} \):
\begin{align*}
\frac{1}{N-1} \sum_{m=1}^{N-1} \mathbb{E} \|\mathbf{v}^{\hat{\pi}^m} - \mathbf{v}^{\hat{\pi}}\|_\infty 
&\leq \frac{1}{N-1} \sum_{m=1}^{N-1} \mathbb{E} \|\mathbf{v}^{\hat{\pi}^m} - \mathbf{v}^*\|_\infty + \mathbb{E} \|\mathbf{v}^{\hat{\pi}} - \mathbf{v}^*\|_\infty \\
&= O(n^{-b_0}).
\end{align*}
Since \( b_0 > 1/2 \), 
\[
\sqrt{\frac{nN}{N-1}} \sum_{m=1}^{N-1} \mathbb{E} \|\mathbf{v}^{\hat{\pi}^m} - \mathbf{v}^{\hat{\pi}}\|_\infty 
= \mathcal{O}(n^{1/2 - b_0}) = o(\|\mathbb{E}\Phi\|_2),
\]
where \(\Phi = (\Phi_1^{\top}, \dots, \Phi_K^{\top})^{\top}\) is the vector of sieve basis functions 
across all stages, 
and the expectation is taken over the initial state distributions \(\eta_k\). 
By Markov’s inequality, 
\[
\|\mathbb{E}\Phi\|_2^{-1} \sqrt{\frac{nN}{N-1}} \sum_{m=1}^{N-1} \|\mathbf{v}^{\hat{\pi}^m} - \mathbf{v}^{\hat{\pi}}\|_\infty = o_p(1),
\]
as \( n \to \infty \).

Let \( A = \max_k A_k \) denote the maximum number of actions across all stages. 
Using the Cauchy-Schwarz inequality, 
we relate the norm of the expected basis functions to that of the expected \(\mathbf{U}^m\) matrix 
for \( m = 1, \dots, N-1 \):
\[
\|\mathbb{E}\Phi\|_2^{-1} \geq A^{-1/2} \|\mathbb{E}[\mathbf{U}^m]\|_2^{-1},
\]
yielding
\begin{equation}\label{eq:asymptotic_proof_error}
\sqrt{\frac{nN}{N-1}} \sum_{m=1}^{N-1} \|\mathbb{E}[\mathbf{U}^m]\|_2^{-1} \|\mathbf{v}^{\hat{\pi}^m} - \mathbf{v}^{\hat{\pi}}\|_\infty = o_p(1).
\end{equation}

Moreover, adapting arguments on the convergence of covariance matrix estimates 
from the proof of Theorem 1 of ~\cite{shi2022statistical}, similar to their proof of Theorem 1, 
we find
\begin{equation}\label{eq:proof_asymptotic_covariance_error}
\max_{1 \leq m < N} \left\| \Sigma_m^{-1/2} (\widehat{\Sigma}_m^{1/2} - \Sigma_m^{1/2}) \right\|_\infty = o_p(1).
\end{equation}

Furthermore, using the results from \textbf{Step 1} on the coefficient matrices \(\mathbf{H}_n\) and \(\widehat{\mathbf{H}}_n\), 
there exists a constant \( c_1 > 0 \) such that, 
for all \( m = 1, \dots, N-1 \), 
with probability approaching 1 as \( n \to \infty \),
\[
\lambda_{\min}(\Sigma_m^{1/2}) \geq c_1 \|\mathbb{E}[\mathbf{U}^m]\|_2, \quad 
\lambda_{\min}(\widehat{\Sigma}_m^{1/2}) \geq c_1 \|\mathbb{E}[\mathbf{U}^m]\|_2.
\]

Thus, from Equation~\eqref{eq:asymptotic_proof_error}, 
\[
\sqrt{\frac{n}{N(N-1)}} \left\| \sum_{m=1}^{N-1} \widehat{\Sigma}_m^{-1/2} \left( \mathbf{v}^{\hat{\pi}^m} - \mathbf{v}^{\hat{\pi}} \right) \right\|_\infty = o_p(1).
\]
Combining this with Equation~\eqref{eq:proof_asymptotic_long_error}, 
we obtain
\[
\sqrt{\frac{n}{N(N-1)}} \left\| \sum_{m=1}^{N-1} \widehat{\Sigma}_m^{-1/2} \left( \hat{\mathbf{v}}_{\mathcal{D}^{m+1}}(\hat{\pi}^m) - \mathbf{v}^{\hat{\pi}} \right)
- \sum_{m=1}^{N-1} \widehat{\Sigma}_m^{-1/2} \mathbb{E}[\mathbf{U}^m]^{\top} \zeta_m \right\|_\infty = o_p(1).
\]
To prove the goal in Equation~\eqref{eq:proof_asymptotic_goal}, 
it suffices to show the distributional convergence
\begin{equation}\label{eq:proof_asymptotic_goal2}
\sqrt{\frac{n}{N(N-1)}} \sum_{m=1}^{N-1} \widehat{\Sigma}_m^{-1/2} \mathbb{E}[\mathbf{U}^m]^{\top} \zeta_m 
\xrightarrow{d} \mathcal{N}_K(\mathbf{0}, \mathbf{I}_K).
\end{equation}

\medskip

\textbf{Step 4: Asymptotic Martingale Representation.}

To replace the estimated covariance matrices \(\widehat{\Sigma}_m\) with their expected values \(\Sigma_m\), 
we decompose the left-hand side of Equation~\eqref{eq:proof_asymptotic_goal2} as
\[
\sqrt{\frac{n}{N(N-1)}} \sum_{m=1}^{N-1} \Sigma_m^{-1/2} \mathbb{E}[\mathbf{U}^m]^{\top} \zeta_m 
+ \sqrt{\frac{n}{N(N-1)}} \sum_{m=1}^{N-1} \bigl( \widehat{\Sigma}_m^{-1/2} \Sigma_m^{1/2} - \mathbf{I} \bigr) \Sigma_m^{-1/2} \mathbb{E}[\mathbf{U}^m]^{\top} \zeta_m.
\]
By Equation~\eqref{eq:proof_asymptotic_covariance_error} and Slutsky’s theorem, 
the supremum norm of the second term converges to zero in probability. 
Thus, it suffices to analyze the first term, which can be expressed as
\begin{equation}\label{eq:martingale_sum}
\sqrt{\frac{N}{n(N-1)}} \sum_{m=1}^{N-1} \sum_{k=1}^K \sum_{i \in \mathcal{D}_k^m} \Sigma_m^{-1/2} \mathbb{E}[\mathbf{U}^m]^{\top} \mathbf{H}_m^{-1} \psi_{k,i} \epsilon_{k,i}^m,
\end{equation}
where \(\psi_{k,i}\) is the feature vector for stage \( k \) and sample \( i \), 
defined in subsection~\ref{sec:inference_setup}, 
and \(\epsilon_{k,i}^m\) is the sample Bellman error from Equation~\eqref{eq:sample_bellman_error}.

To represent this sum as a martingale, 
we construct a filtration \(\{\mathcal{F}^{(g)}\}_{g \geq 0}\), 
where \( g \) indexes the terms in Equation~\eqref{eq:martingale_sum} 
corresponding to the \( i \)-th sample from stage \( k \) of the \( m \)-th data set \(\mathcal{D}_k^m\). 
We order the indices \((m, k, i)\) such that
\[
(m_1, k_1, i_1) \prec (m_2, k_2, i_2)
\]
if and only if \( m_1 < m_2 \), or \( m_1 = m_2 \) and \( k_1 < k_2 \), 
or \( m_1 = m_2 \), \( k_1 = k_2 \), and \( i_1 < i_2 \). 
The index \( g \) ranges from 0 to \( n-1 \), 
and we denote the corresponding \((m, k, i)\) for each \( g \) as \( m(g), k(g), i(g) \), respectively.

The filtration \(\mathcal{F}^{(g)}\) is the sigma-field generated by the union of state-action pairs 
with indices \((m, k, i)\) such that \((m, k, i) \prec (m(g), k(g), i(g))\).

Using this ordering, Equation~\eqref{eq:martingale_sum} can be rewritten as
\begin{equation}\label{eq:martingale_sum_reindex}
\sqrt{\frac{N}{n(N-1)}} \sum_{g=0}^{n-1} \left[ \Sigma_{m(g)}^{-1/2} \mathbb{E}[\mathbf{U}^{m(g)}]^{\top} \mathbf{H}_{m(g)}^{-1} \psi_{k(g),i(g)} \epsilon_{k(g),i(g)}^{m(g)} \right],
\end{equation}
which constitutes a sum of mean-zero martingale differences with respect to the filtration \(\{\mathcal{F}^{(g)}\}_{0 \leq g < n}\). 
We denote \(\psi_{k(g),i(g)} = \psi^{(g)}\) and \(\epsilon_{k(g),i(g)}^{m(g)} = \epsilon^{(g)}\).

Following an argument analogous to the  proof of Theorem 1 in~\cite{shi2022statistical}, 
and leveraging the condition \( L = o\left( \dfrac{\sqrt{n}}{\log n} \right) \), 
using that the \(\epsilon_{k,i}^m\) are uniformly bounded and \(\normtwo{\psi^{(g)}} \leq \sup \normtwo{\Phi} \leq c^* \sqrt{L}\), 
we show that the maximum contribution of the martingale terms is negligible:
\begin{equation}\label{eq:martingale_max_bound}
\max_{0 \leq g < n} \norminf{\Sigma_{m(g)}^{-1/2} \mathbb{E}[\mathbf{U}^{m(g)}]^{\top} \mathbf{H}_{m(g)}^{-1} \psi_{k(g),i(g)} \epsilon_{k(g),i(g)}^{m(g)}} = o_p(1).
\end{equation}

Further, we have
\begin{align*}
&\norminf{\dfrac{N}{n(N-1)} \sum_{g=0}^{n-1} \left[ \Sigma_{m(g)}^{-1/2} \mathbb{E}[\mathbf{U}^{m(g)}]^{\top} \mathbf{H}_{m(g)}^{-1} \psi^{(g)} \psi^{(g) \top} (\mathbf{H}_{m(g)}^{-1})^{\top} \mathbb{E}[\mathbf{U}^{m(g)}] \Sigma_{m(g)}^{-1/2} (\epsilon^{(g)})^2 \right] - \mathbf{I}_K} \\
&\leq \max_{1 \leq m < N} \norminf{\dfrac{N}{n} \sum_{(k,i) \in \mathcal{D}_m} \left[ \Sigma_m^{-1/2} \mathbb{E}[\mathbf{U}^m]^{\top} \mathbf{H}_m^{-1} \psi_{k,i} \psi_{k,i}^{\top} (\mathbf{H}_m^{-1})^{\top} \mathbb{E}[\mathbf{U}^m] \Sigma_m^{-1/2} (\epsilon_{k,i}^m)^2 \right] - \mathbf{I}_K} \\
&= \max_{1 \leq m < N} \norminf{\Sigma_m^{-1/2} \mathbb{E}[\mathbf{U}^m]^{\top} \mathbf{H}_m^{-1} \widehat{\Omega}^*_m (\mathbf{H}_m^{-1})^{\top} \mathbb{E}[\mathbf{U}^m] \Sigma_m^{-1/2} - \mathbf{I}_K} \\
&= \max_{1 \leq m < N} \norminf{\Sigma_m^{-1/2} \mathbb{E}[\mathbf{U}^m]^{\top} \mathbf{H}_m^{-1} (\widehat{\Omega}^*_m - \Omega^*_m) (\mathbf{H}_m^{-1})^{\top} \mathbb{E}[\mathbf{U}^m] \Sigma_m^{-1/2}} \\
&\leq \max_{1 \leq m < N} \normtwo{\widehat{\Omega}^*_m - \Omega_m} \normtwo{(\mathbf{H}_m^{-1})^{\top} \mathbb{E}[\mathbf{U}^m] \Sigma_m^{-1/2}}^2,
\end{align*}
where \(\widehat{\Omega}^*_m = \dfrac{N}{n} \sum_{(k,i) \in \mathcal{D}_{m+1}} \psi_{k,i} \psi_{k,i}^{\top} (\epsilon_{k,i}^m)^2\) 
and \(\Omega_m = \mathbb{E}[\widehat{\Omega}_m]\), 
with \(\widehat{\Omega}_m\) being the matrix \(\widehat{\Omega}\) defined in subsection~\ref{sec:inference_setup} with respect to policy \(\hat{\pi}^m\) and data set \(\mathcal{D}_{m+1}\).

Following an argument analogous to the proof of Theorem 1 in~\cite{shi2022statistical}, 
we derive that
\[
\max_{1 \leq m < N} \normtwo{\widehat{\Omega}^*_m - \Omega_m} = o_p(1),
\]
as \( n \to \infty \), 
and that \(\max\limits_{1 \leq m < N} \normtwo{(\mathbf{H}_m^{-1})^{\top} \mathbb{E}[\mathbf{U}^m] \Sigma_m^{-1/2}}^2\) is bounded above by a constant.

Thus, we obtain
\begin{multline*}
\norminf{\dfrac{N}{n(N-1)} \sum_{g=0}^{n-1} \left[ \Sigma_{m(g)}^{-1/2} \mathbb{E}[\mathbf{U}^{m(g)}]^{\top} \mathbf{H}_{m(g)}^{-1} \psi^{(g)} \psi^{(g) \top} (\mathbf{H}_{m(g)}^{-1})^{\top} \mathbb{E}[\mathbf{U}^{m(g)}] \Sigma_{m(g)}^{-1/2} (\epsilon^{(g)})^2 \right] - \mathbf{I}_K} \\
= o_p(1).    
\end{multline*}

Now, applying a vector-valued martingale central limit theorem for triangular arrays, 
we establish Equation~\eqref{eq:proof_asymptotic_goal2}, 
satisfying the first result of Theorem~\ref{thm:asymptotic}. 

For the second result, since we have shown through this proof that the suboptimality gap converges at a rate \( n^{-b_0} \) with \( b_0 > 1/2 \), 
we can apply Slutsky’s theorem to obtain the second result of the Theorem~\ref{thm:asymptotic}.
\end{proof}

\vskip 0.2in
\bibliography{ref}

\end{document}